\documentclass{article} % For LaTeX2e
\usepackage{preprint, times}
\usepackage{hyperref}
\usepackage{url}

\PassOptionsToPackage{numbers, compress}{natbib}
\usepackage{natbib}

\usepackage[utf8]{inputenc} % allow utf-8 input
\usepackage[T1]{fontenc}    % use 8-bit T1 fonts
\usepackage{hyperref}       % hyperlinks
\usepackage{url}            % simple URL typesetting
\usepackage{booktabs}       % professional-quality tables
\usepackage{amsfonts}       % blackboard math symbols
\usepackage{nicefrac}       % compact symbols for 1/2, etc.
\usepackage{microtype}      % microtypography
\usepackage{xcolor}         % colors

\usepackage{float} 
\usepackage{multirow} 
\usepackage{multicol} 
\usepackage{extarrows}
\usepackage{subfigure} 

\usepackage{algorithm}
\usepackage{algorithmic}
\usepackage{graphicx} % Required for inserting images

\usepackage{wrapfig}
\usepackage{graphicx}
\usepackage{amsmath}
\usepackage{amssymb}
\usepackage{amsthm}

\theoremstyle{plain}
\newtheorem{Theorem}{Theorem}[section]
\newtheorem{proposition}[Theorem]{Proposition}
\newtheorem{lemma}[Theorem]{Lemma}
\newtheorem{corollary}[Theorem]{Corollary}
\theoremstyle{definition}
\newtheorem{definition}[Theorem]{Definition}

\theoremstyle{remark}
\newtheorem{remark}[Theorem]{Remark}

\def\R{{\mathbb{R}}}

\def\Z{{\mathbb{Z}}}

\def\H{{\mathbb{H}}}
\def\E{{\mathbb{E}}}
\def\Pr{{\mathbb{P}}}

\def\L{{\mathcal{L}}}
\def\F{{\mathcal{F}}}
\def\G{{\mathcal{G}}}
\def\D{{\mathcal{D}}}

\def\ID{{\mathbf{I}}}

\def\poly{{\rm{poly}}}
\def\len{{\rm{len}}}
\def\last{{\rm{last}}}
\def\typ{{\rm{Toks}}}
\def\softmax{{\rm{softmax}}}

\def\Relu{{\rm{Relu}}}

\def\Num{{\rm{Num}}}

\def\PPF{{\rm{PPP}}}
\def\DPO{{\rm{DPO}}}

\def\Rad{{\rm{Rad}}}

\def\D{\mathcal{D}}
\def\A{\mathcal{A}}

\def\Loss{\mathcal{L}}

\def\Prom{\mathcal{P}}
\def\PM{{\mathbf{Prompt}}}

\def\ATT{\hbox{\rm{ATT}}}
\def\FNN{\hbox{\rm{FNN}}}
\def\Relu{\hbox{\rm{Relu}}}
\def\softmax{\hbox{\rm{softmax}}}
\def\argmax{\hbox{\rm{argmax}}}
\def\argmin{\hbox{\rm{argmin}}}

\title{On the Capability and Limitation of Hard Prompt}

\author{
Lijia Yu\textsuperscript{\rm 3},
Shuaitong Liu\textsuperscript{\rm 4}, 
Gaojie Jin\textsuperscript{\rm 5}, 
Xinyu Li\textsuperscript{\rm 6}, 
Xiao-Shan Gao\textsuperscript{\rm 1, 2}\thanks{Corresponding author, xgao@mmrc.iss.ac.cn}\\
\textsuperscript{\rm 1}SKLMS, Academy of Mathematics and Systems Science, Chinese Academy of Sciences\\ 
 \textsuperscript{\rm 2}University of Chinese Academy of Sciences\\
 \textsuperscript{\rm 3}Institute of AI for Industries, Chinese Academy of Sciences\\
 \textsuperscript{\rm 4}Southwest University\\
 \textsuperscript{\rm 5}Department of Artificial Intelligence$\&$Institute of Artificial Intelligence and Brain Sciences, University of Macau\\
 \textsuperscript{\rm 6}Department of Computer Science, University of Exeter
}

\begin{document}

%\nocite{*}
\maketitle

\begin{abstract}
Prompt engineering has become an indispensable tool for using large language models (LLMs), turning LLMs into task-specific experts without changing their weights. 
Despite notable theoretical advances in prompt engineering, the theory for the more practical hard or discrete prompts is largely open.
In this paper, we try to fill this gap either by providing a complete solution or by making substantial progress on the three core theoretical questions regarding hard prompts.
First, we show that determining the existence of a hard prompt for a transformer to solve a downstream task is NP-complete and that finding an optimal hard prompt is NP-hard, which is the first computational complexity result for hard prompting, as far as we know.
Second, we show that, unlike soft or continuous prompts, hard prompts have essential limitations: hard prompts are not complete; short hard prompts do not significantly enhance the ability of transformers; and long hard prompts exhibit the ``prompt dominating answer phenomenon,'' meaning that, with high probability, the same answer is given for all queries of the same length. On the other hand, linear hard prompts do not have the limitations of short or long prompts.
Third, we provide a tight bound on the size of the task in terms of the prompt length for the performance of prompts on the finite task to generalize to the entire data distribution, leading to a necessary and sufficient condition for generalizability. This is the first result on generalization for prompting, as far as we know.
%Our findings not only offer enhanced theoretical insights into prompt engineering but also provide  useful practical guidance.
Our findings not only offer the first theoretical insights into hard prompts but also provide provably reliable practical guidance for real-world LLM usage.
%
%Despite notable theoretical advances in prompt engineering, many key questions remain open regarding the more practical setting of designing prompts (not their embedding parameters) for a given fixed transformer.
%In this paper, we either provide a complete solution or make substantial progress on the three core theoretical questions regarding this kind of prompt.
\end{abstract}

\section{Introduction}

Prompt engineering has become a pivotal technique for the great success of large language models (LLMs), enhancing their accuracy, reasoning ability, and applicability across almost all areas \citep{brown2020language,wei2022chain,Autoprompt2020,TGao-2021,Survey-prompt2}.
%Survey-prompt1,
Prompts for LLMs can be broadly categorized into two paradigms: hard (discrete) prompts and soft (continuous) prompts \citep{li2025promptcompression}. 
Hard prompts refer to discrete instructions composed of real natural language tokens, which are human-readable and require no weight changes \citep{li2025promptcompression}. 
%Nevertheless, they suffer from inherent limitations, including heavy reliance on human expertise, unstable performance across tasks, and poor adaptability to low-resource scenarios \citep{liu2021gpt}. 
In contrast, soft prompts adopt learnable continuous embedding vectors inserted into the model input, which freeze the backbone model and only optimize a small number of parameters \citep{li2021prefix,lester2021power,li2024survey}. 
%However, soft prompts lack interpretability due to their non-linguistic continuous nature and usually yield limited gains on small-scale language models.
%Hard prompts represent the most intuitive and conventional manner of prompt utilization in practical scenarios.
Hard prompts are the earliest and most straightforward practical form of prompt design, while soft prompts have become a prevalent, lightweight tuning strategy for large language models.
In practice, prompts can either be
designed by human experts \citep{brown2020language,wei2022chain} or 
found through prompt optimization \citep{li2021prefix,lester2021power,prasad2022grips,zhang2022tempera,pryzant2023automatic} by minimizing the loss.  

Besides the practical progress, substantial efforts have been directed towards establishing a theoretical framework for prompt engineering.
Prompts have been proven to be complete in a certain sense.
%\citep{ICLR2025-Hu1,pp1,pp2} further proves the Turing completeness of prompt.
%
By modifying the embedding layer, soft prompts can turn a transformer model into a perfect memorizer of a finite dataset \citep{pp1,ICLR2025-Hu1}.
Using a suitably constructed transformer, prompts are Turing complete for emulating computable functions \citep{pp2}.
Specific transformers can be designed such that, when equipped with prompts, they can serve as universal approximators for certain functions \citep{pro4,pro5,pro3,ICLR2025-Hu1}.
Theoretical results for prompts in the form of in-context learning were given in \citep{icl2,icl3,icl4,icl1,icl5}.

However, the theoretical understanding of hard prompts remains largely underexplored.  
In this paper, we investigate three core theoretical questions regarding hard prompts and provide theoretical explanations for several questions and phenomena encountered in prompt engineering.
%Therefore, in this article, based on the general transformer, we will examine three fundamental theoretical aspects of hard prompts and offer explanations and theoretical accounts for various questions and phenomena that arise in prompt engineering.

{\bf Question 1:  What is the computational complexity of computing an optimal hard prompt for a given transformer and task? }

Let $\F$ be a transformer, $S=\{(x_i,y_i)\}_{i=1}^N$ a downstream task consisting of query-answer pairs, and $\L(\F,S)$ the supervised loss of $\F$ on $S$.
A sequence $\Prom$ is called a {\em prompt of $\F$ for $S$ with loss $\alpha$} if $\L(\F,\Prom\oplus S)\le\alpha$, where $\Prom\oplus S=\{(\Prom\oplus x_i,y_i)\}_{i=1}^N$ and $\Prom\oplus x_i$ are the concatenation of $\Prom$ and $x_i$.
%and the length of $\Prom$ is polynomial in that of $S$. 
%We say that {\em $\F$ solves $S$ with loss $\alpha$ using a $\Prom$}, if $\L(\F,\Prom\oplus S)\le\alpha$, where $\Prom\oplus S=\{(\Prom\oplus x_i,y_i)\}_{i=1}^N$.

Further, $\Prom$ is called a {\em polynomial prompt} if the length of $\Prom$ is polynomial in that of $S$. 
We thus have (see  Theorem \ref{th-npc} and Corollary \ref{cor-13}):
%$=\frac{1}{N}\sum_{i\in [N]}\L_{\text{nll}}(\F,x_i,y_i)$, $\F$ can solve $S$ with loss means that $\L(\F,S)\le\alpha$, and $\Prom\oplus S=\{(\Prom\oplus x_i,y_i)\}$, we thus get the following result:
\begin{Theorem}[Informal]
\label{th-m1}
%For a given transformer $\F$, task $S=\{(x_i,y_i)\}_{i=1}^N$, and loss $\alpha$, determine if there exists a prompt $\Prom$ with or without length limitation that makes  $\L(\F,\Prom\oplus S)\le \alpha$ valid as NPC. Hence, finding the prompt minimum $\L(\F,\Prom\oplus S)$ with length limitation is NP-hard.
%
Determining the existence of a polynomial hard prompt $\Prom$ of $\F$ for $S$ with a given loss is NPC, and finding a polynomial hard prompt with a minimum loss is NP-hard.
%Determining the existence of a prompt $\Prom$ such that $\F$ solves $S$ with loss $\alpha$ using $\Prom$ is NPC, and finding a prompt $\Prom$ that minimizes $\L(\F,\Prom\oplus S)$ is NP-hard.
\end{Theorem}

Since $\L(\F,\Prom\oplus S)$ is the optimization objective in prompt optimization, this result indicates that finding an optimal polynomial hard prompt is an NP-hard problem.
%
%Polynomial-length prompts are considered because if the prompt size grows faster than polynomial relative to the task, computing such prompts becomes clearly intractable.
%We focus on polynomial-length prompts because super-polynomial prompts are computationally infeasible in practice
To the best of our knowledge, this is the first theoretical result on the computational complexity of hard prompting.

It is reasonable to consider polynomial-length prompts, because super-polynomial prompts are computationally intractable in practice and are already NP-hard to compute in terms of the task.

{\bf Question 2: What are the powers and limitations of hard prompts in enhancing the inference ability of transformers?}

Soft prompts have been demonstrated to be complete in the memorization of a finite dataset \citep{ICLR2025-Hu1}. 
%and \citep{pp2} shows that for specially constructed transformers, prompts can simulate any computable function. 
In contrast, we demonstrate that hard prompts face fundamental limitations. 

First, we show that hard prompts are, in a certain sense, incomplete (see Theorem \ref{negg}). 
\begin{Theorem}[Informal]
\label{th-m2}
Polynomial hard prompts are not complete in that they cannot help a transformer with fixed weights to solve all tasks with a given loss. 
Specifically, for certain query set $S_q=\{x_i\}_{i=1}^N$ and $\alpha{>0}$, there exist no transformers $\F$ that can solve task $S=\{(x_i,y_i)\}_{i=1}^N$ for any answer set $S_a=\{y_i\}_{i=1}^N$ by using a polynomial prompt $\Prom$.%$(\alpha,c,d)$-prompt.%
%
%for certain tasks $S$ and loss $\alpha>0$, there exist no transformer $\F$ and polynomial prompt $\Prom$ such that $\Prom$ is a prompt of $\F$ for $S$ of loss $\alpha$.
\end{Theorem}

This theoretical finding highlights fundamental distinctions between soft prompts and hard prompts.
%because they did not consider loss. Hence, we also show the limitation of the prompt with different lengths:
Note that this is not contradictory to previous work \citep{ICLR2025-Hu1,pp2,pp1}. Refer to Remark \ref{rem-dr} for a detailed comparison. 
We further have 
\begin{Theorem}[Informal]
\label{th-m3}
Let the length of the prompt $\Prom$ be polynomial in that of the task $S$. 

(1) When the prompts are much shorter than the task, they do not significantly enhance a transformer’s performance on that task.
More specifically, if a transformer is unable to solve a long task without any prompt, then, with high probability, it will also fail to solve that task when given only a short prompt (see Theorem \ref{th-51} and Corollary \ref{th1-b}). 

(2) When the prompts are much longer than the task, they display the “Prompt Dominating Answer Phenomenon'', which means that the transformer produces an identical answer—taken directly from the prompt itself—for every query with the same length and last token (see Theorem \ref{cxxx}).

(3) Linear prompts avoid the previously mentioned drawbacks of both short and long prompts, and they can boost the effectiveness of certain transformers on certain tasks with high probability (see Propositions \ref{pl1} and \ref{pl2}).
\end{Theorem}

This result provides an explanation for the possible problems that may occur in practice when using prompts of different lengths. 
By (1), a short prompt cannot solve long tasks with high probability because the prompt has a limited effect on the later part of the output.
%the decrease of the loss function;
By (2), a long prompt can cause the transformer to ignore the query itself with a high probability. This shows a key distinction between prompting and training, where the scaling law \citep{Kaplan2020ScalingLF} implies that more data and parameters lead to better performance, but a long prompt  does not. 

{\bf Question 3: 
Can the effects of prompts for a finite task be generalized to additional query-answer samples? }

We provide a complete answer to this problem (Refer to Theorems \ref{th-genb} and \ref{prop-gen1}).

\begin{Theorem}[Informal]
\label{th-m4}
Let $S$ be a finite task drawn iid from a distribution $\D$.
%Let $L$ be the length of the prompt and $N$ the number of data in the task $S$. 
We provide upper and lower bounds of $|S|$ in terms of $\len(\Prom)$ for the performance of $\Prom$ on $S$ to generalize to $\D$. Both bounds are $\overline{O}(\len(\Prom))$ when omitting some small quantities, resulting in a tight generalization bound.
\end{Theorem}

%Our theoretical findings also offer enhanced theoretical insights and some practical guidance. 
%Our theoretical results further enhance the understanding of prompt engineering. They also provide practical guidance for prompt engineering, as shown in Appendix \ref{sec-54}. 
%We further explain the usability of our theory in practical situations in Appendix \ref{sec-54}.
%Finally, simple experimental results are used to support our main theoretical results in Appendix \ref{app-exp}.
This result indicates that a necessary and sufficient condition for generalizability is that the size of the task should be larger than the length of the prompt. 

{\bf Summary.} We either provide answers or make substantial contributions to the three  aforementioned questions. 
%Our theoretical results further enhance the understanding of prompt engineering. 
%They also provide practical guidance for prompt engineering, as shown in Appendix \ref{sec-54}. 
The main contributions of this paper are:
\vspace{-1mm}
\begin{itemize}
\item We provide a complete solution to Question 1.
Determining the existence of polynomial hard prompts for a finite task is NP-complete, and obtaining optimal ones is NP-hard. 

\item We make substantial progress on Question 2 by showing that polynomial prompts have essential limitations, as demonstrated in Theorems \ref{th-m2} and \ref{th-m3}.

\item We provide a complete solution to Question 3 by presenting a tight generalization bound.
\end{itemize}

Our results not only provide the first theoretical insights on hard prompts but also have practical implications:
keep prompt length proportional to answer length and avoid extremely short or long prompts; a necessary sample size is required to ensure the generalization of prompts. 
Refer to Appendix \ref{sec-54} for more details. 
This paper is primarily theoretical. In Appendix \ref{app-exp}, we include simple experimental results that support our key theoretical results, serving as minimal proof-of-concept demonstrations rather than comprehensive large-scale benchmarks.

%For a theoretical paper, qualitative alignment between theory and practice is sufficient and standard in NeurIPS theoretical papers.
\vspace{-2mm}
\section{Related work}
\vspace{-2mm}
%LLMs can solve various tasks by zero-shot or few-shot prompting, without modifying the weights of the pre-trained model \citep{radford2019language,brown2020language}.
%
%Chain of thoughts (CoT) prompting was introduced to improve the reasoning ability of LLMs by dividing the solution into sequential steps \citep{wei2022chain}.
%
%Retrieval augmented generation (RAG) was proposed to use information retrieved from an external source as prompts \citep{RAG-1,RAG-2,RAG-3}. 
%
%The above prompts are either generated by the user or taken from an external knowledge base, which are fixed in a certain sense. On the other hand, prompt optimization was proposed to generate better prompts for specific tasks by optimizing the prompts or their embeddings \citep{Autoprompt2020,TGao-2021,li2021prefix,prasad2022grips,pryzant2023automatic,zhang2022tempera,sun2022black,wen2023hard,PromptOptimization2024}. 

{\bf Prompt engineering.}
Prompt engineering refers to methods that guide LLMs to produce desired output by designing prompts without modifying the model \citep{Survey-prompt2}. 
\cite{radford2019language} demonstrated that zero-shot prompting can elicit task behavior.
% from a given transformer.
%Extending the zero-shot prompting, 
\cite{brown2020language} showed that LLMs can tackle new tasks via few-shot prompting without fine-tuning.
%, initiating the study of prompt engineering.
\cite{schick2020s} showed that prompting also works on small LLMs. 
\cite{lester2021power} showed that a simple prompt is comparable to full parameter tuning. 
\cite{wei2022chain} presented CoT prompting to enhance the reasoning capabilities of LLMs, leading to numerous powerful methods such as tree of thoughts \citep{ToT-2023} and graph of thoughts \citep{GoT-2024}.
RAG used external sources as prompts to increase accuracy \citep{lewis2020retrieval}. 
%, enabling LLMs to generate responses grounded in factual knowledge. 
Prompt optimization is used to generate hard prompts.
\cite{li2021prefix,pryzant2023automatic} showed that prompts can be found using gradients. 
\cite{prasad2022grips} provided gradient-free local searches. 
% to generate prompts.
\cite{zhang2022tempera} used reinforcement learning to create prompts.  
\cite{sun2022black} employed a black-box method to find prompts. 
%\cite{pryzant2023automatic} utilized natural language feedback as gradients to create prompts.
%
Soft prompts, first formalized and explored in a series of seminal studies \cite{qin2021learning,li2021prefix,liu2021gpt,lester2021power}, are learnable continuous embedding vectors inserted into transformer inputs instead of discrete text prompts.
%, offering the key advantages of freezing the full pre-trained language model, updating only a tiny set of parameters for strong parameter efficiency, and adapting easily to diverse downstream tasks while matching full fine-tuning performance on large-scale models, yet suffering from the inherent disadvantages of being human-uninterpretable, lacking explicit semantic meaning, and exhibiting weaker effectiveness on small language models compared to traditional hard prompts.
\cite{wen2023hard} employed the gradient method for the embedding layer to obtain soft prompts. 

Our theoretical results differ from prompt optimization, which is a practical technique for obtaining improved prompts by minimizing the loss. On the other hand, we provide rigorously proven results: hard prompt optimization is NP-hard, hard prompts are not complete, both short and long herd prompts have limitations, and a tight bound on prompt generalization is given. Our findings can be used to guide prompt optimization by helping to properly choose the prompt length.

{\bf Theory of prompting.} Theory of prompting was studied from several aspects.
%The work on the theoretical analysis of the prompt includes: 
\cite{pro4,pro5,ICLR2025-Hu1,pro3} demonstrated the expressive ability of transformers with prompts to approximate Lipschitz functions or certain differentiable functions. 
The theoretical results of prompts on ICL were presented: \cite{icl2} provided a convergence proof for a single head, single-layer softmax transformer;
% for Gaussian mixture classification;
\cite{icl3} explained ICL from the perspective of representation alignment;
\cite{icl1} formulated ICL as a learning problem within a hypothesis space;
\cite{icl4,icl5} used information theory to provide bounds for the Bayesian error.
%
%answers to the questions ``how many examples are sufficient'' for ICL and ``the impact of model size on ICL.''
%\citep{icl1,icl5} also proves some results of ICL. 
%
%\cite{pro2} formalized LLM as a class of discrete stochastic dynamical systems to explore prompt engineering through the lens of control theory. 
%
\cite{pro2} studied whether prompts with $\le k$ tokens push the output distribution closer to the target distribution using control theory. 
\cite{pro4,meyer2025memory,ICLR2025-Hu1} studied the limitations of prompts in single-head transformers with a single self-attention layer. 
\cite{ICLR2025-Hu1,pp1,pp2} proved that prompts are complete for transformers to simulate a function or memorize a dataset. 
% when the transformers are allowed to be changed in a certain manner.
%
%
%Our theoretical results on hard prompts are neither contradictory to nor derivable from the above results. Refer to Remark \ref{rem-dr} for details. 
%
%Our results on computational complexity and generalization bounds are the first such results, to the best of our knowledge.

%Our results are different from the existing theoretical results. Theorems \ref{th-m1} and \ref{th-m2} are the first results on complexity and existence conditions for perfect prompts.
%
%The generalization bound given in Theorem \ref{th-m3} is different from those in \citep{icl4,icl5} which treated prompting as a Bayesian inference procedure.
%
%We show that prompt optimization is NP-hard. 
%and that deciding the existence of prompts is NP-complete, 
%This is the first computational complexity result on prompting.

\section{Prerequisite}

We will introduce the basic symbols in this section. The details are in the Appendix \ref{xjj}.

%\begin{definition}  
%A subset $S=\{(x_i,y_i)\}_{i\in\Delta}\subset \Sigma^{*}\times \Sigma$ is called a {\em language} based on symbols %$\Sigma$, if: $y_i\ne y_j$ implies $x_i\ne x_j$.
%\end{definition}

%\begin{definition}  
%A subset $S=\{(x_i,y_i)\}_{i\in\Delta}\subset \Sigma^{*}\times \Sigma^{*}$ is called a {\em multiple-language} based on symbols $\Sigma$, if: $y_i\ne y_j$ implies $x_i\ne x_j$.
%\end{definition}

\subsection{Data and Autoregressive transformer}
%\paragraph{Autoregressive Transformer.}
\paragraph{Data.}
Let $\Sigma=\{\sigma_i\}_{i=0}^T$ be a set of vocabulary or tokens and  $\Sigma^{*}$ the set of all sequences of the form $(\sigma_{k_j})_{j=1}^n$.  
We assume $T\ge 3$ in this paper, which is reasonable. 
%A sequence $(\sigma_{k_j})_{j=1}^k$ is called a sequence with length $k$, and $\Sigma^{*}$ is the set of all sequences. 
The symbol $\sigma_0$ is used only to stop the output of the transformer and does not appear in the middle of a sequence. 
For a sequence $x$, let $\len(x)$ be its length, $x[i]$ be the $i$-th symbol in $x$, and $\last(x)=x[\len(x)]$ be the last symbol of $x$.
%, and $\typ(x)\subset \Sigma$ be the set of all the distinct symbols in $x$. 
%
For $x_1,x_2\in\Sigma^{*}$, we use $x_1\oplus x_2$ to denote the 
concatenation of $x_1$ and $x_2$.
%For $\gamma\in\Sigma$ and $x\in\Sigma^{*}$, we use 
%$\gamma\oplus x$ to denote the 
%1
%A {\em task} $S=\{x_i\}_{i=1}^N\subset \Sigma^{*}$ is a finite subset of $\Sigma^{*}$ such that 
%; that is, $S=\{x_i\}_{i=1}^N$ and $x_i\in\Sigma^{*}$. 
For a finite subset $S$ of $\Sigma^{*}$, denote 
%$|S|=N$ and 
$\len(S)=\max_{x\in S}\len(x)$.

An autoregressive transformer $\F$ has three parts:
the embedding, hidden layers, and the output layer.

{\bf Embedding.}
The transformer first embeds each token  $\sigma_i$ into a vector $v_i\in\R^w$, where $v_i$ serves as an adjustable parameter  and $w$ is called the {\em embedding dimension}.
%1
Then a sequence $x=(\sigma_{k_j})_{j=1}^n\subset\Sigma^{*}$ is embedded into a matrix in $\R^{n\times w}$:
$(v_{k_1},v_{k_2},\dots,v_{k_n})^{\tau}$, which is the input to the first hidden layer.
We will use $x$ and its matrix representation interchangeably whenever this does not lead to confusion.

%\begin{remark}
%We use relative positional encoding in the attention layer, but not in the embedding layer.
%\end{remark}

{\bf Hidden layer.}
We first define the feedforward layer.
For  $x\in\R^{n\times w}$ of length $n$, the feedforward layer with width $W$ is
$\FNN(x)=\Relu(xE_1\oplus b)E_2,$
where $E_1\in\R^{w\times W}, b\in\R^{1\times W}, E_2\in\R^{W\times w}$ are the parameters, and $xE_1\oplus b$ means adding $b$ to each row of $xE_1$. 
%This layer does not change the number of rows of input matrix, just do the same calculation to each row.
%
For an input $x\in\R^{n\times w}$, the attention-layer with width $W$ and head $H$ is
$\ATT(x)=\sum_{i=1}^H \softmax(R(xQ_i,K_ix^\tau)+M)xV_i,$
where $Q_i\in\R^{w\times W},K_i\in\R^{W\times w},V_i\in\R^{w\times w}$ are parameters. 
$M\in\{-\infty,0\}^{n\times n}$ is a causal mask defined as $M_{i,j}=-\infty$ if and only if $j>i$. $R(\cdot,\cdot)$ means relative position embedding. 
%, which is a core structure of the autoregressive transformer.
Without loss of generality, we  take $W=w$. 
%

%\end{remark}

Let $x^{(0)}$ be the input. Then the output of the  $i$-th hidden layer of the transformer is
{\small $$x^{(i)}=x^{(i-1)}+\ATT_i(x^{(i-1)})+\FNN_i(x^{(i-1)}+\ATT_i(x^{(i-1)}))$$}

\vskip-8pt\noindent
where $\ATT_i$ and $\FNN_i$ are the $i$-th attention and feedforward layers.
%, and the width of these layers is equal to the embedding length. 
%It is easy to see that $x^{(i)}\in\R^{n\times w}$ for all $i$.

{\bf Output layer.}
The output layer performs a linear transformation on the last row of the output from the last hidden layer, written as $x^{(L)}[n]$, that is: 
$\F(x)=x^{(L)}[n]H+c\in \R^{T+1}$,
where $H\in\R^{w\times (T+1)}$ and $c\in\R^{1\times (T+1)}$ are the parameters of the output layer. 
%
%Then the classification result of $\F(x)$, written as $\widehat{\F}(x)$, is $\sigma_j$,
%where $j=\arg\max_{i\in[T]}(\F(x))_i$. 
%
Let 
$\pi_{\F}(\cdot\mid x)=\softmax(\F(x))$, which can be saw as a distribution over $\Sigma$.

\paragraph{Output form.}
The final output $y=\widehat{\F}(x)$ is obtained as follows, with the initial value $y=()$: 
%
%(1) Let $s$ be a token  selected according to the distribution $\pi_{\F}(\cdot\mid x\oplus y)$ based on certain methods, such as the greedy rule.
%
(1) The next token $s\in\Sigma$ be randomly sampled according to the  distribution $\pi_{\F}(\cdot\mid x\oplus y)$;
(2) If $s\ne \sigma_0$, then let $y=y\oplus\{s\}$ and return to step (1);
%insert $s$ as the last element of $y$, and return to step (1).
%
(3) If $s=\sigma_0$, stop and return the output $y=\widehat{\F}(x)$.

Given the input $x$, the probability that $\F$ outputs $y$ is
\begin{equation}
\label{eq-pi}
    \pi_{\F}(y\mid x) =\prod_{i=1}^{\len(y)+1}\pi_{\F}(y[i] \mid x\oplus y[0:i-1])
\end{equation}
where $y[0:0]=()$ and $\last(y)=\sigma_0$. 
Let $\L_{\text{nll}}(\F,x,y)=-\log \pi_{\F}(y\mid x)$ represent the negative log-likelihood \citep{bengio2000neural}.
Note that stochastic sampling strategies, such as temperature sampling \citep{sampling-tem}, have been widely used.
%\citep{sampling-tem,sampling-topk,sampling-topp}.

\subsection{Polynomial-length prompt}
%We will introduce the concept of the prompt used in this paper.
%

In this paper, a {\em task} is a finite set of query-answer pairs: $S=\{(x_i,y_i)\}_{i=1}^N\subset \Sigma^{*}\times\Sigma^{*}$,
where $S_q=\{x_i\}_{i=1}^N$ is the {\em query set} and $S_a=\{y_i\}_{i=1}^N$ is the {\em answer set}.
For a prompt $\Prom\in\Sigma^{*}$, denote $\Prom\oplus S=\{(\Prom\oplus x_i,y_i)\}_{i=1}^N$ and $\Prom\oplus S_q=\{\Prom\oplus x_i\}_{i=1}^N$.

\begin{remark}
In this paper,  a single prompt $\Prom$ is used for all query-answer pairs of a task $S$, which is a standard practice  \citep{pp1,pp2}.
%This is also the case in practice. 
For instance, Goedel-Prover \citep{Goedel-Prover} uses a single prompt, while \cite{Llemma} uses a single prompt for each benchmark set.
%In mathematical provers based on LLMs, a fixed prompt is used, for instance, in \citep{Prover-Kmina}.
\end{remark}

{\bf Prompt with a given loss.}
%In general, the exact prompt mentioned above is usually too challenging, and a common alternative is to use the value of the training loss as a measure of the prompt's success.
%
Following the common practice in machine learning, we will use the value of the training loss as a measure of the prompt's success.
The commonly used crossentropy loss for  $\F$ on a task  $S=\{(x_i,y_i)\}_{i=1}^N$ is 
$\Loss(\F,S) = \frac{1}{N} \sum_{i=1}^N 
\L_{\text{nll}}(\F,x_i,y_i).$
%=-\frac{1}{N} \sum_{i=1}^N \log \pi_{\F}(y_i\mid x_i).$$
%
%We say that a transformer $\F$ {\em memorizes a task  $S=\{(x_i,y_i)\}_{i=1}^N$ with loss $\alpha$} if  $\Loss_{\F}(S) \le \alpha$.
%$\Loss(\F,S)$ is called the {\em loss} of $\F$ in solving $S$.

$\Prom\in\Sigma^{*}$ is called {\em a prompt of transformer $\F$ for a task $S$ with loss $\alpha\in\R_{>0}$} if 
$\Loss(\F,\Prom\oplus S)\le\alpha$.
%$, or simply say $\Prom$ is  an {\em $\alpha$-prompt of $\F$ for $S$}.
%
For convenience, we also say that {\em $\F$ ``solves'' task $S$ using prompt $\Prom$ with loss $\alpha$}.

\begin{remark}
In machine learning, the loss widely serves as a surrogate for accuracy in training and in generalization bounds \citep{mohri2018foundations}, because a lower loss value generally indicates higher accuracy. 
For example, for the single data case, the probability of $y=\widehat{\F}(x)$ is $\pi_{\F}(y\mid x) = \exp^{-\L_{\text{nll}}(\F,x,y)}$ which is higher when the loss is smaller.
Refer to Appendix \ref{pa} for more details.
\end{remark}

{\bf Polynomial-length prompt with loss.}
To be computationally tractable, we consider prompts of polynomial size, leading to the following main concept of this paper. 
We say that a transformer $\F$ {\em ``solves'' task $S$ with loss $\alpha$ using a $(c,d)$-prompt $\Prom\in\Sigma^{*}$} if 
$\Loss(\F,\Prom\oplus S)\le\alpha$ and $\len(\Prom)\le c\,\len(S_a)^d$, where $\alpha,c,d\in\R_{>0}$.
$\Prom$ is also called {\em a $(\alpha,c,d)$-prompt} or  {\em a polynomial prompt}.
%{\color{red} Why not consider $\len(\Prom)\le c\,\max\{\len(S_q)^d,\len(S_a)^d\}$? Seems more meaningful.
%}
%In this paper, we use the length of the answer set $\len(S_a)$ to measure the length of the prompt, which is discussed in section \ref{sec-cot-cuf}.

\begin{remark}
We focus on polynomial-length prompts because super-polynomial prompts are computationally intractable in practice.
%Polynomial-length prompts are considered because if the prompt size grows faster than polynomial relative to the task, computing such prompts becomes clearly intractable.
\end{remark}
\begin{remark}
Hard prompts are discrete natural language instructions that do not involve adapting model weights. Therefore, a fixed transformer is used in our problem formulation.
\end{remark}

\section{Computational complexity for hard prompting}
\label{sec-npc}
In this section, we show that finding a polynomial-length hard prompt with a given loss is an NP-hard problem. Some supplementary instructions are in Appendix \ref{sbsgr}. 
Let $\F$ be a transformer and $S=\{(x_i,y_i)\}_{i=1}^N\subset \Sigma^{*}\times\Sigma^{*}$ a task with $N$ elements.  
We define the following problem:

\begin{definition}
The decision problem $\PM(\F,$ $S,\alpha,c,d)$ is: for a transformer $\F$, 
a task $S$, and $\alpha, c,d\in\R_{>0}$, determine whether there exists an $(\alpha,c,d)$-prompt $\Prom\in\Sigma^{*}$ of $\F$ for $S$?
%such that $\F$ can solve $S$ with loss $c$ using  prompt $\Prom$. 
\end{definition}

%To be computationally meaningful, we assume $\alpha, c,d\in\Q_{>0}$ are rational numbers.
Then we have
\begin{Theorem}    
\label{th-npc}
$\PM(\F,S,\alpha,c,d)$ is an NPC problem.
Thus, finding an $(\alpha,c,d)$-prompt of $\F$ for $S$ is NP-hard.
\end{Theorem}
{\noindent\bf Proof Idea.} 
%It suffices to show that any 3-SAT problem can be reduced to this problem. 
Given a 3-SAT problem with $N$ variables and $M$ Bohr expressions, we can construct a symbol set $\Sigma$, $\alpha,c,d\in\R_{>0}$, a transformer $\F$, and a task $S$ with size $\poly (N, M)$ such that the existence of an $(\alpha,c,d)$-prompt of $\F$ for $S$ is computationally equivalent to this 3-SAT problem.

Based on the proof of Theorem \ref{th-npc}, we can deduce the following result.
\begin{corollary}
\label{cor-11}
$\PM(\F,S,\alpha,c,d)$ is NPC for the single pair task $S=\{(x_1,y_1)\}$.
\end{corollary}

In practice, effective hard prompts can be obtained with optimization \citep{Autoprompt2020,PromptOptimization2024}.
We formulate our problem as the following optimization problem.
\begin{definition}
For a transformer $\F$, task $S$, and $c,d\in\R_{>0}$, the problem $\DPO(\F,S,c,d)$ is to solve the following optimization problem to obtain a prompt $\Prom$:
%$$\argmin_{\Prom\in\Sigma^{*},\len(\Prom)\le c\,\len(S_a)^d} \max_{i\in[N]} \Loss(\F,\Prom\oplus S).$$
%$DPO_1$ means that solve the following optimization problem to obtain a prompt:
$$\argmin_{\Prom\in\Sigma^{*},\len(\Prom)\le c\,\len(S_a)^d}\,\, \Loss(\F,\Prom\oplus S).$$
\end{definition}
%This is the problem for finding the required prompt that minimizes the average value of the \Loss function in the dataset. Then 
As a direct consequence of Theorem \ref{th-npc}, we have
\begin{corollary}    
\label{cor-13}
$\DPO(\F,S,c,d)$ is an NP-hard problem.
\end{corollary}

%{\color{red} 
%Without the length constraint on the prompt, the above results are also valid. The proofs are given in Appendix \ref{app-unbounded}.
%\begin{proposition}
%label{cor-1x}
%$\PM(\F,S,\alpha,\infty,\infty)$ is NPC and
%$\DPO(\F,S,\infty,\infty)$ is NP-hard.
%\end{proposition}
%}

%{\color{red}
%Without the length constraint on the prompt, DPO is also NP-hard.
%\begin{corollary}    
%\label{cor-13}
%The following optimization problem to obtain a prompt is NP-hard
%$$\argmin_{\Prom\in\Sigma^{*}} \max_{i\in[N]}P(\F,\Prom\oplus S).$$
%\end{corollary}
%}

When the greedy sampling strategy is applied to next token prediction—meaning the next token $\sigma$ is chosen as the one that maximizes $\pi_{\F}(\sigma\mid x\oplus y)$—the transformer produces a deterministic output, denoted by $y=\overline{\F}(x)$.
We show that finding prompts in this case is also NPC.
Refer to Appendix \ref{app-greedy} for a more detailed discussion.
\begin{proposition}
\label{cor-15}
For any transformer $\F$, task $S=\{(x_i,y_i)\}_{i=1}^N$, and $c,d\in\R_{>0}$,  deciding whether there exists a prompt $\Prom\in\Sigma^{*}$ such that $y_i = \overline{\F}(\Prom\oplus x_i)$ for $i=1,\ldots,N$ and $\len(\Prom)\le c\,\len(S_a)^d$ is NPC. The result is also true for $N=1$.
\end{proposition}
%In summary, we show that computing optimal hard prompts is an NP-hard problem.

\section{Expressive limitation and capabilities of hard prompt}
\label{sec-5}
%
%\subsection{A situation where $\Prom$ exists}

Prompts can greatly enhance the expressive power of transformers.
%\citep{wei2022chain,pp1,pp2}. 
However, empirical studies show that prompts also have certain limitations 
\citep{pro4,meyer2025memory}. 
In this section, we will provide a theoretical analysis of the capabilities and limitations of hard prompts. Practical implications of our theoretical results in this section are provided in Appendix \ref{sec-54}.
%, and a simple experimental validation is given in Appendix \ref{app-exp}.

\subsection{Polynomial hard prompts are not complete}
\label{sec-50}
In this section, we show that polynomial hard prompts are not complete in that they cannot help a fixed transformer solving all tasks with a given loss.

We begin with a positive result, showing that transformers equipped with prompts can solve most tasks of bounded length, provided that the transformers are designed with knowledge of the query set.
%
%In the following, we show that when using an easier solution pattern, transformers can solve most tasks that satisfy certain conditions with polynomial prompts.
%
\begin{proposition}
\label{prop-g11}
For any query set $S_q=\{x_i\}_{i=1}^N\subset\Sigma^{*}$ satisfying $x_i\ne x_j[1:\len(x_i)]$ for all $i\ne j\in[N]$, and $\alpha>0,L\ge 2,c\ge N\log_T(T+1)+\log_T L+2,d>1$, 
%such that $\len(x_j)\ge\len(x_i)$, 
%then for any $\alpha>0,c\ge N,e\ge 2$, 
there exists a transformer $\F$ that can solve task $S=\{(x_i,y_i)\}_{i\in[N]}$ with an $(\alpha,c,d)$-prompt for any answer set $S_a=\{y_i\}_{i=1}^N\subset \Sigma^{*}$ satisfying $\len(S_a)\le L$.
\end{proposition}
%
%{\bf Proof idea:} The proofs of such proposition utilize the theorem in the work \citep{yu2025analyzing}. By encoding the information of the answer into the prompt, we can make the transformer output the required answer. 

%Proposition \ref{prop-g11} shows that if the transformers are allowed to change according to the query set, they can solve most tasks that have length limitations. 

%\begin{corollary}
%\label{cor-neg1}
%For certain tasks $S$ and $\alpha$, there exist no transformers $\F$ and $c\in\R_{>0},e\ge 1$ such that $\F$ can solve $S$ using a $(\alpha,c,d)$-prompt.
%\end{corollary}
%Is the result true for $\alpha=0$? that is,

%\begin{corollary}
%For certain tasks $S$ and $c,d\in\R_{>0},e\ge 2$, there exist no transformers $\F$ and $(c,d)$-prompts $\Prom$ such that 
%$y=\widehat{\F}(\Prom\oplus x)$ for any $(x,y)\in S$.
%\end{corollary}
%Proposition \ref{g11} shows that if the query set satisfies certain conditions, then there exists a transformer that can solve any answers with a prompt. Proposition \ref{g12} shows that  such a transformer does not exist for a general query set.
The condition $x_i\ne x_j[1:\len(x_i)]$ is easily achievable, %for instance, 
by adding a punctuation mark. More interestingly, if there is no length limitation for the answer set, for some query sets that satisfy the condition in Proposition \ref{prop-g11}, it is impossible to use a prompt to solve all answers with a given loss. 
%In other words, the prompts are not complete from the perspective of loss. 
\begin{Theorem}
\label{negg}
For some query set $S_q=\{x_i\}_{i=1}^N\subset\Sigma^{*}$ satisfying $x_i\ne x_j[1:\len(x_i)]$ for all $i,j\in[N]$  and $\alpha>0$, there exist no transformer $\F$ and $c{>0},d>1$ such that $\F$ can solve task $S=\{(x_i,y_i)\}_{i=1}^N$ with an $(\alpha,c,d)$-prompt for any answer set $S_a=\{y_i\}_{i=1}^N\subset \Sigma^{*}$.
\end{Theorem}
%\end{proposition}
{\bf Proof Idea:} We can construct two queries $x,x'$ such that for any transformer $\F$, the output sequences of $\F(\Prom\oplus x)$ and $\F(\Prom \oplus x')$ will gradually become similar as the length of the prompt $\Prom$ increases, resulting in $\F$ being unable to achieve a small loss in both samples simultaneously.

As a direct consequence, we can demonstrate that polynomial-length hard prompts do not enable transformers to solve all tasks with arbitrarily small loss when given the query set $S_q$.%, as shown below.
%\begin{corollary}
%\label{prop-g12}
%For certain query set $S_q=\{x_i\}$ and $\alpha{>0}$, there exist no transformers $\F$ and $c{>0},d>1$ such that $\F$ solves $S=\{(x_i,y_i)\}_{i=1}^N$ for any answer set $S_a=\{y_i\}_{i=1}^N$ by using an $(\alpha,c,d)$-prompt.
%\end{corollary}
%Theorem \ref{prop-g12} shows that not all tasks can be solved by a long prompt with any loss; that is, prompts are not complete  in this situation. 
%The reason for this phenomenon is that some specifically constructed datasets are difficult for transformers to memorize, similar to \citep{yu2025analyzing}. 
%
%We thus proved Theorem \ref{th-m2}, which states that {\bf prompts are not complete in this setting.} 

\begin{remark}
\label{rem-dr}
Our negative results in Theorem \ref{negg} on hard prompts do not contradict the completeness of prompting \citep{pp1,pp2,ICLR2025-Hu1,pro3} because the settings are quite different.
%The major difference is that the transformers in \citep{pp2,pp1} are allowed to be changed, whereas the transformer is fixed in our case.
More precisely, soft prompts are used in \citep{pp1,ICLR2025-Hu1}, which allow for changing the embedding layer; the transformer in \citep{pp2} is constructed as a certain Turing complete transformer similar to \citep{perez2021attention}; 
and the transformer in \citep{pro3} is also a constructed one with 7-layers.
%There are other differences, such as next-token prediction strategies, prompt lengths, and model structures.
%, which requires more components, similar to \citep{perez2021attention}.
%Also, prompt embeddings in \citep{pp1} cannot lead to exact prompts for a fixed transformer because the embedding space is continuous and the prompt space is discrete.
%%
%Soft prompts are used in \citep{pp1}, which allow for changing the parameters of the embedding layer, while we consider hard prompts and the transformer is fixed.
%In \citep{pp2}, the transformer is constructed, while our transformer is given and fixed. 
%
%{\bf Compare with the previous work.} There have been some previous works proving that prompt are turing complete such as \citep{pp1,pp2}, our conclusion does not conflict with theirs. The main point is the difference in settings, including differences in transformer structure details, output methods and prompt type. For example, \citep{pp1} need the soft prompt but we use the hard prompt; the \citep{pp2} need a  tokenize and a readout to convert the input and output to adapt turing machine, but we do not consider the turing machine; hence, they consider the deterministic output and no mask transformer.
\end{remark}

\subsection{Short hard prompt does not help solving long task}
\label{sec-51}
%The short prompt means that $e<1$ in the limitation of prompt. In such case, when the answer is too long, the prompt is quite shorter than answer. This situation is very common in practice, such as constant length prompt.

In this section, we will show that when the length of the prompt is shorter than that of the answer, that is $d\in(0,1)$, the prompt has limited ability, as shown below. 
%This situation is very common in practice, such as ``You are a mathematician. Try hard to solve the problem.'' We will show that if a short prompt is used, it is impossible to achieve a very long answer for any transformer. 
%Firstly, we have the following result.

\begin{Theorem}
\label{th-51}
For any transformer $\F$, $c{>0}$, and $\epsilon,d\in(0,1)$, there exists an $L\in\Z_{>0}$ such that if a task $S=\{(x_i,y_i)\}_{i=1}^N$ satisfies $\len(y_i)\ge L$ for all $i\in[N]$ and $\F$ 
cannot solve $S$ with loss $(\len(S)+1)\alpha$ without a prompt, then $\F$ also cannot solve $S$ using a $(\max\{(\len(S)+1)(\alpha-\epsilon),0\},c,d)$-prompt. 
%{\color{red}cannot solve $S$ with loss $\alpha$ without a prompt, then $\F$ also cannot solve $S$ using a $(\alpha-\epsilon,c,d)$-prompt. }
\end{Theorem}
{\bf Proof Idea.}
We will show that $||\F(\Prom\oplus x \oplus y[1:i])-\F(x \oplus y[1:i])||_2\propto \frac{\len(\Prom)}{i^d}$ and then use it to estimate the difference between loss functions for using a prompt and without a prompt. 
%This approach can only prove the existence of $L$, and the specific value of $L$ is influenced by the parameters and structure of $\F$, making it difficult to calculate in detail.

%By Theorem \ref{th-51}, when a transformer can only solve the task that has a long answer ($\len(S)\ge L$) with poor loss ($\alpha\gg\epsilon$), it is also unable to solve such a task using short prompts much more effectively. This suggests that a short prompt does little to enable the transformer to solve tasks that it is otherwise unable to solve.
By Theorem \ref{th-51}, if a transformer is hard to solve a task, i.e. $\alpha\gg0$, then it cannot solve such a task using short prompts with a smaller loss. 
Intuitively, this suggests that {\bf a short prompt does little to enable the transformer to solve long tasks that it is otherwise unable to solve without a prompt.}

To better understand Theorem \ref{th-51}, we consider the single data situation and use the length normalized loss function $\L_{\text{Nnll}}(\F,x,y)=\L_{\text{nll}}(\F,x,y)/(\len(y)+1)$ \citep{yang2020predicting}. We have the following corollary
\begin{corollary}
\label{th1-b}
For any transformer $\F$ and $\epsilon,c{>0},d\in(0,1)$, 
there exists an $L\in\Z_{>0}$ such that if a task $S=\{(x,y)\}$ satisfies $\len(y)\ge L$, then $|\L_{\text{Nnll}}(\F,\Prom\oplus x,y)-\L_{\text{Nnll}}(\F, x,y)|\le \epsilon$ for any $(c,d)$-prompt $\Prom$.
%
%and $\F$ cannot solve $S$ with loss $\alpha$ (under $L_{\text{Nnll}}$) without a prompt, then $\F$ cannot solve $S$ using a $(\max\{(\alpha-\epsilon),0\},c,d)$-prompt (under $L_{\text{Nnll}}$).
%then $\L_{\text{Nnll}}(\F,x_1,y_1)\ge\alpha$  implies that $\L_{\text{Nnll}}(\F,\Prom\oplus x_1,y_1)\ge\alpha-\epsilon$ for any  $(c,d)$ prompt $\Prom$.
\end{corollary}

This corollary illustrates the reason behind Theorem \ref{th-51}: for any query $x$, when the answer is much longer than the prompt, the average loss will not change significantly by using a short prompt, which implies that even if the prompt has a strong impact on the early parts of the output, the effectiveness of the prompt decreases in the later portions of the output.

Since Theorem \ref{th-51} only stands for long answers, it is natural to ask: is there an $\F$ that can solve any task with a fixed length using short prompts? The answer is no, as demonstrated below.  
%For any given coefficient $c,\eta,\alpha$, length $L$ and query $x$, can a transformer $\F$ satisfy: for any answer $y$ such that $\len(y)\le L$, $\F$ can solve $(x,y)$ with loss $\alpha$ by a prompt with length $c,\eta$? The answer is no, :

\begin{proposition}
\label{prop1}
For any token  set $\Sigma$ of size $T$, transformer $\F$, query $x\in \Sigma^{*}$, and $\alpha>0,c>0, d\in(0,1),L\in\Z_+$ satisfying $L>({\log_T(T+1)c}{+\alpha})^{\frac{1}{1-d}}$, there exists a $y\in \Sigma^{*}$ with $\len(y)=L$ such that $\F$ cannot solve $\{(x,y)\}$ using an $(\alpha,c,d)$-prompt.
\end{proposition}

This shows that even if the length of the answer does not depend on $\F$, it is still impossible to solve some tasks using a short prompt with a small loss when $L$ is longer than a threshold independent of $\F$. This is because the number of selectable short prompts is fewer than the number of answers, which naturally leads to situations where solving arbitrary answers with short prompts is not possible.

\subsection{Long prompt has “Prompt Dominating Answer Phenomenon”}
\label{sec-52}
For long prompts, that is $d> 1$, the limitations of short prompts, such as those in Theorem \ref{th-51}, no longer exist because long prompts significantly enhance their impact on the final result. 
However, in this section, we show that long prompts still have certain limitations: 
if the prompt is too long, then the output of the transformer no longer focuses on the task itself but is primarily derived from the prompt, or equivalently, the prompt ``dominates'' the answer. We refer to this as the “Prompt Dominating Answer Phenomenon.”
We consider the single data situation. 

\begin{Theorem}  
\label{cxxx}
For any transformer $\F$, $\epsilon>0$, there exists $c>0,d>1$ such that for any $x,x',y$ satisfying $\len(x)=\len(x')$, $\last(x)=\last(x')$, and $\len(y)> \len(x)$, if $\F$ can solve $\{(x,y)\}$ with a loss $\alpha$ using a prompt $\Prom$ with a length greater than $c\len(y)^d$, then $\F$ can also solve $\{(x',y)\}$ with a loss $\alpha+\epsilon$ by using the same prompt $\Prom$.
\end{Theorem}
\textbf{Proof Idea}
We will show $||\F(\Prom\oplus x \oplus y[1:i])-\F(\Prom\oplus x' \oplus y[1:i])||_2\propto \frac{\len(x)+i}{\len(\Prom)}$ at first. Hence, we can 
construct $c,d$ such that $|\L_{\text{nll}}(\F,\Prom\oplus x,y)-\L_{\text{nll}}(\F,\Prom\oplus x',y)|\le \epsilon$ for any prompt satisfying $\len(\Prom)\ge c(\len(x)+\len(y))^d$, which leads to the theorem directly.

{\bf The theorem shows that when the prompt is much longer than the answer, it yields the same answer with high probability for all queries with the same length and last token.} 
%In other words, the prompt ``Dominates'' the answer. 
%This is a behavior we do not want to occur in practical applications, as it results in the transformer effectively ignoring the query. 
%
Intuitively, this result is reasonable because when the prompt is too long, the information from the query is obscured, causing the transformer to focus too much on the prompt and ignore the information in the query. 
%In fact, when $e>2$, it is enough to lead the ``remembers'' the answer for prompt when the answer is too long.

%For long answers similar to those in Section \ref{sec-51}, the following result shows that if the answer is much longer than the query and $d>2$, then the same phenomenon occurs.
If  the answer is much longer than the query, then we can provide a stronger result for any $d>2$.
\begin{proposition}
\label{lj}
For any transformer $\F$, $\epsilon,c>0,d>2$, there exists an $L$ such that for any $x,x',y$ satisfying $\len(x)=\len(x')$, $\last(x)=\last(x')$, and $\len(y)\ge L+\len(x)$, if $\F$ can solve $\{(x,y)\}$ with a loss $\alpha$ using a prompt $\Prom$ with a length greater than $\Prom\ge c\len(y)^d$, then $\F$ can also solve $\{(x',y)\}$ with a loss $\alpha+\epsilon$ by using the same prompt $\Prom$.
\end{proposition} 

For the loss function $\L_{\text{Nnll}}$, the above result is true for $d>1$. 
\begin{corollary}
\label{lj1}
Under $\L_{\text{Nnll}}$ loss, Proposition \ref{lj} holds for any $d>1$ when other conditions remain unchanged.
%
%for any transformer $\F$, $\epsilon,c>0,d>1$, there exists an $L$ such that for any $x,x',y$ satisfying $\len(x)=\len(x')$, $\len(y)\ge L+\len(x)$, if $\F$ solves $\{(x,y)\}$ using an $(\alpha,c,d)$-prompt, then $\F$ also solves $\{(x',y)\}$ with loss $\alpha+\epsilon$ using the same prompt.
\end{corollary}
%
%This proposition shows that when $e>2$, using such a prompt to solve the task must lead to the ``remembers'' when the answer is too long. 
Both Theorem \ref{cxxx} and Proposition \ref{lj} can be used to draw conclusions for multiple samples with long prompts. Here, we present one of them.
%, based on the above result, we have the following negative finding regarding the expressive ability of the long prompt.

\begin{proposition}
    \label{cx1}
For any transformer $\F$, $\epsilon>0$, there exist $c>0,d>1$ such that for any task $S=\{(x_i,y_i)\}_{i=1}^N$ satisfying $\len(y_i)>\len(x_i)$ and $y_i\ne y_j$, $\last(x_i)=\last(x_j)$ for all $i\ne j\in[N]$, it holds $\L(\F,\Prom\oplus S)\ge\sum_{i\in\Z_+}\frac{N_i\ln N_i}{N}-\epsilon$ for any prompt $\Prom$ with a length of more than $c\len(S)^d$, where $N_i$ is the number of queries in $S$ with a length of $i$.
\end{proposition}

As a consequence of Proposition \ref{cx1}, if $S$ satisfies  $\len(x_i)=\len(x_j)$ for all $i,j\in[N]$, then it is easy to see that $\L(\F,\Prom\oplus S)\ge \ln N-\epsilon$. Thus, it is not possible to solve such a task with a loss of about $\ln N$ by using a long prompt. 

 Theorem \ref{cxxx} and Proposition \ref{cx1} reveal an important distinction between prompting and training. Training obeys the scaling law \citep{Kaplan2020ScalingLF,Hoffmann2022TrainingCL}, meaning that more parameters and more training samples lead to better accuracy. On the other hand,  Theorem \ref{cxxx} and Proposition \ref{cx1} imply that even a very long prompt is not enough to solve multiple tasks with long answers when using a fixed transformer.
 In other words, \textbf{no scaling law exists for hard prompts}.

\subsection{Linear prompt}
\label{sec-53}
%\paragraph{Linear-Prompt.}
In this section, we will consider linear prompts; that is $d=1$.
In this case, the length of the prompt and that of the answer will remain in a fixed ratio $c$.
We show that the main limitations in Sections \ref{sec-51} and \ref{sec-52} do not occur in this case, which illustrates certain advantages of linear prompts.

\paragraph{Linear prompt enhances the power of transformers in certain cases.}
We first show that Theorem \ref{th-51}  no longer holds, as demonstrated below.
\begin{proposition}
\label{pl1}
For any $\epsilon,c>0,d=1$ and  $\alpha_1<\alpha_2$, there exists a transformer $\F$ such that for any given $L$, there exists a task $S=\{(x,y)\}$ such that $\len(y)>L$ and $\F$ cannot solve $S$ with loss $\alpha_2(L+1)$ without a prompt, but $\F$ can solve $S$ with loss $\alpha_1(L+1)$ using a $(c,d)$-prompt. 
\end{proposition}

In the case of $\alpha_1\ll\alpha_2$, this proposition shows that with a $\Omega(\len(y))$ length prompt, the transformer can enhance its ability to solve certain tasks with linear answers. 
%This also indicates that as the length of the answer increases, the prompt must also demonstrate a linearly proportional increase to ensure the effective improvement of output loss.

But when $c\ll 1$, the prompt is also much shorter than the task, and it is easy to see that Proposition \ref{prop1} still holds in this case. Following Proposition \ref{prop1}, we have
 
\begin{corollary}
\label{zss}
For any token  set $\Sigma$ of size $T$, transformer $\F$, query $x\in \Sigma^{*}$, and $\alpha,c>0,d=1,L\in\Z_+$ satisfying $c<1/\log_{T}(T+1)$ and $L>\frac{\alpha}{1-c\log_{T}(T+1)}$, there exists a $y\in \Sigma^{*}$ with $\len(y)=L$ such that $\F$ cannot solve $\{(x,y)\}$ using  an $(\alpha,c,d)$-prompt.
\end{corollary}

%In this subsection, we consider the situation of $e=1$. In this case, the length of the prompt and the length of the answer will always remain in a fixed ratio, mainly depending on $c$. 

\paragraph{Prompt dominating answer phenomenon does not always hold for linear prompts.} 
We show that a linear prompt with a certain $c$ can help alleviate the prompt dominating answer phenomenon. 

\begin{proposition}
\label{pl2}
   For any $\beta>0,d=1$, there exist a transformer $\F$ and a constant $c$ such that for any length $L>0$, there exist queries $x,x'\in\Sigma^{*}$ and an answer $y$ satisfying $\len(x)=\len(x')$, $\last(x)=\last(x')$, $\len(y)\ge L+\len(x)$ such that $\F$ can solve $\{(x',y)\}$ with certain $\alpha$ by a prompt with length more that $c\len(y)$, but $\F$ cannot solve $\{(x,y)\}$ with $\alpha+\epsilon$ by the same prompt.
\end{proposition}

This proposition shows that we can avoid the prompt dominating problem in Theorem \ref{cxxx} and Proposition \ref{lj} by using a linear prompt with a certain transformer and $c$. For the $\L_{\text{Nnll}}$ loss, we have:
\begin{proposition}
\label{pl3}
    When considering the $\L_{\text{Nnll}}$ loss, Proposition \ref{pl2} still stands.
\end{proposition}

\begin{remark}
Please note that Propositions \ref{pl1} and \ref{pl2} only show that a specific transformer can avoid such problems in Theorems \ref{th-51} and \ref{cxxx}, but not all transformers. This is reasonable because, in certain extreme cases such as $\F\equiv 0$, Propositions \ref{pl1} and \ref{pl2} must not hold.
\end{remark}

\section{Generalization of prompt}
\label{sec-gen}

An important issue with prompt engineering is whether prompts for a finite task can be effectively generalized to new samples. In this section, we provide a necessary and sufficient criterion for this kind of generalization.
Note that the results in this section hold for arbitrary prompts and are not limited to polynomial ones.
%the condition under which show that prompts for an iid dataset can be generalized to the entire data distribution.

We adopt the usual setting for deriving generalization bounds \citep{mohri2018foundations}.
Let $\D$ be a distribution over $\Sigma^{*}\times\Sigma^{*}$ and $S=\{(x_i,y_i)\}_{i=1}^N$ a dataset sampled iid according to $\D$, that is, $S\sim \D^N$. 
Let 
$$\A_{\F,\Prom, S,\alpha}= \frac{1}{N} \sum_{i=1}^N \ID(\L_{\text{nll}}(\F,\Prom\oplus x_i,y_i)\le\alpha)$$ be the probability of $\Prom$ being an $\alpha$-prompt of $\F$ for $S$, and
$$\A_{\F,\Prom,\D,\alpha}=\E_{(x,y)\sim \D}[\ID(\L_{\text{nll}}(\F,\Prom\oplus x,y)\le\alpha)]$$
be the  probability of $\Prom$ being an $\alpha$-prompt of $\F$ for the entire data distribution.
Then the gap $|\A_{\F,\Prom,\D,\alpha}-\A_{\F,\Prom, S,\alpha}|$
can be used to measure the generalizability of the prompt $\Prom$.
%
%Then, for a given prompt $\Prom$, the gap 
%$$\A_{\F,\Prom,\D,\alpha}-\A_{\F,\Prom, S,\alpha}=\E_{(x,y)\sim D}[\ID(\L_{\text{nll}}(\F,\Prom\oplus x,y)\le\alpha)]-\E_{(x,y)\in S}[\ID(\L_{\text{nll}}(\F,\Prom\oplus x,y)\le\alpha)],$$
%can be used to measure the generalization of the accuracy of the prompt. 
%%
%If the gap is small, then prompts that can achieve the target on a finite set can also achieve the desired target on the entire distribution. 
We now provide a uniform generalization bound for prompting.

\begin{Theorem}
\label{th-genb}
For any token  set $\Sigma$ of size $T$, data distribution $\D$, transformer $\F$, $\alpha>0$, and $N\in\Z_+$, 
with a probability $1-\delta$ of $S\sim \D^N$,  the following generalization bound holds
\begin{equation*}
%\begin{array}{ll}
|\A_{\F,\Prom,\D,\alpha}-\A_{\F,\Prom, S,\alpha}|\le \frac{\sqrt{8L\ln (T+1)}+\sqrt{0.5\ln1/\delta}}{\sqrt{N}},
\end{equation*}
for any prompt $\Prom$ satisfying $\len(\Prom)\le L$.
\end{Theorem}
{\bf Proof Idea.} We treat the set of $\F(\Prom\oplus x)$ as a hypothesis space that uses  $\Prom$ as the  parameters. Then we can use the generalization bound based on Rademacher complexity \citep{mohri2018foundations} on such a hypothesis space to prove the theorem.

%Note that the above bounds do not depend on $\F$ and $\D$. 
This result provides practical guidance for generalization: if the size of the dataset is significantly larger than the length of the prompts, that is $L=o(N)$, then performance on $S$ can be generalized to the entire distribution. 
We can make this explicit in the following corollary.
\begin{corollary}
\label{cor-61}
For any $\epsilon\in(0,1)$, if $N\ge (\sqrt{8\len(\Prom)\ln (T+1)}+\sqrt{0.5\ln 1/\delta})^2/\epsilon^2$, then with a probability $1-\delta$ of $S\sim \D^N$,  we have 
$|\A_{\F,\Prom,\D,\alpha}-\A_{\F,\Prom, S,\alpha}|\le \epsilon$.
\end{corollary}

The following theorem shows that the above generalization bound is tight in terms of the main quantities $N$ and $\len(\Prom)$.
\begin{Theorem}
\label{prop-gen1}
For any token  set $\Sigma$ of size $T$, $N,L\in\Z_+$, $\alpha\in\R_+$, $\epsilon\in(0,1)$, if $L\ge \frac{N(2-\ln\epsilon)}{\ln T}$, then there exists a transformer $\F$ and a distribution $\D$ such that, with a probability $1-\epsilon$ of $S\sim D^N$, 
%there exists a $\Prom$ such that 
it holds $|\A_{\F,\Prom,\D,\alpha}-\A_{\F,\Prom,S,\alpha}|\ge 1-\epsilon$ for some prompt $\Prom$ satisfying $\len(\Prom)= L$.
%{\color{red} Is it $\len(\Prom)\ge O(N\ln (N/\epsilon)/\ln(T))$? Otherwise, it contradicts to Corollary 6.2?}
%it holds $|\A_{\F,\Prom,\D,\alpha}-\A_{\F,\Prom,S,\alpha}|\ge 1-\frac{N}{e^{\frac{L \ln(T)}{fN}}}$ for some prompt $\Prom$ satisfying $\len(\Prom)\le L$, where $f$ is a constant.
\end{Theorem}
{\bf Proof Idea.} 
%We first construct a distribution such that with high probability of randomly sampled dataset $D$ with $N$ samples in it corresponds to a prompt $\Prom$, and then construct a transformer such that $\Prom$ has poor generalization to $D$. 
The distribution $\D$ is defined on a finite set $\mathcal{S}$ 
with more than $[N/\epsilon]$ elements. Then every $S_k\sim\D^N$ has a distinct prompt $\Prom_k$ with the given length.
Using the theory of memorization \citep{mahdavi2023memorization,yu2025analyzing}, we construct an $\F$ that can solve $\Prom_k\oplus S_k$ with loss $\alpha$ but cannot solve $\Prom_k\oplus(\mathcal{S}\setminus S_k)$ with loss $\alpha$. Now, the theorem derives from the fact $|S_k|/|\mathcal{S}|<\epsilon$.
%Please note that $\F$ and $\D$ in the theorem are not arbitrary but specifically constructed.
%
%Please note that $\F$ and $\D$ in the theorem are not arbitrary, if $\D$ is only defined on a few points, or $\F\equiv0$, then any prompt can be generalized. We consider the worst-case scenario where non generalization may occur in here.

Theorem \ref{prop-gen1} shows that, for prompts with a length greater than $\Omega(N)$, their effect on $S$ may not generalize to $\D$. 
By Corollary \ref{cor-61}, a sufficient condition for such a generalization is $N\ge \overline{O}(L)$. {\bf We thus obtain a tight bound:}  
%for $N$ in terms of the length of the prompt.
\begin{corollary}
A necessary and sufficient condition for the performance of  $\Prom$ on $S\sim\D^N$ to generalize on $\D$ for any transformer and distribution $\D$ is $N\ge \overline{O}(\len(\Prom))$, when omitting other small quantities.
\end{corollary}

%Please note that $\F$ and $\D$ in the theorem are not arbitrary; if $\D$ is only defined at a few points or $\F\equiv0$, then any prompt can be generalized. We consider the worst-case scenario where non generalization may occur here.
%This phenomenon is similar to ``overfitting'' in neural networks: when prompts are too fixated on improving accuracy on $S$, they lose generalizability. 

%\begin{remark}
%    
%\end{remark}
%This proposition tells us that the length of the prompt must not exceed the number of samples in order to ensure the generalizability of the prompt,   which is consistent with the upper bound in Theorem \ref{th-genb} when ignoring the small quantities.
%$$\L_{\F,\Prom,\D}-\L_{\F,\Prom\oplus S}=\E_{(x,y)\sim D}[\L_{\text{nll}}(\F,\Prom\oplus x,y)]-\L(\F,\Prom\oplus S),$$
%the answer is not stand:
%\begin{proposition}
%\label{th-genbx}
%For any $\Sigma=\{\sigma_i\}_{i=1}^T$ and $\alpha>0,N,L\in\Z_+$, there exist a transformer $\F$ and distribution $\D$ such that: with a probability of $1-\delta$ of $S\sim \D^N$, there exist a $\Prom\in \Sigma^{*}$ satisfying $\len(\Prom)\le L$, such that: $|\L_{\F,\Prom,\D}-\L_{\F,\Prom\oplus S}|\ge \alpha$. 
%\end{proposition}
%This is because $\L$ does not have an upper bound.

\section{Conclusion}
\label{conc}
This work establishes the first systematic theoretical framework for hard (discrete) prompts in LLMs.
We settle three key theoretical problems: computational complexity, inherent expressiveness limits,
and tight generalization bounds. Our results formally explain why short prompts fail for long tasks and
why long prompts cause answer domination.
%, and why linear-length prompts are practically optimal.
These findings unify empirical phenomena in prompt engineering and provide provable guidance
for practical LLM deployment.

%This paper aims to establish the theoretical framework of prompt engineering by focusing on three basic questions for hard prompts: computational complexity, expressive limitations, and generalizability. We show that deciding the existence of hard prompts with a given loss is NP-complete, and finding the optimal ones is NP-hard.
%We also prove that hard prompts are not complete; short prompts do not enhance the ability of transformers to solve long tasks with high probability;  long prompts demonstrate the ``Prompt dominating Answer Phenomenon;'' and linear prompts do not have these limitations.
%Finally, a tight generalization bound for prompts has been proved. 

{\bf Limitation and future work.}
While the results in Sections 4 and 6 are quite complete, the results in Section 5 mainly discuss the limitations of hard prompts. While limitations are indeed significant and informative, an equally important yet challenging task is to establish sufficient conditions and develop construction methods for hard prompts to enhance the capability of transformers.

%
%can be further extended: providing an estimation for $L$ in Theorems \ref{th-51} and \ref{cxxx}; determining whether Proposition \ref{lj} holds for any $d>1$; extending our results in Section 5 to other output sampling strategies, and, most importantly, giving sufficient conditions and construction methods for hard prompts to enhance the power of transformers. 

%Also, this is a theory paper and more experiments
%Moreover, , the calculation of $L$ in theorem \ref{th-51} and \ref{cxxx} is a hard problem.
%Most of our results in Section 5 discuss the limitations of prompts. It is interesting to provide positive results, such as conditions for $S$ to exist as a prompt with any loss.

%\bibliographystyle{plain}
\bibliography{refs}
\bibliographystyle{iclr2027_conference}
\newpage
\appendix
%\onecolumn

%\section{Statements}
%\label{staa}
%\subsubsection*{Impact Statement}
%
%This paper focuses on developing theoretical foundations to improve the interpretability of prompt engineering, and it does not involve any ethical or moral concerns.
%This paper presents work whose goal is to advance the field of machine learning. 
%There are many potential societal consequences of our work, none of which we feel must be specifically highlighted here.
%
%\subsubsection*{Reproducibility Statement}
%
%Our theorems have been rigorously proven. The experiments use relatively small and open-source models, which are easy to replicate. The theorems, their proofs, and the experiments are all carried out without using an LLM.

\section{Details of Setting}
\label{xjj}

\paragraph{Notation.}
In this paper, we use $O(A)$ to mean a value not greater than $cA$ for some constant $c$, and $\overline{O}$ to mean that small quantities, such as logarithms, are omitted. We use $\Omega(A)$ to mean a value not less than $cA$ for some constant $c$.
%, and $\overline{\Omega}$ or $\overline{O}$ to mean that small quantities are omitted.  %We say for all $(x,y)\sim\D$  there is event A stand means that $\Pr_{(x,y)\sim \D}(A)=1$.

\paragraph{Positional Encoding.} 
{\em Relative position embedding} is used: $R(xQ_i,K_ix^\tau)\in\R^{n\times n}$, where the $(u,v)$-th weight of $R(xQ_i,K_ix^\tau)$ is $x[u]Q_iA_{u-v}K_ix^\tau[v]$, where $x[i]$ is the $i$-th row of $x$ and $A_{u-v}\in[-1,1]^{W\times W}$ is the matrix for the relative position embedding. Note that $A_{u-v}$ is free of $i$.
%
%\begin{remark}
This definition aligns with most relative position embeddings, such as RoPE \citep{su2024roformer}.

\section{Proofs in Section \ref{sec-npc}}
\label{sec-npc-app}

\subsection{Scope of application of the conclusion}
\label{sbsgr}

Our proof relies on standard complexity-theoretic techniques, specifically by constructing a reduction from an NP-complete problem to the target problem.

Our proof is applicable to the hypothesis space we defined and all hypothesis spaces containing the transformers constructed in our proof. For other hypothetical spaces $\H$, if for any $\F$ used in the proof, there is a $\F_h\in\H$ such that $\F_h=\F$ on the constructed point, and $\F_h$ has at most polynomial multiples scale of $\F$, then our proof is applicable on $\H$.

This point is not strict, because in order to complete any task by transformer, we need the transformer class to have adequate expressive capacity. And considering that we only use a single attention layer transformer with single structure in our proof, so it is not hard to approximate it.

On the other hand, our results \ref{cor-13} and \ref{cor-133} are strictly applicable to globally optimal prompt strategies. Since most gradient-based prompt strategies in practice merely converge to a local optimum, our conclusions do not extend to such empirically derived locally optimal strategies.

\subsection{Proof of Theorem \ref{th-npc}}
%We repeat the theorem here.
%\setcounter{theorem}{\numexpr\th-npc-app}
%\setcounter{theorem}{3}

\begin{definition}[The 3-SAT decision problem]
\label{def-3sat}
 Let $\{z_i\}_{i=1}^N$ be Boolean variables and $\Phi=\wedge_{j=1}^M\phi_j$, where $\phi_j=\vee_{k=1,2,3} z'_{j_k}$ is a Boolean clause and $z'_{j_k}\in\{z_i\}_{i=1}^n\cup \{\neg z_i\}_{i=1}^n$. The 3-SAT problem is to decide whether there exist values for the variables under which $\Phi=1$.
\end{definition}

Now we prove Theorem \ref{th-npc} and Corollary \ref{cor-11} together.
\begin{proof}
Firstly, for a given prompt $\Prom$, it is easy to check whether $\Prom$ is a solution such that $\PM(\F,S,\alpha,c,d)=1$ in polynomial time based on the size of $\Sigma,S,\alpha,c,d$ and the parameters of $\F$. Therefore, $\PM$ is an NP problem. 

To prove Theorem \ref{th-npc}, it suffices to show that an NP-complete problem (here, we use 3-SAT) can be reduced to $\PM(\F,S,\alpha,c,d)=1$.

Let $\Phi=\wedge_{j=1}^M\phi_j$ be a 3-SAT problem as defined in Definition \ref{def-3sat}. Without loss of generality, we need $\Phi$ to satisfy the conditions that: $N,M\ge 6$ and each $z_i$ is in at least one of $\phi_j$, which implies $N\le 3M$. It is obvious that all SAT problems can be reduced to this kind of situation.

We first construct a $\Sigma$, a dataset $S$ and a transformer $\F$ based on $\Phi$.

{\bf Step 1. Construction of $\Sigma$, $\F$, $S$ and $\alpha,c,d$.}
%
%Then we will construct the following symbol set $\Sigma$, dataset $S$ and transformer $\F$.

To prove the (1) in Corollary \ref{cor-11} together, we will consider a situation in which $|S|=1$.

The symbol set is $\Sigma=\{\sigma_{i}\}_{i=1}^{N+M}$ and the dataset is $S=\{(\sigma_{1+N},\sigma_{2+N}\oplus\sigma_{3+N}\oplus\dots\oplus\sigma_{M+N})\}\subset \Sigma^{*}\times \Sigma^{*}$, where the query set $\{\sigma_{1+N}\}$ only contains one sequence with a single symbol $\sigma_{1+N}$ in it. 

Let $\alpha=\ln 2$, $c\ge 3N$, and $d\ge1$. %denote the lower bound of $c,d$; it also stands for $c,d=\infty$, which can prove (2) in Corollary \ref{cor-11}. 
The transformer is defined as follows.

(1) The embedding layer: Each $\sigma_i$ will be embedded into a $4(N+M)$-dimensional vector $v_{\sigma_i}$, such that for the $i\in[1,N+M]$, the $(4i-3)$-th to $4i$-th weights in the $v_{\sigma_i}$ are 1, and other weights are $0$. For a $x\in \Sigma^{*}$, let $v(x)$ be its embedding matrix. And the position matrix $A_{i}$ is an identity matrix $I$ for any $i\in N_+$.

(2) The hidden layers: In the first hidden layer $\F_1(x)$, the attention layer is defined as: the position embedding is $I$ and $Q_1K_1=I$, $V_1=I$. The feedforward layer $\FNN_1(x)$ in the first layer is defined as $\FNN_1(x)=0$. 

Then, based on the definition of the embedding layer, the output of the first layer when the input is $x$ is $\F_1(x)=v(x)+\softmax(4I_{\len(x)}+M)v(x)$. Hence, its last row is $v_{x[\len(x)]}+\sum_{i=1}^{\len(x)}\frac{v_{x[i]}}{\len(x)}$. Name it as $N(x)$, and let $N_i(x)$ be the $i$-th weight of $N(x)$.
%$$=(\frac{\Num_{\sigma_1}(x)}{\len(x) },\dots,\frac{\Num_{\sigma_{N+M}}(x)}{\len(x)})+v_{x[\len(x)]},$$
Before defining the next layer, we first define 3 kinds of Relu networks $G_i(x):\R\to\R$:

$G_1(x)=\Relu(6N^2(x-1/N))+x;$ and $G_2(x)=x+\Relu((N+M+3)(1/(N+M+2)-x))$, $\G_3(x)=\Relu(-x+1/(N+M+2))+x$.
%This is a function such that: $F_1(1/N)=1/N$, and $F_1(1/(N-0.5))>3$, and $F_1(x)$ is increased with $x$.
%$$F_2(x)=\Relu(3N^2(x-1/N))+x;$$

%This is a function such that: $F_2(1/N)=2.1/N$, and $F_1(0)=-1$, and $F_1(x)\le2.1/N$ for any $x$.

Then we define the next hidden layer, which is also the last hidden layer. Let the parameters in the attention layer all be 0, and the FNN layer makes the output of the last hidden layer $\F_l(x)$ to be: 

\begin{equation*}
    \begin{array}{ll}
\F_l(x)=& (N_4(x),\G_1(N_4(x)),\G_2(N_4(x)),\G_3(N_4(x)),\\
&N_8(x),\G_1(N_8(x)),\G_2(N_8(x)),\G_3(N_8(x)),\\
&        \dots,N_{4{(N+M)}}(x),\G_1(N_{4{(N+M)}}(x)),\G_2(N_{4{(N+M)}}(x)),\G_3(N_{4(N+M)}(x))).\\
    \end{array}
\end{equation*}

Notice that in the $N(x)$, there are $N_{4i-3}(x)=N_{4i-2}(x)=N_{4i-1}(x)=N_{4i}(x)=\frac{\Num_{\sigma_i}(x)}{\len(x) }+I({x[\len(x)]}=\sigma_i)$, so it is easy to implement the above output by $\FNN$ layer. In this context, $\Num_{\sigma_i}(x)$ is the number of $\sigma_i$ in the sentence $x$.

(3) The output layer:
The first $N+1$ elements of the transformer output are $0$. For $j\in[M]$, based on the different forms of $\phi_j$, the $(N+1+j)$-th element of the transformer output has a different form (3.1)-(3.4), which is defined below (for convenience, let the $(N+M+1)$-th element of the transformer be the probability for $\sigma_0$).

(3.1) $M(\G_1(N_{4i'}(x))+\G_1(N_{4j'}(x))+\G_1(N_{4k'}(x))-3/N)$ 

$-100M^2\sum_{i\in[N+1,M+N],i\ne N+j,}N_{4i}(x)$ if $\phi_j=z_{i'}\cup z_{j'}\cup z_{k'}$ for some $i',j',k'\in[N]$;

(3.2) $M(\G_1(N_{4i'}(x))+\G_1(N_{4j'}(x))-2.1(1/N-\G_2(N_{4k'}(x))+\G_3(N_{4k'}(x)))+0.1/N)-100M^2\sum_{i\in[N+1,M+N],i\ne N+j}N_{4i}(x)$ if $\phi_j=z_{i'}\cup z_{j'}\cup \neg z_{k'}$ for some $i',j',k'\in[N]$;

(3.3) $M(\G_1(N_{4i'}(x))-2.1(1/N-\G_2(N_{4j'}(x))+\G_3(N_{4j'}(x)))-2.1(1/N-\G_2(N_{4k'}(x))+\G_3(N_{4k'}(x)))+0.1/N)-100M^2\sum_{i\in[N+1,M+N],i\ne N+j}N_{4i}(x)$ if $\phi_j=z_{i'}\cup \neg z_{j'}\cup \neg z_{k'}$ for some $i',j',k'\in[N]$;

(3.4) $M(6.2/N-2.1(1/N-\G_2(N_{4i'}(x))+\G_3(N_{4i'}(x)))-2.1(1/N-\G_2(N_{4j'}(x))+\G_3(N_{4j'}(x)))-2.1(1/N-\G_2(N_{4k'}(x))+\G_3(N_{4k'}(x))))$
$-100M^2\sum_{i\in[N+1,M+N],i\ne N+j}N_{4i}(x)$ if $\phi_j=\neg z_{i'}\cup \neg z_{j'}\cup \neg z_{k'}$ for some $i',j',k'\in[N]$.

It can be easily verified that this transformer can indeed be realized with $O(\poly(M,N))$ parameters and requires $O(\log(\poly(M,N)))$ bits for each parameter value.

{\bf Step 2. We show that if  $\PM(\F,S,\alpha,c,d))$ is decidable, then the 3-SAT problem $\Phi$ is also decidable.} 
%The theorem follows from the step.

To be convenient, write $2.1(1/N-\G_2(x)+\G_3(x))=2.1(1/N-(N+M+2)Relu(1/(N+M+2)-x))=G_4(x)$. The proof for this step has three parts.

%Then, based on the above definition, we show how it can be corresponding to the 3-SAT problem $\Phi$.

%If we have a polynomial algorithm that can solve the $\PPF$ problem, then we can use such an algorithm to solve the $\PPF$ problem defined on the above transformer $\F$ and dataset $S$, then we have that: 

{\bf Part 1. If  $\PM(\F,S,\alpha,c,d)=1$, then the 3-SAT problem $\Phi$ has a solution.}

Assume that $\PM$ problem has a solution $\Prom$, then we can find the following solution for the 3-SAT problem $\Phi$: $z_i=1$ if $\Num_{\sigma_i}(\Prom)/(\len(\Prom)+M)\ge 1/N$ where $i\in[N]$; $z_i=0$ if $\Num_{\sigma_i}(\Prom)/(\len(\Prom)+1)\le 1/(N+M+2)$ where $i\in[N]$; in other cases, $z_i$ can be taken arbitrarily. Please note that this value will not create ambiguity.

We will show why this is a solution to the 3-SAT problem $\Phi$. To demonstrate that, we only need to show that each $\phi_j=1$ holds for any $j\in[M]$ in $\Phi$. For convenience, in this part, we write $\Prom\oplus \sigma_{1+N}\oplus\sigma_{2+N}\oplus\sigma_{3+N}\oplus\dots\oplus\sigma_{j+N}=x_j$, where $j\in[M]$.

{\bf Firstly,} for any $i$, the probability of $\pi_{\F}(\sigma_{i+1+N} \mid x_{i})$ must be higher than the probability of $\pi_{\F}(\sigma_{j} \mid x_{i})$ for any $j\ne i+1+N$.

If not, there must be $\L_{\text{nll}}(\F, x_{i},\sigma_{N+1+i})> \ln2$, then $L(\F,S)> \ln2=\alpha$, which is contradictory to conclusion that $\PM(\F,S,\alpha,c,d)=1$.

{\bf Secondly,}  because the probability of $\pi_{\F}(\sigma_{i+1+N} \mid x_{i})$ is the largest, we have that:

(1)  if $\phi_{i}=z_{i'}\cup z_{j'}\cup z_{k'}$ for some $i',j',k'\in[N]$, consider the definition of $\Prom$ and transformer $\F$. There must be $\G_1(N_{4i'}(x_{i}))+\G_1(N_{4j'}(x_{i}))+\G_1(N_{4k'}(x_{i}))-3/N>0$; consider that $\G_1$ is increased with input and $\G_1(1/N)=1/N$. so at least one of $N_{4i'}(x_{i})\ge 1/N$, $N_{4j'}(x_{i})\ge 1/N$, and $N_{4k'}(x_{i})\ge 1/N$ stands, that is $\Num_{\sigma_i}(\Prom)/(\len(\Prom)+i)\ge 1/N$, which implies $z_{i'}=1$ or $z_{j'}=1$ or $z_{k'}=1$. Thus, we have that $\phi_i=1$;

(2)  if $\phi_i=z_{i'}\cup z_{j'}\cup \neg z_{k'}$ for some $i',j',k'\in[N]$, then we have that $\G_1(N_{4i'}(x_{i}))+\G_1(N_{4j'}(x_{i}))-\G_4(N_{4k'}(x_{i}))+0.1/N>0$. If $\neg z_{k'}=0$, then $N_{4k'}(x_{i})\ge 1/{(N+M+2)}$ based on the definition of solution; there are $-\G_4(N_{4k'}(x_{i}))+0.1/N=-2/N$ based on the definition of $\G_2$ and $\G_3$, which implies $\G_1(N_{4i'}(x_{i}))+\G_1(N_{4j'}(x_{i}))>2/N$. similar to before, at least one of $N_{4i'}\ge 1/N$ and $N_{4j'}\ge 1/N$ stands. So $\neg z_{k'}=0$ implies $z_{i'}=1$ or $z_{j'}=1$, hence $\phi_i=1$; when $\neg z_{k'}=1$, there is also $\phi_{i}=1$;

(3)  if $\phi_i=z_{i'}\cup\neg z_{j'}\cup \neg z_{k'}$ for some $i',j',k'\in[N]$, then we have that $\G_1(N_{4i'}(x_{i}))-\G_4(N_{4j'}(x_{i}))-\G_4(N_{4k'}(x_{i}))+0.1/N>0$. 
Similarly, if $\neg z_{k'}=0$ and $\neg z_{j'}=0$, then $N_{4j'}(x_{i}),N_{4k'}(x_{i})\ge 1/(N+M+2)$; hence $-\G_4(N_{4j'}(x_{i}))-\G_4(N_{4k'}(x_{i}))+0.1/N\le-1/N$, 
which implies that $N_{4i'}(x_{i})\ge 1/N$; so $\neg z_{k'}=0$ and $\neg z_{j'}=0$ imply $z_{i'}=1$; hence $\phi_i=1$;  when $\neg z_{k'}=1$ or $\neg z_{j'}=1$, this is also $\phi_i=1$;

(4)  if $\phi_i=\neg z_{i'}\cup\neg z_{j'}\cup \neg z_{k'}$ for some $i',j',k'\in[N]$, then we have that $6.2/N-\G_4(N_{4i'}(x_{i}))-\G_4(N_{4j'}(x_{i}))-\G_4(N_{4k'}(x_{i}))>0$. Consider that when $\neg z_{i'}=0$, there is 
$N_{4i'}(x_{i-1})\ge1/(N+M+2)$; then $2.1(1/N-\G_4(N_{4i'}(x_{i})))=2.1/N$, similar for $j',k'$. So, when $\neg z_{i'}=0,\neg z_{j'}=0,\neg z_{k'}=0$, there is $6.2/N-\G_4(N_{4i'}(x_{i}))-\G_4(N_{4j'}(x_{i}))-\G_4(N_{4k'}(x_{i}))\le 0$, which is contradictory to the above result. Therefore, at least one of $\neg z_{i'}=1$, $\neg z_{j'}=1$, or $\neg z_{k'}=1$ stands, so $\phi_i=1$.

Thus, we prove Part One.

{\bf Part 2. If  $\PM(\F,S,\alpha,c,d)=0$, then $\Phi$ has no solution, or it has a solution $z_i=1$ for all $i$ or $z_i=0$ for all $i$.}

We just need to show that if the 3-SAT problem $\Phi$ has a solution that is not $z_i=1$  for all $i$ or $z_i=0$  for all $i$, then such $\PM=1$.

Let $\sigma^{3M}_{i}$ denote the sequence that contains only $3M$ number of $\sigma_i$. Then let $\Prom=\sigma^{3M}_{i_1}\oplus\sigma^{3M}_{i_2}\oplus\dots\sigma^{3M}_{i_{j-1}}\oplus\sigma^{3M}_{i_j}$ where $i_j$ satisfies $z_{i_j}=1$ in the solution of the 3-SAT problem $\Phi$. It is easy to see that $\len(\Prom)\le c\len(y)^d$ by the value of $c,d$. For convenience, in this part, we still write $ \Prom\oplus\sigma_{1+N}\oplus\sigma_{2+N}\oplus\sigma_{3+N}\oplus\dots\oplus\sigma_{j+N}=x_j$, where $j\in[M]$; we just need to show that $-\ln\pi_\F(\sigma_{1+N+i}| x_{i})\le \ln2/(M+1)$ for any $i\in[M]$.

Because the solution is not $z_i=1$  for all $i$ or $z_i=0$ for all $i$, we know that $\frac{N_{\sigma_i}(x_k)}{\len(x_k)}\ge\frac{3M}{3M(N-1)+1+M}\ge 1/(N-0.5)$ for any $i,k\in[M]$ such that $z_{i}=1$ in the solution of $\Phi$, and $\frac{N_{\sigma_i}(x_k)}{\len(x_k)}=0$ when $z_i=0$ in the solution of $\Phi$. 

Then we consider the $\F(x_{q})$ where $q\in[M]$, such that:

(1)  if $\phi_q=z_{i}\cup z_{j}\cup z_{k}$ for some $i,j,k\in[N]$, then at least one of $N_{4i}(x_{q})\ge 1/(N-0.5)$, $N_{4j}(x_{q})\ge 1/(N-0.5)$, and $N_{4k}(x_{q-1})\ge 1/(N-0.5)$ stands, so 
\begin{equation*}
    \begin{array}{cl}
        &M(\G_1(N_{4i}(x_{q-1}))+\G_1(N_{4j}(x_{q-1}))+\G_1(N_{4k}(x_{q-1}))-3/N)\\
        &-100M^2\sum_{t\in[N+1,M+N],t\ne N+q}N_{4t}(x_{q-1})\\
        =&M(\G_1(N_{4i}(x_{q-1}))+\G_1(N_{4j}(x_{q-1}))+\G_1(N_{4k}(x_{q-1}))-3/N)\\
        \ge&M(3-3/N)>M.
    \end{array}
\end{equation*}
We use $\G_1(1/(N-0.5))>3$ and $\G_1(x)\ge0$ when $x\ge0$ here. 

On the other hand, for $t\ne q$, there is $-100M^2\sum_{i\ne N+t,i\in[N+1,M+N]}N_{4i}(x_{q})\le -100M^2$ because $N_{4(N+q)}(x_{q})\ge1$, so the $N+t$ weight has a value of no more than $0$. Consider that the value of the first N+1-weight is $0$, so $-\ln\pi_{\F}(\sigma_{q+1+N} \mid x_{q-1})\le -\ln\frac{e^M}{e^M+M+N}\le \ln(1+4M/e^M)\le 4M/e^M\le \ln2/(M+1)$; use $M\ge 6$ and $N\le 3M$ here, which is what we want;

(2) If $\phi_q=z_i\cup z_j\cup \neg z_k$ for some $i,j,k\in[N]$, then at least one of $N_{4i}(x_{q})\ge 1/(N-0.5)$, $N_{4j}(x_{q})\ge 1/(N-0.5)$ and $N_{4k}(x_{q})=0$ stands, so 
\begin{equation*}
    \begin{array}{cl}
         & M(\G_1(N_{4i}(x_{q}))+\G_1(N_{4j}(x_{q}))-\G_4(N_{4k}(x_{q}))+0.1/N)\\
         &-100M^2\sum_{t\in[N+1,M+N],t\ne N+q}N_{4t}(x_{q}) \\
         =&M(\G_1(N_{4i}(x_{q}))+\G_1(N_{4j}(x_{q}))-\G_4(N_{4k}(x_{q}))+0.1/N)\\
         \ge&M\min\{2.1-2/N,3-2/N\}>M. 
    \end{array}
\end{equation*}
where we use $\G_1(1/(N-0.5))>3$ and $G_4(0)=-2.1(1-1/N)$. Similar to before, other weights have values of no more than $0$, so $-\ln\pi_{\F}(\sigma_{q+1+N} \mid x_{q})\le -\ln\frac{e^M}{e^M+M+N}\le \ln2/(M+1)$;

(3)  if $\phi_q=z_i\cup\neg  z_j\cup \neg z_k$ for some $i,j,k\in[N]$, then at least one of $N_{4i}(x_{q})\ge 1/(N-0.5)$, $N_{4j}(x_{q})=0$ and $N_{4k}(x_{q})=0$ stands, so 
\begin{equation*}
    \begin{array}{cl}
         & M(\G_1(N_{4i}(x_{q}))-\G_4(N_{4j}(x_{q}))-\G_4(N_{4k}(x_{q}))+0.1/N)\\
         &-100M^2\sum_{t\in[N+1,M+N],t\ne N+q}N_{4t}(x_{q})\\
         =&M(\G_1(N_{4i}(x))-\G_4(N_{4j}(x))-\G_4(N_{4k}(x))+0.1/N)\\
         \ge&M\min\{3-4.1/N,2.1-4.1/N\}>M.
    \end{array}
\end{equation*}

 Similar to before, other weights have a value of no more than $0$. Therefore, $-\ln\pi_{\F}(\sigma_{q+1+N} \mid x_{q})\le -\ln\frac{e^M}{e^M+M+N}\le \ln2/(M+1)$;

(4)  If $\phi_1=\neg z_i\cup\neg  z_j\cup \neg z_k$ for   some $i,j,k\in[N]$, then we have at least one of $N_{4i}(x_{q})=0$, $N_{4j}(x_{q})=0$, and $N_{4k}(x_{q})=0$ standing, so 

\begin{equation*}
    \begin{array}{cl}
      &  M(6.2/N-\G_4(N_{4i}(x_{q}))-\G_4(N_{4j}(x_{q}))-\G_4(N_{4k}(x_{q})))\\
      &-100M^2\sum_{t\in[N+1,M+N],t\ne N+q}N_{4t}(x_{q})\\
      =&M(6.2/N-\G_4(N_{4i}(x_{q}))-\G_4(N_{4j}(x_{q}))-\G_4(N_{4k}(x_{q})))\\
      \ge&M(2.1-0.1/N)>M.
    \end{array}
\end{equation*}

Similar to before, other weights have a value of no more than $0$, so $-\ln\pi_{\F}(\sigma_{q+1+N} \mid x_{q})\le -\ln\frac{e^M}{e^M+M+N}\le \ln2/(M+1)$.

Therefore, we have that $\L(\F,S)\le (M+1)\ln2/(M+1)=\ln 2$, so we prove Part Two.

{\bf Part 3.}
Now we prove the theorem. If we can solve the $\PM$ problem, then we can solve the 3-SAT existence problem as follows:

(1) Construct such transformer $\F$, dataset $S$ and symbols $\Sigma$ and $c,\alpha,e$ based on $\Phi$ as before;

(2) If $\PM(\F,S,\alpha,c,d)=1$, then the 3-SAT Problem $\Phi$ has a solution;

(3) If $\PM(\F,S,\alpha,c,d)=0$, then the 3-SAT $\Phi$ problem has no solution, or the solution satisfies that all the values of $z_i$ are the same, which is easy to check in polynomial time.
\end{proof}

%\subsection{Computational complexity of prompts without length bound}
%\label{app-unbounded}

%In the proof for Theorem \ref{th-npc}, if we ignore the values of $c,d$, it still holds. We thus have Corollary \ref{cor-11}. 
%{\color{red} We will give a sketch of the proof below.}

\subsection{Computational complexity of prompting for greedy sampling}
\label{app-greedy}
%In some other work, they consider the deterministic output of a transformer, such as \citep{pp1,pp2}, which is defined as:

Greedy sampling is used in many works on prompting \citep{pp1,pp2} as well as in theoretical studies of transformers \citep{perez2021attention,merrill2023expressive,yu2025analyzing}. 
In this case, the output for an autoregressive transformer is computed as follows. Let the initial value $y=()$.

(1) Let $s$ be the token  that maximizes $\pi_{\F}(s\mid x\oplus y)$.

(2) If $s\ne \sigma_0$, insert $s$ as the last element of $y$, and return to step (1).

(3) If $s=\sigma_0$, stop and return the output $y=\overline{\F}(x)$.

Under this type of output, we have the following:

{\bf Prompt.}
We say that a transformer $\F$ {\em solves the task $S=\{(x_i,y_i)\}_{i=1}^N$ with prompt $\Prom$}, if $y_i=\overline{\F}(\Prom\oplus x_i)$ for $i=1,\ldots,N$, or equivalently, $\F$ is a memorization network for $\Prom\oplus S$.

Determining the existence of prompts is also NPC.
\begin{proposition}
\label{cor-12}
For given $c,d>0$, $\F$, and task $S=\{(x_i,y_i)\}_{i=1}^N$, deciding whether there exists a prompt $\Prom\in\Sigma^{*}$ such that $y_i = \overline{\F}(\Prom\oplus x_i)$ and $\len(\Prom)\le c\len(S_a)^d$ is NPC. The result is true for $N=1$.
\end{proposition}

The proof is similar to that for Theorem \ref{th-npc}; simply ignore the part regarding loss calculation (the first part in step 2 of part 1 and the calculation of $-\ln(\pi_\F(\sigma_{1+N+i}|x_i))$ in part 2). 

%The result is also true for the unbounded case.
%\begin{corollary}
%\label{cor-13}
%For given $\F$ and task $S=\{(x_i,y_i)\}_{i=1}^N$, deciding whether there exists a prompt $\Prom\in\Sigma^{*}$ such that $y_i = \overline{\F}(\Prom\oplus x_i)$ is NPC. 
%\end{corollary}

We can similarly define the optimization problem to find a prompt.
\begin{definition}
The problem $\DPO_d(\F,S,c,d)$ is defined as follows: for the given transformer $\F$, task $S$, and $c,d\in\R_{>0}$, the problem $\DPO_d$ is to solve the following optimization problem to obtain a prompt $\Prom$:

$$\argmax_{\Prom\in\Sigma^{*},\len(\Prom)\le c\,\len(S_a)^d} \sum_{(x_i,y_i)\in S}\ID(\overline{\F}(\Prom\oplus x_i)=y_i).$$
\end{definition}

Then, by Corollary \ref{cor-12}, we have that:
\begin{corollary}    
\label{cor-133}
$\DPO_d(\F,S,c,d)$ is an NP-hard problem.
\end{corollary}

\section{Proofs of Section \ref{sec-50}}

\subsection{Proof of Proposition \ref{prop-g11}}

%We apply the theorem presented in \citep{yu2025analyzing}, which can be written as:
We use the following result.
\begin{Theorem}[Theorem 4.4 of \citep{yu2025analyzing}]
\label{th-thapp12}
    Let $S=\{(x_i,y_i)\}_{i=1}^N\subset\Sigma^*\times\Sigma^*$ satisfy that $x_i[1:\len(x_j)]\ne x_j$ for any $i,j\in[N]$ such that $\len(x_i)\ge\len(x_j)$, then there is a transformer $\F$ that can memory $S$.
\end{Theorem}

In the setting of \citep{yu2025analyzing}, $\F$ can memory $S$ means that: $\pi_\F(y_i[k]|x_i\oplus y_i[1:k-1])>\pi_\F(\sigma_j|x_i\oplus y_i[1:k-1])$ for any $\sigma_j\ne y_i[k]$ and $i\in[N]$, $k\in[\len(y_i)]$.
Therefore, we can prove Proposition \ref{prop-g11} as follows.
\begin{proof}
    For length $l\le L$, there exist at most $(T+1)^{Nl}$ answer sets $S_a$ such that $\len(S_a)=l$. Hence, there is at least $ T^{\lceil \log_T(T+1)Nl \rceil}\ge (T+1)^{Nl}$ number of prompts with a length $\lceil(\log_T(T+1)+1)Nl\rceil$. So for any answer set $S_a$ with length $l\le L$, we can find a prompt $\Prom_{S_a}$ with length $\lceil\log_T(T+1)Nl\rceil$ corresponding to $S_a$. Hence, for each $l\le L$, we can find a different prompt $\Prom_l$ with length $\lceil \log_TL\rceil$ corresponding to such $l$. So, combined, for any answer set $S_a$ with length $l\le L$, we can find a prompt $\Prom_{\len(S_a)}\oplus\Prom_{S_a}$  corresponding to $S_a$.

Now, we will show that for any $S_q$ satisfying the given loss $\alpha$ and length $L$, there is a transformer $\F$ such that $\L(\F,\Prom_{\len(S_a)}\oplus \Prom_{S_a}\oplus S_q,S_a)\le \alpha$ for any $S_a$ such that $\len(S_a)\le L$. Consider that $\len(\Prom_{\len(S_a)})+\len(\Prom_{S_a})\le(\log_T(T+1)N+\log_T L+2)\len(S_a)\le c\len(S_a)^d$, so this is what we want.

% %the following dataset $S_1=\{\Prom_i\}_{i=1}^L\in \Sigma^*$ such that: $\len(\Prom_i)=[\log_TL]+1$ and $\Prom_i\ne \Prom_j$. 

 %   Now let
 To prove it, we define that $(\Prom_{\len(S_a)}\oplus\Prom_{S_a}\oplus S_q, S_a)=\{(\Prom_{\len(S_a)}\oplus\Prom_{S_a}\oplus x_i,y_i)\}_{i=1}^N$, and we consider the dataset $S=\cup_{S_a:\len(S_a)\le L} (\Prom_{\len(S_a)}\oplus\Prom_{S_a}\oplus S_q, S_a)$. It is easy to see that we only need to show that $S$ can be memorized by a transformer with loss $\alpha$.

We have that for any $(x_1,y_1),(x_2,y_2)\in S$ such that $\len(x_1)\ge\len(x_2)$, there is $x_1[1:\len(x_2)]\ne x_2$. Because if $(x_1,y_1),(x_2,y_2)$ is from the same set $(\Prom_{\len(S_a)}\oplus\Prom_{S_a}\oplus S_q, S_a)$, then by the assumption of $S_q$, we know that it holds. If not, let $(x_i,y_i)$ be from the same set $(\Prom_{\len(S_{a_i})}\oplus\Prom_{S_{a_i}}\oplus S_{q}, S_{a_i})$. 
If $\len(S_{a_1})\ne\len(S_{a_2})$, since the lengths of $\Prom_{\len(S_{a_i})},i\in[2]$ are the same, we know it holds; 
if not, since $\Prom_{\len(S_{a_1})}=\Prom_{\len(S_{a_2})}$ and the lengths of $\Prom_{S_{a_1}},i\in[2]$ are also the same, we know it still holds. We prove it. 

Then by Theorem \ref{th-thapp12} and Theorem 4.4 of \citep{yu2025analyzing}, we know that there is a transformer $\F$ such that $\F$ can memorize $S$, which means that for any $(x,y)\in S$ and $k$, there is $\pi_\F(y[k]|x\oplus y[1:k-1])>\pi_\F(\sigma_j|x\oplus y[1:k-1])$ for any $\sigma_j\ne y[k]$.

Let $\F_A(x)=A\F(x)$ where $A\in\R$. It is easy to see that there is a $A\in\R_+$ satisfied that when $c>A$, there is $-\ln\pi_{\F_c}(y[i]|x\oplus y[1:i-1])<\alpha/(L+1)$ for any $(x,y)\in S$ and $i$, which implies that $\L_{\text{nll}}(\F_c,x,y)\le \alpha/(L+1)(\len(S_a)+1)\le\alpha$ for any $(x,y)$. So $\F_c(c>A)$ is what we want.
%
%Then we show that $\F_A$ is what we want. For any answer set $S_a$, we find the prompt $\Prom_{\len(S_a)}\oplus\Prom_{S_a}$, easy to see that $\len(\Prom_{\len(S_a)}\oplus\Prom_{S_a})=+\le c\len(S_a)^d$ by the definition of $\Prom_{\len(S_a)},\Prom_{S_a}$, and by the definition of $\F_A$, there is $\L_{\text{nll}}(\F_A,\Prom_{\len(S_a)}\oplus\Prom_{S_a}\oplus x,y)<\alpha$, which is what we want. 
\end{proof}

\subsection{Proof of Theorem \ref{negg}}

Firstly, we have the following lemma.
\begin{lemma}
\label{l6}
Let $x,z\in\Sigma^{*}$ and $\len(x)=\len(z)$, for any given $t,w$, let $x_{t,w}=t\oplus x\oplus w$ and $z_{t,w}=t\oplus z\oplus w$. Then, for any transformer $\F$, we have that $\lim_{\len(t)\to\infty,\len(w)\to\infty}||\F(x_{t,w})-\F(z_{t,w})||_2=0$. 
\end{lemma}

\begin{proof}

    We assume the $l$-th hidden layer of $\F(x)$ is $\F_{l}(x)$, and the $\F_{l,k}(x)$ is the $k$-th row of $\F_l(x)$. 

Now we assume that the following result stands for the $L-1$-th hidden layer: there is $\lim_{\len(t),\len(w)\to\infty}||\F_{L-1,\len(t)+\len(x)+\len(w)}(x_{t,w})-\F_{L-1,\len(t)+\len(z)+\len(w)}(z_{t,w})||_2=0$. Then we can show that the result is also satisfied for the $L$-th hidden layers. 

It is easy to see that it stands for the embedding layer, which can be regarded as the $0$-th hidden layer. Therefore, this result can be related to the last hidden layer, and considering that $\F(x_{t,w})$ only depends on the last row of the last hidden layer, this result can easily lead to the lemma.

By the definition of the transformer, we know that $\F_{L}(x_{t,w})=\FNN_{L}(\F_{L-1}(x_{t,w})+\ATT_{L}(\F_{L-1}(x_{t,w})))+\F_{L-1}(x_{t,w})+\ATT_{L}(\F_{L-1}(x_{t,w}))$, similar to $\F_{L}(z_{t,w})$.

We prove this result in two parts.

{\bf The gap in the attention layer.}

We will prove that $\lim_{\len(t)\to\infty,\len(w)\to\infty}||ATT_{L}(\F_{L-1}(x_{t,w}))[\len(x_{t,w})]-ATT_{L}(\F_{L-1}(z_{t,w}))[\len(z_{t,w})]||_2=0$.

Firstly, we consider the attention layer, easy to see that the last row of $\ATT_{L}(\F_{L-1}(x_{t,w}))$ is $\frac{\sum_{k=1}^{\len(x_{t,w})}e^{u_k}\F_{L-1,k}(x_{t,w})}{\sum_{k=1}^{\len(x_{t,w})}e^{u_k}}V_{L}$, and the last row of $\ATT_{L}(\F_{L-1}(z_{t,w}))$ is $\frac{\sum_{k=1}^{\len(z_{t,w})}e^{u'_k}\F_{L-1,k}(z_{t,w})}{\sum_{k=1}^{\len(z_{t,w})}e^{u'_k}}V_{L}$. In here $u_k=\F_{L-1,\len(x_{t,w})}(x_{t,w})Q_LA_{\len(x_{t,w})-k}K_L\F^T_{L-1,k}(x_{t,w})$ and $u'_k=\F_{L-1,\len(z_{t,w})}(z_{t,w})Q_LA_{\len(z_{t,w})-k}K_L\F^T_{L-1,k}(z_{t,w})$.

According to the assumption and lemma, we have the following facts:

(1)  Based on the assumption,  there is $\lim_{\len(t)\to\infty,\len(w)\to\infty}||\F_{L-1,\len(x_{t,w})}(x_{t,w})-\F_{L-1,\len(z_{t,w})}(z_{t,w})||_2\to 0$, so for any $\epsilon$, there is a $T(\epsilon)$ such that when $\len(t),\len(w)\ge T(\epsilon)$, there is $||\F_{L-1,\len(x_{t,w})}(x_{t,w})-\F_{L-1,\len(z_{t,w})}(z_{t,w})||_2\le\epsilon$.

(2)  By the Lemma \ref{sx2},  the $||\F_{L-1}(x_{t,w})||_{2,\infty}$ and $||\F_{L-1}(z_{t,w})||_{2,\infty}$ have the upper bound $A$; $||Q_LA_jK_L||_2$ also has an upper bound $B$.

Regarding the molecular output of the attention layer, we have that 
\begin{equation*}
\begin{array}{cl}
&\sum_{k=1}^{\len(x_{t,w})}e^{u_k}\F_{L-1,k}(x_{t,w})\\
&=\sum_{k=1}^{\len(t)}e^{u_k}\F_{L-1,k}(x_{t,w})+\sum_{k=1+\len(t)}^{\len(t)+\len(x)}e^{u_k}\F_{L-1,k}(x_{t,w})\\
&+\sum_{k=1+\len(x)+\len(t)}^{\len(w)+\len(x)+\len(t)}e^{u_k}\F_{L-1,k}(x_{t,w})\\
&\hbox{and}\\
&\sum_{k=1}^{\len(z_{t,w})}e^{u'_k}\F_{L-1,k}(z_{t,w})\\
&=\sum_{k=1}^{\len(t)}e^{u'_k}\F_{L-1,k}(z_{t,w})+\sum_{k=1+\len(t)}^{\len(t)+\len(z)}e^{u'_k}\F_{L-1,k}(z_{t,w})\\
&+\sum_{k=1+\len(z)+\len(t)}^{\len(w)+\len(z)+\len(t)}e^{u'_k}\F_{L-1,k}(z_{t,w})\\
\end{array}
\end{equation*}

By fact (2), it holds $\sum_{k=1+\len(t)}^{\len(t)+\len(x)}e^{u_k}\le\len(x)e^{A^2B}$ for any $t,w$, but $$\lim_{\len(t)\to\infty,\len(w)\to\infty}\sum_{k=1}^{\len(t)}e^{u_k}+\sum_{k=1+\len(x)+\len(t)}^{\len(w)+\len(x)+\len(t)}e^{u_k}=\infty,$$
which is named as (*). This is similar to the $\sum_{k=1}^{\len(z_{t,w})}e^{u'_k}$.

For the denominator of the attention layer output, we have a similar result; therefore, we will prove the following four results:

(r1)  $\lim_{\len(t)\to\infty,\len(w)\to\infty}\frac{\sum_{k=1}^{\len(t)}e^{u'_k}}{\sum_{k=1}^{\len(t)}e^{u_k}}=1$;

(r2)  $\lim_{\len(t)\to\infty,\len(w)\to\infty}\frac{\sum_{k=1}^{\len(t)}e^{u_k}\F_{L-1,k}(x_{t,w})-\sum_{k=1}^{\len(t)}e^{u'_k}\F_{L-1,k}(z_{t,w})}{\sum_{k=1}^{\len(t)}e^{u_k}}=0$;

(r3)  $\lim_{\len(t)\to\infty,\len(w)\to\infty}\frac{\sum_{k=1+\len(z)+\len(t)}^{\len(w)+\len(z)+\len(t)}e^{u'_k}}{\sum_{k=1+\len(x)+\len(t)}^{\len(w)+\len(x)+\len(t)}e^{u_k}}=1$;

(r4)  $\lim_{\len(t)\to\infty,\len(w)\to\infty}\frac{\sum_{k=1+\len(x)+\len(t)}^{\len(w)+\len(x)+\len(t)}e^{u_k}\F_{L-1,k}(x_{t,w})-\sum_{k=1+\len(z)+\len(t)}^{\len(w)+\len(z)+\len(t)}e^{u'_k}\F_{L-1,k}(z_{t,w})}{\sum_{k=1+\len(x)+\len(t)}^{\len(w)+\len(x)+\len(t)}e^{u_k}}=0$. 

Such four results and the fact (2) can lead to what we desire, as shown below:

\begin{equation*}
    \begin{array}{cl}
         & \lim_{\len(t)\to\infty,\len(w)\to\infty}||ATT_{L}(\F_{L-1}(x_{t,w}))[\len(x_{t,w})]-ATT_{L}(\F_{L-1}(z_{t,w}))[\len(z_{t,w})]||_2 \\
         =& \lim_{\len(t)\to\infty,\len(w)\to\infty}||\frac{\sum_{k=1}^{\len(x_{t,w})}e^{u_k}\F_{L-1,k}(x_{t,w})}{\sum_{k=1}^{\len(x_{t,w})}e^{u_k}}V_{L}-\frac{\sum_{k=1}^{\len(z_{t,w})}e^{u'_k}\F_{L-1,k}(z_{t,w})}{\sum_{k=1}^{\len(z_{t,w})}e^{u'_k}}V_{L}||_2\\
         =& \lim_{\len(t)\to\infty,\len(w)\to\infty}||\frac{\sum_{k=1}^{\len(z_{t,w})}e^{u'_k}\F_{L-1,k}(z_{t,w})}{\sum_{k=1}^{\len(x_{t,w})}e^{u_k}}V_{L}-\frac{\sum_{k=1}^{\len(z_{t,w})}e^{u'_k}\F_{L-1,k}(z_{t,w})}{\sum_{k=1}^{\len(z_{t,w})}e^{u'_k}}V_{L}||_2\, \textbf{(by(*), (r2), and (r4))}\\
         =&\lim_{\len(t)\to\infty,\len(w)\to\infty}||\frac{\sum_{k=1}^{\len(z_{t,w})}e^{u'_k}\F_{L-1,k}(z_{t,w})}{\sum_{k=1}^{\len(x_{t,w})}e^{u_k}}V_{L}(1-\frac{\sum_{k=1}^{\len(x_{t,w})}e^{u_k}}{\sum_{k=1}^{\len(z_{t,w})}e^{u'_k}})||_2\\
         =&0\, \textbf{(by (r1), (r3), and fact (2))}
    \end{array}
\end{equation*}

Thus, we only need to prove such four results.

Firstly, we prove (r1) for any $\epsilon$. By fact (1), there is a $T(\epsilon)$ such that when $\len(t),\len(w)\ge T(\epsilon)$, there is $||\F_{L-1,\len(x_{t,w})}(x_{t,w})-\F_{L-1,\len(z_{t,w})}(z_{t,w})||_2\le\epsilon$. Then, we have that: $u_k-u'_k=(\F_{L-1,\len(x_{t,w})}(x_{t,w})-\F_{L-1,\len(z_{t,w})}(z_{t,w}))Q_LA_{\len(x_{t,w}-k)}K_L\F_{L-1,k}(x_{t,w})$ when $k\le\len(t)$ by lemma \ref{qn}. So when $\len(t),\len(w)\ge T(\epsilon)$, there is:

\begin{equation*}
    \begin{array}{cl}
         & \frac{\sum_{k=1}^{\len(t)}e^{u_k}}{\sum_{k=1}^{\len(t)}e^{u'_k}} 
=\frac{\sum_{k=1}^{T(\epsilon)}e^{u_k}+\sum_{k=T(\epsilon)+1}^{\len(t)}e^{u_k}}{\sum_{k=1}^{T(\epsilon)}e^{u'_k}+\sum_{k=T(\epsilon)+1}^{\len(t)}e^{u'_k}}
         \le \frac{T(\epsilon)e^{A^2B}+\sum_{k=T(\epsilon)+1}^{\len(t)}e^{u_k}}{T(\epsilon)e^{-A^2B}+\sum_{k=T(\epsilon)+1}^{\len(t)}e^{u_k}e^{-\epsilon AB}}.
    \end{array}
\end{equation*}
Hence, for any $\eta>0$, let $\epsilon\to 0$ and $\len(t)>>T(\epsilon)$ satisfy  $e^{-\epsilon AB}\ge 1/(1+\eta/2)$ and $\eta e^{-\epsilon AB}\sum_{k=T(\epsilon)+1}^{\len(t)}e^{u_k}/2>T(\epsilon)e^{A^2B}$. Then, there are $\frac{\sum_{k=1}^{\len(t)}e^{u_k}}{\sum_{k=1}^{\len(t)}e^{u'_k}}\le\frac{T(\epsilon)e^{A^2B}+\sum_{k=T(\epsilon)+1}^{\len(t)}e^{u_k}}{T(\epsilon)e^{-A^2B}+\sum_{k=T(\epsilon)+1}^{\len(t)}e^{u_k}e^{-\epsilon AB}}\le1+\eta$. Similarly, there is also $\frac{\sum_{k=1}^{\len(t)}e^{u_k}}{\sum_{k=1}^{\len(t)}e^{u'_k}}\ge\frac{T(\epsilon)e^{-A^2B}+\sum_{k=T(\epsilon)+1}^{\len(t)}e^{u_k}}{T(\epsilon)e^{A^2B}+\sum_{k=T(\epsilon)+1}^{\len(t)}e^{u_k}e^{\epsilon AB}}\ge1-\eta$ when $\epsilon\to 0$ and $\len(t)>>T(\epsilon)$. Then by the arbitrariness of $\eta$, so $\lim_{\len(t)\to\infty,\len(w)\to\infty}\frac{\sum_{k=1}^{\len(t)}e^{u_k}}{\sum_{k=1}^{\len(t)}e^{u'_k}}=1$.

Now we prove (r2):
\begin{equation}
\label{gafg}
    \begin{array}{cl}
         &  \frac{\sum_{k=1}^{\len(t)}e^{u_k}\F_{L-1,k}(x_{t,w})-\sum_{k=1}^{\len(t)}e^{u'_k}\F_{L-1,k}(z_{t,w})}{\sum_{k=1}^{\len(t)}e^{u_k}}\\
         \le &\frac{\sum_{k=T(\epsilon)+1}^{\len(t)}||e^{u_k}\F_{L-1,k}(x_{t,w})-e^{u'_k}\F_{L-1,k}(z_{t,w})||_2+2T(\epsilon)e^{A^2B}A}{\sum_{k=1}^{\len(t)}e^{u_k}}
    \end{array}
\end{equation}

We have $||e^{u_k}\F_{L-1,k}(x_{t,w})-e^{u'_k}\F_{L-1,k}(z_{t,w})||_2\le||e^{u_k}\F_{L-1,k}(x_{t,w})-e^{u'_k}\F_{L-1,k}(x_{t,w})||_2+||e^{u'_k}\F_{L-1,k}(x_{t,w})-e^{u'_k}\F_{L-1,k}(z_{t,w})||_2 \le \epsilon e^{A^2B}+(e^{\epsilon AB}-1)Ae^{A^2B}$ when $k\ge T(\epsilon)$. %, we use $|u_k-u'_k|\le|(\F_{L-1,\len(x_{t,w})}(x_{t,w})-\F_{L-1,\len(z_{t,w})}(z_{t,w}))Q_LA_{\len(x_{t,w})-k}K_L\F^T_{L-1,k}(x_{t,w})|+|\F_{L-1,\len(z_{t,w})}(z_{t,w})Q_LA_{\len(x_{t,w})-k}K_L(\F^T_{L-1,k}(z_{t,w})-\F^T_{L-1,k}(x_{t,w}))|$ in here.
Further considering that $||e^{u_k}\F_{L-1,k}(x_{t,w})-e^{u'_k}\F_{L-1,k}(z_{t,w})||_2/e^{u_k}\le(\epsilon e^{A^2B}+(e^{\epsilon AB}-1)Ae^{A^2B})/e^{u_k}=0$ when $\epsilon\to 0$ for any $k$, so taking $\epsilon\to 0$ and $\len(t)>>T(\epsilon)$, we can prove that $\lim_{\len(t)\to\infty,\len(w)\to\infty}\frac{\sum_{k=1}^{\len(t)}e^{u_k}\F_{L-1,k}(x_{t,w})-\sum_{k=1}^{\len(t)}e^{u'_k}\F_{L-1,k}(z_{t,w})}{\sum_{k=1}^{\len(t)}e^{u_k}}=0$ by \eqref{gafg}.

Now we prove (r3). for the last $\len(w)$ samples, we have that $|u_k-u'_k|\le|(\F_{L-1,\len(x_{t,w})}(x_{t,w})-\F_{L-1,\len(z_{t,w})}(z_{t,w}))Q_LA_{\len(x_{t,w})-k}K_L\F^T_{L-1,k}(x_{t,w})|+|\F_{L-1,\len(z_{t,w})}(z_{t,w})Q_LA_{\len(x_{t,w})-k}K_L(\F^T_{L-1,k}(x_{t,w})-\F^T_{L-1,k}(z_{t,w}))|$. Thus, when $\len(w),\len(t)\ge T(\epsilon)$, we have that:

\begin{equation*}
    \begin{array}{cl}
         & \frac{\sum_{k=1+\len(t)+\len(x)}^{\len(w)+\len(t)+\len(x)}e^{u_k}}{\sum_{k=1+\len(t)+\len(z)}^{\len(w)+\len(t)+\len(z)}e^{u'_k}} \\
         =&\frac{\sum_{k=1+\len(t)+\len(x)}^{T(\epsilon)+\len(t)+\len(x)}e^{u_k}+\sum_{k=T(\epsilon)+1+\len(t)+\len(x)}^{\len(w)+\len(t)+\len(x)}e^{u_k}}{\sum_{k=1+\len(t)+\len(z)}^{T(\epsilon)+\len(t)+\len(z)}e^{u'_k}+\sum_{k=T(\epsilon)+1+\len(t)+\len(z)}^{\len(w)+\len(t)+\len(z)}e^{u'_k}}\\
         \le & \frac{T(\epsilon)e^{A^2B}+\sum_{k=T(\epsilon)+1+\len(t)+\len(x)}^{\len(w)+\len(t)+\len(x)}e^{u_k}}{T(\epsilon)e^{-A^2B}+\sum_{k=T(\epsilon)+1+\len(t)+\len(z)}^{\len(w)+\len(t)+\len(z)}e^{u_k}e^{-2\epsilon AB}},
    \end{array}
\end{equation*}

For any $\eta>0$, let $\epsilon\to 0$ and $\len(w)>>T(\epsilon)$ such that similar to (r1), there are $\frac{\sum_{k=1+\len(t)+\len(x)}^{\len(w)+\len(t)+\len(x)}e^{u_k}}{\sum_{k=1+\len(t)+\len(z)}^{\len(w)+\len(t)+\len(z)}e^{u'_k}}\le1+\eta$. Similarly, there is also $\frac{\sum_{k=1+\len(t)+\len(x)}^{\len(w)+\len(t)+\len(x)}e^{u_k}}{\sum_{k=1+\len(t)+\len(z)}^{\len(w)+\len(t)+\len(z)}e^{u'_k}}\ge1-\eta$, so $\lim_{\len(t)\to\infty,\len(w)\to\infty}\frac{\sum_{k=1+\len(t)+\len(x)}^{\len(w)+\len(t)+\len(x)}e^{u_k}}{\sum_{k=1+\len(t)+\len(z)}^{\len(w)+\len(t)+\len(z)}e^{u'_k}}=1$. 

For (r4), just similar to (r2). Therefore, we prove the four results.

{\bf The FNN layer.}

By $\lim_{\epsilon\to 0}||\FNN(x)-\FNN(x+\epsilon)||_2=1$, we can use the results in the attention layer and the assumptions to show that 
%{\small
%\noindent
\begin{align*}
&\lim_{\len(t)\to\infty,\len(w)\to\infty}||{\FNN_{L}(\F_{L-1}(x_{t,w})+\ATT_{L}(\F_{L-1}(x_{t,w}))[\len(x_{t,w})])}\\
&-{\FNN_{L}(\F_{L-1}(z_{t,w})+\ATT_{L}(\F_{L-1}(z_{t,w}))[\len(z_{t,w})])}||_2=0    
\end{align*}
This is valid, because $\lim_{\len(t)\to\infty,\len(w)\to\infty}||\F_{L-1}(z_{t,w}))-\F_{L-1}(x_{t,w}))||_2=0$ and $\lim_{\len(t)\to\infty,\len(w)\to\infty}||\ATT_{L}(\F_{L-1}(x_{t,w}))[\len(x_{t,w})]-\ATT_{L}(\F_{L-1}(z_{t,w}))[\len(z_{t,w})]||_2=0$, as shown before.

Combining the results of such two parts is what we want.
\end{proof}

We now prove Theorem \ref{negg}:

\begin{proof}
    We consider the following query set: $S_q=\{x,z\}\subset\Sigma^*$, where $\len(x)=\len(z)$ and $x\ne z$. It is easy to see that $x,z$ satisfies the condition in the theorem. Take $\alpha=-\ln0.55$.
    
       Assume there is an $\F$ such that for any answer set $S_a$, there is a $\Prom$ such that  $\L(\F,\Prom\oplus S_q,S_a)\le \alpha$. Then we consider the answer set $S^y_a=\{(y\oplus\sigma_1,y\oplus\sigma_2)\}$ where $y\in\Sigma^*$.

         Firstly, for any $N$, we can find a $y_N$ such that: $\L(\F,\Prom\oplus S_q,S^{y_N}_a)> \alpha$ for any $\Prom$ such that $\len(\Prom)\le N$, and $\len(y_N)\ge N$. This is because for any prompt $\Prom$, there is a limited number of answer sets $S_a$ such that $\L(\F,\Prom\oplus S_q,S_a)\le \alpha$, and there is also a limited number of prompts with a length of at most $N$.

    Now we consider that when $N\to \infty$, by lemma \ref{l6}, easy to see that when $\len(\Prom)\ge N$ and $N\to\infty$, there is $||\F(\Prom\oplus x\oplus y_{N})-\F(\Prom\oplus z\oplus y_N)||_2\to 0$. Assume that $N_0$ satisfies $||\F(\Prom\oplus x\oplus y_{N_0})-\F(\Prom\oplus z\oplus y_{N_0})||_2< 0.05$ for any $\Prom$ such that $\len(\Prom)\ge N_0$. We also have that  
    $2\L(\F,\Prom\oplus S_q,S_a^y)\ge-\ln(\pi_\F(\sigma_1|\Prom\oplus x\oplus y_N))-\ln(\pi_\F(\sigma_2|\Prom\oplus z\oplus y_N))= -\ln(\pi_\F(\sigma_1|\Prom\oplus x\oplus y_N)\pi_\F(\sigma_2|\Prom\oplus z\oplus y_N))$.
    
       Now we consider the $S_a^{y_{N_0}}$, we show that such an answer set cannot be memorized with loss $\alpha$, which contradicts the assumption and leads to the theorem. By the assumption of $\F$, there is a $\Prom$ such that $\L(\F,\Prom\oplus S_q,S^{y_{N_0}}_a)\le \alpha$; by the definition of $y_{N_0}$, there is $\len(\Prom)\ge N_0$. Therefore, by the above result, there is $\pi_\F(\sigma_1|\Prom\oplus x\oplus y_{N_0})+\pi_\F(\sigma_2|\Prom\oplus z\oplus y_{N_0})\le 1+(\pi_\F(\sigma_2|\Prom\oplus z\oplus y_{N_0})-\pi_\F(\sigma_2|\Prom\oplus x\oplus y_{N_0}))\le 1.1$ by the assumption of $N_0$. Hence, using the equation $(x+y)^2/4\ge xy$, we have that $\L(\F,\Prom\oplus S_q,S_a^{y_{N_0}})\ge \ln(\pi_\F(\sigma_1|\Prom\oplus x\oplus y_N)\pi_\F(\sigma_2|\Prom\oplus z\oplus y_N))/2\ge -\ln (0.55)^2/2=-\ln 0.55$ for such a prompt $\Prom$, which contradicts the assumption; thus, we prove the theorem.
\end{proof}

\section{Proofs of Section \ref{sec-51}}

\subsection{Proof of Theorem \ref{th-51}}

We first prove several lemmas.
\begin{lemma}
\label{sx2}
    For any given transformer $\F$, there is an $A\in\R_+$ that only depends on $\F$ such that: for any $x\in\Sigma^{*}$ and $l$, $||\F_l(x)||_{2,\infty}\le A$, where $\F_l$ means the $l$-th hidden layer.
\end{lemma}
\begin{proof}

Based on the definition of the embedding vector, for any sample $x$, the input of the first hidden layer is the embedding matrix $v_x$, whose $L_{2,\infty}$ norm has an upper bound $A_0$.

To prove the lemma, it suffices to show that for the $l$-th hidden layer, if the $L_{2,\infty}$ norm of its input has an upper bound $A_{l-1}$ for any input $x$, then the $L_{2,\infty}$ norm of its output has an upper bound $A_{l}$ for any input $x$.

Assume that $\F_l(x)=\FNN(\F_{l-1}(x)+\ATT(\F_{l-1}(x)))+\F_{l-1}(x)+\ATT(\F_{l-1}(x))$.

In the attention layer, because the $j$-th row of the attention layer is $\frac{\sum_{i=1}^j e^{u_{j,i}}v_i}{\sum_{i=1}^j e^{u_{j,i}}}V_l$, where $u_{j,i}=v_{j}Q_lA_{j-i}K_lv^T_i$, $v_i$ is the $i$-th row of $\F_{l-1}(x)$ and $Q_l,K_l,V_l,A_{j-i}$ are parameters in the attention layer, consider that $||\frac{\sum_{i=1}^j e^{u_i}v_i}{\sum_{i=1}^j e^{u_i}}||_{2}\le \max_{i\le j}\{||v_i||_2\}$, so $||\ATT(\F_{l-1}(x))||_{2,\infty}\le ||\F_{l-1}(x)||_{2,\infty}||V_l||_2$.

Considering that the $\FNN$ layer does not involve interaction between rows of $\F_{l-1}(x)$, $||\FNN(x)||_{2,\infty}\le ||W_l||_2||x||_{2,\infty}+||b_l||_{2,\infty}$ holds for any $x$, where $W_l$ and $b_l$ are the parameters in $\FNN$. 

Combining the results of the attention layer and an FNN layer, we have that $||\F_l(x)||_{2,\infty}\le||W_l||_2(A_{l-1}+A_{l-1}||V_l||_2)+(A_{l-1}+A_{l-1}||V_l||_2)+||b_l||_{2,\infty}=A_l$, which is what we want.

Taking $A=\max_l\{A_l\}$, we prove the lemma.
\end{proof}

\begin{lemma}
\label{qn}
For any sequences $x$ and $z$, if the first $n$ symbols in $x$ and $z$ are the same, then for any transformer $\F$ and $j\in\Z_+$, the first $n$ rows in the outputs of the $j$-th hidden layer of $\F(x)$ and $\F(z)$ are the same.
\end{lemma}
\begin{proof}
Firstly, we prove that for the $j$-th hidden layer $\F^j$ of $\F$, if $x_1$ and $z_1$ satisfy that the first $n$ rows in $x_1$ and $z_1$ are the same, then the first rows $n$ of $\F^j(x_1)$ and $\F^j(z_1)$ are the same.

Firstly, we prove the situation for the first layer. If the first $n$ symbols in $x$ and $z$ are the same, then the first $n$ rows of their embedding matrix are also the same. Then, if it refers to the $(j-1)$-th layer, we now consider the $j$-th layer.

 Assume that $\F^j$ can be written as: $\F^j(x)=\F^{j-1}(x)+\sum_{k=1}^H\softmax(R(\F^{j-1}(x)Q_j,V_j\F^{j-1}(x)^T)+M)\F^{j-1}(x)K_j+\FNN(\F^{j-1}(x)+\sum_{k=1}^H\softmax(R(\F^{j-1}(x)Q_j,V_j\F^{j-1}(x)^T)+M)\F^{j-1}(x)K_j)$. 

In the residual layer and FNN layer, the calculations in this layer do not include the interactions between different rows; instead, they just apply the same transformation to each row. So, when $x_1$ and $z_1$ satisfy that the first $n$ rows in $x_1$ and $z_1$ are the same, $\FNN(x_1)$ and $\FNN(z_1)$ also satisfy that the first  $n$ symbols are the same.

 In the attention layer, considering the definition of $M$, for any $i
 \in\Z_+$ and given sample $x_1$, we know that the $i$-th row of $\softmax(R(xQ_k,V_kx^T)+M)$ can be written as $(x_iQ_kA_{i-1}V_kx_1^T,$ $x_iQ_kA_{i-2}V_kx_2^T,$ $x_iQ_kA_{i-3}V_kx_3^T,\dots,x_iQ_kA_{i-i}V_kx_i^T,$ $0,0,\dots,0)$, where $x_i$ is the $i$-th row of $x$. It is easy to see that it only depends on the first $i$-th row of $x$. So when samples $x$ and $z$ satisfy that the first $n$ rows in $x$ and $z$ are the same, then the first 
 $n$ rows of $\softmax(R(xQ_k,V_kx^T)+M)x$ and $\softmax(R(zQ_k,V_kz^T)+M)z$ are also the same. Hence, because the other parts in the attention layer do not include the interactions between different rows, we have that the whole output of the attention also satisfies that the first $n$ symbols are the same when input $x$ and $z$.

 By the above result and assumption, the first $n$ rows of 

 \hskip30pt
 $\sum_{k=1}^H \softmax(R(\F^{j-1}(x)Q_j,V_j\F^{j-1}(x)^T)+M)\F^{j-1}(x)K_j$ and 
 
\hskip30pt $\sum_{k=1}^H\softmax(R(\F^{j-1}(z)Q_j,V_j\F^{j-1}(z)^T)+M)\F^{j-1}(z)K_j$ 
  
are the same, similar to the residual layer and the FNN layer. Combining them, we can obtain the result. Thus, we prove the lemma. 
\end{proof}

\begin{lemma}
\label{l4}
Let $u\in(0,1)$ and $\F$ be a transformer. Then, for any $t\in \Sigma^{*}$ and $x,d\in\Sigma^{*}$, let $z\oplus t=z_t$, $x\oplus t=x_t$. we have that: $||\F(z_t)-\F(x_t)||_2\le \frac{C\max\{\len(x),\len(z)\}}{\len(t)^u}$ for some constant $C$ which depends only on $u,\F$.
\end{lemma}

\begin{proof}
%Firstly, by the Lemma \ref{sx2}, we know that $||\F(x_t)||_2$ and $||\F(x_z)||_2$ have an upper bound $A_m$, so we take $C>2A_m$, then such lemma must stand for $\len(t)\le \max\{\len(x),\len(z)\}$. So, now we just need to prove it for the situation that $\len(t)\ge \max\{\len(x),\len(z)\}$. 

Let $\F_{i}$ be the $i$-th hidden layer of $\F$ and  $\F_{i,j}$ the $j$-th weight of the $i$-th hidden layer of $\F$. 

For convenience, let $\max\{\len(x),\len(z)\}=h$ in this proof. Now we assume that the following result holds for the $i-1$-th hidden layer:  there exists a $C_{i-1}$ such that  $||\F_{i-1,\len(t)+\len(x)}(x_t)-\F_{i-1,\len(t)+\len(z)}(z_t)||_2<C_{i-1}h/\len(t)^u$ for any $x,z,t$. Then we can show that such a result is also satisfied for the $i$-th hidden layers. 

It is easy to see that the embedding vector of $x_t$ and $z_t$ satisfies the above result and can be viewed as a 0-hidden layer transformer; therefore, $i=0$ stands. Hence, consider that the output of the transformer only depends on the last row of $\F_{L}$, where $L$ is the depth of $\F$. Thus, the above result can easily lead to the lemma, and we just need to prove this result.
%Consider that $L=1$ implies the embedding vector of $x$, which is 

It is easy to see that the $i$-th hidden layer is:
$$\F_{i}(x_t)=\FNN_{i}(\ATT_i(\F_{i-1}(x_{t}))+\F_{i-1}(x_{t}))+\ATT_i(\F_{i-1}(x_{t}))+\F_{i-1}(x_{t}),$$
similar to $z_t$. Now we will show that there is a $C_i$ that depends only on $\F,u,C_{i-1}$ and satisfies $||\F_{i,\len(t)+\len(x)}(x_t)-\F_{i,\len(t)+\len(z)}(z_t)||_2<C_{i}h/\len(t)^u$ for any $x,z,t$, which is what we want. Such a result requires two parts.

{\bf The Attention Layer.}
Use $M[j]$ to denote the $j$-th row of matrix $M$. Regarding the attention layer, we have that:

$$\ATT_i(\F_{i-1}(x_t))[\len(x)+\len(t)]=\frac{\sum_{k=1}^{\len(x)+\len(t)} e^{u_k}\F_{i-1,k}(x_t)}{\sum_{k=1}^{\len(x)+\len(t)} e^{u_k}}V_i,$$ and $$\ATT_i(\F_{i-1}(z_t))[\len(z)+\len(t)]=\frac{\sum_{k=1}^{\len(z)+\len(t)} e^{u'_k}\F_{i-1,k}(z_{t})}{\sum_{k=1}^{\len(z)+\len(t)} e^{u'_k}}V_i,$$
where
 $u_k=R(\F_{i-1,\len(x_t)}(x_t)Q_i,K_i\F^{\tau}_{i-1,k}(x_t))$ and $u'_k=R(\F_{i-1,\len(z_t)}(z_t)Q_i,K_i$ $\F^{\tau}_{i-1,k}(z_t))$.

%To be convenient, let $\sum_{k=1}^{\len(x^{m,z})} e^{u'_k}\F_{L-1,k}(x^{m,z})=(\sum_{k=1}^{\len(x^m)-1} e^{u'_k}\F_{L-1,k}(x^{m,z})+e^{u'_{\len(x^{m,z})}}\F_{L-1,\len(x^{m,z})}(x^{m,z}))+(\sum_{k=\len(x^{m})}^{\len(x^{m,z})-1} e^{u'_k}\F_{L-1,k}(x^{m,z}))=(V_1)+(V_2)$, and $\sum_{k=1}^{\len(x^{m,z})} e^{u'_k}=(\sum_{k=1}^{\len(x^m)-1} e^{u'_k}+e^{u'_{\len(x^{m,z})}})+(\sum_{k=\len(x^m)}^{\len(x^{m,z})-1} e^{u'_k})=(W_1)+(W_2)$. And write $\ATT_{L}(\F_{L-1}(x^m))[\len(x^m)]=Q/P$.

%Then, we have that 
%\begin{equation*}
%\small
% \begin{array}{ll}
%&\ATT_L(\F_{L-1}(x^m))[\len(x^m)]-\ATT_L(\F_{L-1,\len(x^{m,z})}(x^{m,z}))[\len(x^{m,z})]\\
%=&\frac{(PV_1-QW_1)+(PV_2-QW_2)}{(W_1+W_2)P}V.
% \end{array}
%\end{equation*}

By the above lemmas, we have the following facts:

(1) Let the subset transformer of $\F$ consist of the first $i-1$ layer named $\F'$. Based on this assumption, there exists a $C_{i-1}$ such that $||\F_{i-1,\len(x)+\len(t)}(x_t)-\F_{i-1,\len(z)+\len(t)}(z_t)||_2\le C_{i-1}h/\len(t)^u$. Hence, by Lemma \ref{qn}, we have the following fact: for any $\len(t)$, there are $||\F_{i-1,\len(x)+j}(x_t)-\F_{i-1,\len(z)+j}(z_t)||_2\le C_{i-1}h/j^u$ for any $j$.

(2) By Lemma \ref{sx2},  $||\F_{i-1}(x_t)||_{2,\infty}$ and $||\F_{i-1}(z_t)||_{2,\infty}$ have the upper bound $A$; $||Q_iA_jK_i||_2$ have the upper bound $B$ for any $j$.

Now we consider the $||\ATT_i(\F_{i-1}(x_t))[\len(x)+\len(t)]-\ATT_i(\F_{i-1}(z_t))[\len(z)+\len(t)]||_2$. We will show that there is a $C'_i$ such that $||\ATT_i(\F_{i-1}(x_t))[\len(x)+\len(t)]-\ATT_i(\F_{i-1}(z_t))[\len(z)+\len(t)]||_2\le \frac{C'_ih}{\len(t)^u}$.

To be convenient, let $\sum_{k=1}^{\len(x)+\len(t)} e^{u_k}\F_{i-1,k}(x_t)=(\sum_{k=1}^{\len(x)} e^{u_k}\F_{i-1,k}(x_t))+(\sum_{k=\len(x)+1}^{\len(x)+\len(t)} e^{u_k}\F_{i-1,k}(x_t))=V_{1,x}+V_{2,x}$, and $\sum_{k=1}^{\len(x)+\len(t)} e^{u_k}=(\sum_{k=1}^{\len(x)} e^{u_k})+(\sum_{k=\len(x)+1}^{\len(x)+\len(t)} e^{u_k})=W_{1,x}+W_{2,x}$. Similarly, we also have $W_{1,z}+W_{2,z}$ and $V_{1,z}+V_{2,z}$.

Then we have that:
\begin{equation*}
    \begin{array}{cl}
         & \ATT_i(\F_{i-1}(x_t))[\len(x)+\len(t)]-\ATT_i(\F_{i-1}(z_t))[\len(z)+\len(t)] \\
         =&\frac{V_{1,x}+V_{2,x}}{W_{1,x}+W_{2,x}}V_i-\frac{V_{1,z}+V_{2,z}}{W_{1,z}+W_{2,z}}V_i\\
         =&\frac{(V_{1,x}W_{1,z}-V_{1,z}W_{1,x})+(V_{1,x}W_{2,z}-V_{1,z}W_{2,x})+(V_{2,x}W_{1,z}-V_{2,z}W_{1,x})+(V_{2,x}W_{2,z}-V_{2,z}W_{2,x})}{(W_{1,z}+W_{2,z})(W_{1,x}+W_{2,x})}V_i
    \end{array}
\end{equation*}

For the first part, we have the following:
\begin{equation*}
    \begin{array}{cl}
       &||\frac{V_{1,x}W_{1,z}-V_{1,z}W_{1,x}}{(W_{1,z}+W_{2,z})(W_{1,x}+W_{2,x})}||_2\\
       \le & \frac{||V_{1,x}W_{1,z}||_2+||V_{1,z}W_{1,x}||_2}{(\len(t)e^{-A^2B})((\len(t)+h)e^{-A^2B})}\\
       \le &\frac{2W_{1,z}W_{1,x}A}{\len(t)(\len(t)+h)e^{-2A^2B}}.\\       
    \end{array}
\end{equation*}

By fact (2), it is easy to see that $W_{1,z}$ and $W_{1,x}$ have an upper bound $e^{A^2B}h$, so the first part $||\frac{V_{1,x}W_{1,z}-V_{1,z}W_{1,x}}{(W_{1,z}+W_{2,z})(W_{1,x}+W_{2,x})}||_2\le \frac{2h^2Ae^{4A^2B}}{\len(t)(\len(t)+h)}\le \frac{2hAe^{4A^2B}}{\len(t)}$.

For the second part, 
\begin{equation*}
    \begin{array}{cl}
        &||\frac{V_{1,x}W_{2,z}-V_{1,z}W_{2,x}}{(W_{1,z}+W_{2,z})(W_{1,x}+W_{2,x})}||_2\\
       \le & \frac{W_{1,x}W_{2,z}A+W_{1,z}W_{2,x}A}{(W_{1,z}+W_{2,z})(W_{1,x}+W_{2,x})}\\
        =& \frac{W_{2,z}/(W_{1,z}+W_{2,z})W_{1,x}A}{W_{1,x}+W_{2,x}}+\frac{W_{2,x}/(W_{1,x}+W_{2,x})W_{1,z}A}{W_{1,z}+W_{2,z}}\\
        \le &\frac{W_{1,x}A}{W_{1,x}+W_{2,x}}+\frac{W_{1,z}A}{W_{1,z}+W_{2,z}}\\
        \le &\frac{2e^{A^2B}hA}{\len(t)e^{-A^2B}}.\\
    \end{array}
\end{equation*}

Then, similar to the first part, it is easy to see that $||\frac{V_{1,x}W_{2,z}-V_{1,z}W_{2,x}}{(W_{1,z}+W_{2,z})(W_{1,x}+W_{2,x})}||_2\le\frac{2hAe^{2A^2B}}{\len(t)}$. It is similar to the third part.

Now we consider the last part; we have the following:
\begin{equation*}
    \begin{array}{cl}
         &  |\sum_{k=\len(x)+1}^{\len(x)+\len(t)} e^{u_k}-\sum_{k=\len(z)+1}^{\len(z)+\len(t)}e^{u'_k}|
        \le \sum_{k=1}^{\len(t)}|e^{u_{k+\len(x)}}-e^{u'_{k+\len(z)}}|.
    \end{array}
\end{equation*}

Hence, by fact (1), we have that:
\begin{equation*}
    \begin{array}{cl}
         &|{u_{k+\len(x)}}-{u'_{k+\len(z)}}|\\
         =&|\F_{i-1,\len(t)+\len(x)}Q_iA_{\len(t)-k}K_i\F^T_{i-1,k+\len(x)}\\
         &-\F_{i-1,\len(t)+\len(z)}Q_iA_{\len(t)-k}K_i\F_{i-1,k+\len(z)}|\\
         \le&|(\F_{i-1,\len(t)+\len(x)}-\F_{i-1,\len(t)+\len(z)})Q_iA_{\len(t)-k}K_i\F_{i-1,\len(x)+k}|\\
         &+|\F_{i-1,\len(t)+\len(z)}Q_iA_{\len(t)-k}K_i(\F_{i-1,\len(x)+k}-\F_{i-1,\len(z)+k})|\\
         \le& 2BA(\frac{C_{i-1}}{k^u}+\frac{C_{i-1}}{\len(t)^u})h
    \end{array}
\end{equation*}

So we have that:  
$|e^{u_{\len(x)+k}}-e^{u'_{\len(z)+k}}|=|e^{u_{\len(x)+k}}(1-e^{-u_{\len(x)+k}+u'_{\len(z)+k}})|\le e^{A^2B}(e^{2ABh(\frac{C_{i-1}}{k^u}+\frac{C_{i-1}}{\len(t)^u})}-1)$; on the other hand, there is also $|e^{u_{\len(x)+k}}-e^{u'_{\len(z)+k}}|\le e^{A^2B}$. So, combining them, we have that $|e^{u_{\len(x)+k}}-e^{u'_{\len(z)+k}}|\le e^{A^2B}4ABh(\frac{C_{i-1}}{k^u}+\frac{C_{i-1}}{\len(t)^u})$.
So 
\begin{equation*}
    \begin{array}{cl}
  & |W_{2,x}-W_{2,z}| \\
        \le& \sum_{k=1}^{\len(t)}|e^{u_{k+\len(x)}}-e^{u'_{k+\len(z)}}|\\
        \le&\sum_{k=1}^{\len(t)}4ABhC_{i-1}e^{A^2B}(\frac{1}{k^u}+\frac{1}{\len(t)^u})\\
        \le& 4ABhC_{i-1}e^{A^2B}(\len(t)^{1-u}(1+1/(1-u)))\\
        =&O(h\len(t)^{1-u})
    \end{array}
\end{equation*}

It is easy to see that the coefficient in $O$ depends only on $C_{i-1},\F,u$. Then, similar to before, we have that:

\begin{equation*}
    \begin{array}{cl}
  & |V_{2,x}-V_{2,z}| \\
        \le& \sum_{k=1}^{\len(t)}||(e^{u_{k+\len(x)}}-e^{u'_{k+\len(z)}})\F_{i-1,k+\len(x)}(x_t)||_2+\\
        &\sum_{k=1}^{\len(t)}||e^{u'_{k+\len(z)}}(\F_{i-1,k+\len(z)}(z_t)-\F_{i-1,k+\len(x)}(x_t))||_2\\
        \le&\sum_{k=1}^{\len(t)}|(e^{u_{k+\len(x)}}-e^{u'_{k+\len(z)}})|A\\
        &+e^{A^2B}\sum_{k=1}^{\len(t)}||\F_{i-1,k+\len(z)}(z_t)-\F_{i-1,k+\len(x)}(x_t)||_2\\
        =&\sum_{k=1}^{\len(t)}|(e^{u_{k+\len(x)}}-e^{u'_{k+\len(z)}})|A+e^{A^2B}\sum_{k=1}^{\len(t)}\frac{C_{i-1}h}{k^u}\\
        =&O(h\len(t)^{1-u})
    \end{array}
\end{equation*}

So, we write $W_\Delta=|W_{2,x}-W_{2,z}|$ and $V_{\Delta}=||V_{2,x}-V_{2,z}||_2$, and we have that: 
\begin{equation*}
    \begin{array}{cl}
        &||\frac{V_{2,x}W_{2,z}-V_{2,z}W_{2,x}}{(W_{1,z}+W_{2,z})(W_{1,x}+W_{2,x})}||_2\\
        \le & \frac{||V_{2,x}W_{\Delta}||_2+|W_{2,x}V_{\Delta}|}{\len(t)^2e^{-2A^2B}}\\
        \le & O(\frac{h\len(t)^{2-u}}{\len(t)^2})+O(h\frac{\len(t)^{2-u}}{\len(t)^2}) \\
    \end{array}
\end{equation*}

We use $W_{2,x}\le \len(t)e^{A^2B}$ and $||V_{2,x}||_2\le A\len(t)e^{A^2B}$ here, and the coefficient in $O$ depends only on $A,B,C_{i-1}$. So, combining such four parts, we have that $||\ATT_i(\F_{i-1}(x_t))[\len(x)+\len(t)]-\ATT_i(\F_{i-1}(z_t))[\len(z)+\len(t)]||_2\le O(\frac{h}{\len(t)}+h/\len(t)^u)=O(\frac{h}{\len(t)^u})$, which implies that $||\ATT_i(\F_{i-1}(x_t))[\len(x)+\len(t)]-\ATT_i(\F_{i-1}(z_t))[\len(z)+\len(t)]||_2\le\frac{C'_ih}{\len(t)^u}$ for some $C'_i\in\R_+$.

{\bf The FNN Layer.}

For the FNN layer, it is easy to see that for any Relu network $\FNN_r$ with transition matrix $W$, there is $||\FNN_r(x)-\FNN_r(z)||_2\le||W||_2||x-z||_2$.

So consider that $||\ATT_{i}(\F_{i-1}(x_t))[\len(x)+\len(t)]-\ATT_{i}(\F_{i-1}(z_t))[\len(z)+\len(t)]||\le C'_ih/\len(t)^u$ and $||\F_{i,\len(x)+\len(t)}(x_t)-\F_{i,\len(z)+\len(t)}(z_t)||_2\le C_{i-1}h/\len(t)^u$; it is obvious that
\begin{equation*}
    \begin{array}{ll}
&||\FNN_{i}(\ATT_{i}(\F_{i-1}(x_t))+\F_{i-1}(x_t))[\len(x)+\len(t)]\\
-&\FNN_{i}(\ATT_{i}(\F_{i-1}(z_t))+\F_{i-1}(z_t))[\len(z)+\len(t)]||_2\\
\le &||W_i||_2(C'_i+C_{i-1})h/\len(t)^u,
\end{array}
\end{equation*}

Finally, by combining the gap from the attention layer and the FNN layer, and taking $C_i=(||W_i||_2+1)(C'_i+C_{i-1})$, we prove the result.
\end{proof}

%Now we prove the following proposition:

\begin{lemma}
\label{prop-th1}
For any transformer $\F$ and $\epsilon,c\in\R_{>0},d\in(0,1)$, 
there exists an $L\in\Z_{>0}$ such that if a task $S=\{(x_1,y_1)\}$ satisfies $\len(y_1)\ge L$, then {\em $|\L_{\text{nll}}(\F,x_1,y_1)-\L_{\text{nll}}(\F,\Prom\oplus x_1,y_1)|\le \epsilon(\len(y_1)+1)$} for any $\Prom$ satisfying $\len(\Prom)\le c\len(y_1)^d$.
\end{lemma}
\begin{proof}
%We just need to prove that: $\lim_{\len(y)\to\infty,\len(\Prom)\le c\len(y)^d}|\L_{\text{nll}}(\F,x,y)-\L_{\text{nll}}(\F,\Prom\oplus x,y)|\to 0$. 

    If not, we know that there is a $\epsilon$ such that, for any $L$, there is a prompt $\Prom$, query $x$, and answer $y$ that satisfy $|\L(\F,\Prom\oplus x,y)-\L(\F,x,y)|>\epsilon(\len(y)+1)$, $\len(y)\ge L$, and $\Prom\le c\len(y)^d$. 

Let $y_j=y[1,2,\dots,j]$. Take a $z'\in(d,1)$. By Lemma \ref{l4},  we know that there is a $C_\F$ such that $||\F(\Prom\oplus x\oplus y_j)-\F(x\oplus y_j)||_2\le\frac{C_\F\len(\Prom)}{(\len(x)+j)^{z'}}$. 

    So we have that: $e^{-2\frac{C_\F\len(\Prom)}{(\len(x)+j)^{z'}}}\le\pi_{\F}(y[j+1]|\Prom\oplus x\oplus y_j)/\pi_{\F}(y[j+1]|x\oplus y_j)\le e^{2\frac{C_\F\len(\Prom)}{(\len(x)+j)^{z'}}}$; hence, $|-\ln(\pi_{\F}(y[j+1]|\Prom\oplus x\oplus y_j))-(-\ln(\pi_{\F}(y[j+1]|x\oplus y_j)))|\le 2\frac{C_\F\len(\Prom)}{(\len(x)+j)^{z'}}$. Here, we use $y[\len(y)+1]$ to represent $\sigma_0$.

    So there is:
    \begin{equation*}
        \begin{array}{cl}
             & |\L_{\text{nll}}(\F,x,y)-\L_{\text{nll}}(\F,\Prom\oplus x,y)| \\
        \le& {2C_\F\len(P)}\sum_{i=0}^{\len(y)} \frac{1}{(1+i)^{z'}}\\
            \le&O(\frac{\len(P)}{(\len(y)+1)^{z'}}(\len(y)+1))
        \end{array}
    \end{equation*}

By $z'>d$ and $\len(\Prom)\le c\len(y)^d$, we know that for any $\epsilon$, take $L\ge (Ac/\epsilon)^{1/(z'-d)}$ where $A$ is the coefficient in $O$, there must be $\frac{\len(P)}{\len(y)^{z'}}\le\epsilon/A$ for any $\len(y)\le L$, which implies that $|\L_{\text{nll}}(\F,x,y)-\L_{\text{nll}}(\F,\Prom\oplus x,y)|\le\epsilon (\len(y)+1)$, which is contradictory to the existence of $\epsilon$ and the arbitrariness of $L$ in the assumption, so we prove the lemma.
\end{proof}     

%Such proposition can be directly deduced to corollary \ref{th1-b}. Hence, use such proposition, 
We now prove Theorem \ref{th-51}.

\begin{proof}
Using Lemma \ref{prop-th1}, we have
\begin{equation*}
    \begin{array}{cl}
         &\L(\F,S)  \\
         =&\frac{1}{|S|}\sum_{i=1}^{|S|} \L_{\text{nll}}(\F,x_i,y_i)\\
         \ge&\frac{1}{|S|}\sum_{i=1}^{|S|} \L_{\text{nll}}(\F,\Prom\oplus x_i,y_i)-{\epsilon}{(\len(y_i)+1)}\\
         \ge&\frac{1}{|S|}\sum_{i=1}^{|S|} \L_{\text{nll}}(\F,\Prom\oplus x_i,y_i)-{\epsilon}{(\len(S)+1)}\\
         =&\L(\F,\Prom\oplus S)-\epsilon(\len(S)+1)
    \end{array}
\end{equation*}
which is what we want.
\end{proof}

\subsection{Proof of Proposition \ref{prop1} and Corollary\ref{zss}}

\begin{proof}
It is easy to see that there exist at most $(T+1)^{cL^{d}}$ prompts with lengths no more than $cL^{d}$. For any $x$ and prompt $\Prom$, there exist at most $e^{\alpha }$ answers $y$ such that $\L_{\text{nll}}(\F,\Prom\oplus x,y)\le \alpha$. There is $T^L$ number of answers with length $L$. So we just need to show that $T^L>e^\alpha(T+1)^{cL^{d}}$.

Note that $L>({\log_T({T+1})c+\alpha})^{\frac{1}{1-d}}$ implies that $\log_T({T+1})cL^{d}+\alpha<L$ when $d\in(0,1)$. And $L>\frac{\alpha}{1-c\log_{T}(T+1)}$ also implies that $\log_T({T+1})cL^{d}+\alpha<L$ when $d\in(0,1)$. 

So we have that $(T+1)^{cL^{d}}e^{\alpha}\le T^{\log_T(T+1)cL^d+\alpha}< T^L$, and Proposition \ref{prop1} and Corollary \ref{zss} are proved.
\end{proof}

\section{Proofs of Section \ref{sec-52}}

%\subsection{Proof of Theorem \ref{prop-g12}}

%\begin{proof}
%Consider the following datasets with two pairs of samples in it:  for any $x,y\in\Sigma^*$, let $S_{x,y}=\{(x_1,y_1),(x_2,y_2)\}=\{(x,\sigma_1\oplus\sigma_2\oplus y),(x\oplus \sigma_1,\sigma_1\oplus y)\}$. Take $\alpha=\ln2$. Now we show that $S_{x,y}$ and such $\alpha$ are what we want.
%
%For any transformer $\F$, $x,y$, and prompt $\Prom$, we have that $\pi_\F(\sigma_1|\Prom\oplus x\oplus \sigma_1)+\pi_\F(\sigma_2|\Prom\oplus x\oplus \sigma_1)\le 1$. Therefore, for any $S_{x,y}$, there are: $\L_{\text{nll}}(y_1|x_1)+\L_{\text{nll}}(y_2|x_2)\ge -\ln\pi_\F(\sigma_2|x_1\oplus \sigma_1)-\ln\pi_\F(\sigma_1|x_2)=-\ln\pi_\F(\sigma_2|x\oplus \sigma_1)-\ln\pi_\F(\sigma_1|x\oplus \sigma_1)\ge 2\ln2$, so $\L(\F,S)\ge \ln2=\alpha$, which is what we want.
%\end{proof}

\subsection{Proofs of Theorem \ref{cxxx} and Proposition \ref{lj}, Corollary \ref{lj1} }
We need the following lemma first.
\begin{lemma}
\label{l5}
For any transformer $\F$, there is a constant $C$ such that: for any $x,z,t\in\Sigma^{*}$, $x_t=t\oplus x$, $z_t=t\oplus z$, $\len(x)=\len(z)$, and $x[\len(x)]=z[\len(z)]$, there is $||\F(x_t)-\F(z_t)||_2\le C\len(x)/\len(t)$.
\end{lemma}

\begin{proof}
We assume the $l$-th hidden layer of $\F(x)$ is $\F_{l}(x)$, and the $\F_{l,k}(x)$ is the $k$-th row of $\F_l(x)$. 

For convenience, let $h=\max\{\len(x),\len(z)\}$. Now we assume that the following result holds for the $L-1$-th hidden layer: there is a $C_{L-1}$ such that  $||\F_{L-1,\len(t)+\len(x)}(x_t)-\F_{L-1,\len(t)+\len(z)}(z_t)||_2<C_{i-1}h/\len(t)$ for any $x,z,t\in \Sigma^{*}$. Then we can show that such a result is also satisfied for the $L$-th hidden layers. 

Consider that in the embedding layer, which can be seen as the 0-th hidden layer, by $x[\len(x)]=z[\len(z)]$, this result stands for the constant $C=0$. Then consider that the output of the transformer depends only on the last row of the last attention layer; if the above result holds, it can easily lead to the lemma. 

By the definition of the transformer, we know that $\F_{L}(x_t)=\FNN_{L}(\F_{L-1}(x_t)+\ATT_{L}(\F_{L-1}(x_t)))+\F_{L-1}(x_t)+\ATT_{L}(\F_{L-1}(x_t))$, similar to $\F_{L}(z_t)$. We prove the result in two parts.

{\bf The attention layer gap.}

Firstly, we consider the attention layer; 
we will show that 
$$||\ATT_{L}(\F_{L-1}(x_t))[\len(x_t)]-\ATT_{L}(\F_{L-1}(z_t))[\len(z_t)]||_2\le C'_{L}h/\len(t)$$
for some constant $C'_{L}$, which does not depend on $x,z,t$.

It is easy to see that the last row of $\ATT_{L}(\F_{L-1}(x_t))$ is $\frac{\sum_{k=1}^{\len(x_t)}e^{u_k}\F_{L-1,k}(x_t)}{\sum_{k=1}^{\len(x_t)} e^{u_k}}V_{L}$, where $u_k=\F_{L-1,\len(x_t)}(x_t)Q_{L}$ $A_{\len(x_t)-k}K_{L}\F_{L-1,k}(x_t)^T$.

Similarly, the last row of $\ATT_L(\F_{L-1}(z_t))$ is $\frac{\sum_{k=1}^{\len(z_t)}e^{u'_k}\F_{L-1,k}(z_t)}{\sum_{k=1}^{\len(z_t)} e^{u'_k}}V_{L}$, where $u'_k=\F_{L-1,\len(z_t)}(z_t)Q_{L}$ $A_{\len(z_t)-k}K_{L}\F_{L-1,k}(z_t)^T$.

To be convenient, let $\sum_{k=1}^{\len(x_t)}e^{u_k}\F_{L-1,k}(x_t)=(\sum_{k=1}^{\len(t)}e^{u_k}\F_{L-1,k}(x^m))+(\sum_{k=\len(t)+1}^{\len(x_t)}e^{u_k}\F_{L-1,k}(x^m))=V_{x,1}+V_{x,2}$ and $\sum_{k=1}^{\len(x_t)}e^{u_k}=(\sum_{k=1}^{\len(t)}e^{u_k})+(\sum_{k=\len(t)+1}^{\len(x_t)}e^{u_k})=W_{x,1}+W_{x,2}$. Similarly, we have $V_{z,i}$ and $W_{z,i}$ where $i=1,2$.

And by the above assumptions and lemma, we have the following facts:

(1)  Based on the assumption, there is a $C_{L-1}$ such that $||\F_{L-1,\len(x_t)}(x_t)-\F_{L-1,\len(z_t)}(z_t)||_2\le C_{L-1}h/\len(t)$.

(2)  By the Lemma \ref{sx2},  the $||\F_{L-1}(x_t)||_{2,\infty}$ and $||\F_{L-1}(z_t)||_{2,\infty}$ have the upper bound $A$; $||Q_LA_jK_L||_2$ also has an upper bound $B$ for any $j$.

Then we have that: 
\begin{equation}
\label{msf}
    \begin{array}{ll}
         & \ATT_{L}(\F_{L-1}(x_t))[\len(x_t)]-\ATT_{L}(\F_{L-1}(z_t))[\len(z_t)]\\
         =&(\frac{V_{x,1}+V_{x,2}}{W_{x,1}+W_{x,2}}-\frac{V_{z,1}+V_{z,2}}{W_{z,1}+W_{z,2}})V_L\\
         =&\frac{(V_{x,1}W_{z,1}-V_{z,1}W_{x,1})+(V_{x,1}W_{z,2}-V_{z,1}W_{x,2})+(V_{x,2}W_{z,1}-V_{z,2}W_{x,1})+(V_{x,2}W_{z,2}-V_{z,2}W_{x,2})}{(W_{x,1}+W_{x,2})(W_{z,1}+W_{z,2})}V_L.
    \end{array}
\end{equation}

Now, we will estimate the values of these four parts one by one. For the first part, let $W_\Delta=W_{z,1}-W_{x,1}$ and $V_\Delta=V_{z,1}-V_{x,1}$, we have that:

\begin{equation}
\label{ms1}
    \begin{array}{cc}
         & ||\frac{V_{x,1}W_{z,1}-V_{z,1}W_{x,1}}{(W_{z,1}+W_{z,2})(W_{z,1}+W_{z,2})}||_2 \\
        \le & \frac{||V_{x,1}W_\Delta||_2+||W_{x,1}V_\Delta||_2}{(W_{z,1}+W_{z,2})(W_{z,1}+W_{z,2})}\\
        \le &  \frac{|W_\Delta| A}{W_{z,1}+W_{z,2}}+\frac{||V_\Delta||_2}{W_{z,1}+W_{z,2}}\\
        \le &\frac{|W_\Delta| A+||V_\Delta||_2}{\len(t)e^{-A^2B}}.\\
    \end{array}
\end{equation}

Firstly, for $k\in[\len(t)]$, we have that $|u_k-u'_k|=|(\F_{L-1,\len(x_t)}-\F_{L-1,\len(z_t)})$ $Q_LA_{L-k}K_L\F_{L-1,k}^T|\le C_{L-1}BAh/\len(t)$, and consider that $u_k\le A^2B$, by Lemma \ref{qn} and the above facts, so there are: 
\begin{equation*}
    \begin{array}{cl}
         & |W_\Delta|=|W_{x,1}-W_{z,1}|\\
         \le&\sum_{k=1}^{\len(t)}|e^{u_k}-e^{u'_k}| \\
         =&\sum_{k=1}^{\len(t)}e^{u_k}|(e^{u'_k-u_k}-1)|\\
\le&2\len(t) e^{A^2B}C_{L-1}BAh/\len(t)=O(h)\\
    \end{array}
\end{equation*}

Similarly, by Lemma \ref{qn}, we also have that:
\begin{equation*}
    \begin{array}{cl}
         & ||V_\Delta||_2=||V_{x,1}-V_{z,1}||_2\\
         \le&\sum_{i=1}^{\len(t)}||e^{u_k}\F_{L-1,k}(x_t)-e^{u'_k}\F_{L-1,k}(z_t)||_2 \\
         \le& \sum_{i=1}^{\len(t)}|e^{u_k}-e^{u'_k}|*||\F_{L-1,k}(x_t)||_2 \\
         =&\sum_{i=1}^{\len(t)}|e^{u_k}-e^{u'_k}|A=O(h).\\
    \end{array}
\end{equation*}

So, taking them into account in equation \ref{ms1}, we have $||\frac{V_{x,1}W_{z,1}-V_{z,1}W_{x,1}}{(W_{z,1}+W_{z,2})(W_{z,1}+W_{z,2})}||_2\le O(h/\len(t))$, where the coefficient in $O$ depends only on $\F$.

For part two, we have that:

\begin{equation}
\label{m3sy}
    \begin{array}{cc}
         & ||\frac{V_{x,1}W_{z,2}-V_{z,1}W_{x,2}}{(W_{z,1}+W_{z,2})(W_{z,1}+W_{z,2})}||_2 \\
        \le & \frac{||V_{x,1}W_{z,2}||_2+||W_{x,2}V_{z,1}||_2}{(W_{z,1}+W_{z,2})(W_{z,1}+W_{z,2})}\\
        \le &  \frac{|W_{z,2}|A}{W_{z,1}+W_{z,2}}+\frac{||V_{z,2}||_2}{W_{z,1}+W_{z,2}}\\
        \le &\frac{ 2Ae^{A^2B}h}{\len(t)e^{-A^2B}}=O(h/\len(t))\\
    \end{array}
\end{equation}

Part three can be proved similarly. For part four, we have the following:

\begin{equation}
\label{msy2}
    \begin{array}{cl}
         & ||\frac{V_{x,2}W_{z,2}-V_{z,2}W_{x,2}}{(W_{z,1}+W_{z,2})(W_{z,1}+W_{z,2})}||_2 \\
        \le & \frac{ 2(Ae^{A^2B}h)^2}{\len(t)(\len(t)+h)e^{-2A^2B}}\\
        \le & \frac{ 2(Ae^{A^2B})^2h}{\len(t)e^{-2A^2B}}=O(h/\len(t))\\
    \end{array}
\end{equation}

By combining such four parts into the equation \ref{msf}, we have that: 

\noindent$||\ATT_{L}(\F_{L-1}(x_t))[\len(x_t)]-\ATT_{L}(\F_{L-1}(z_t))[\len(z_t)]||_2\le O(h/\len(t))$, which implies that there is a $C'_L$ such that $||\ATT_{L}(\F_{L-1}(x_t))[\len(x_t)]-\ATT_{L}(\F_{L-1}(z_t))[\len(z_t)]||_2\le C'_Lh/\len(t)$ and $C'_L$ only depend on $C_{L-1},\F$. 

{\bf The gap of the FNN layer.}

Because $||\ATT_{L}(\F_{L-1}(x_t))[\len(x_t)]-\ATT_{L}(\F_{L-1}(z_t))[\len(z_t)]||_2\le C'_Lh/\len(t)$, which is so similar to the proof of theorem \ref{th-51}, it is obvious that 
\begin{equation*}
 \begin{array}{ll}
&||\FNN_{L}(\ATT_{L}(\F_{L-1}(x_t))+\F_{L-1}(x_t))[\len(x)+\len(t)]\\
-&\FNN_{L}(\ATT_{L}(\F_{L-1}(z_t))+\F_{L-1}(z_t))[\len(z)+\len(t)]||_2\\
\le &||W_L||_2(C'_L+C_{L-1})h/\len(t),
\end{array}
\end{equation*}
 So, taking $C_L=(||W_L||_2+1)(C'_L+C_{L-1})$, we prove the result.

Now, we can prove 
Theorem \ref{cxxx} can be easily proved from       Proposition \ref{th2}.

Finally, combining the gap from the attention layer and FNN layer, we prove the result.
\end{proof}

We first prove the following proposition using the lemma mentioned above.
\begin{proposition}
\label{th2}
For any transformer $\F$, $\epsilon$, there exists an $(c,d)$ such that for any $x,x',y$ satisfied $\len(x)=\len(x')$, $\last(x)=\last(x')$, $\len(y)\ge 1$, there is $|\L_{\text{nll}}(\F,\Prom\oplus x,y)-\L_{\text{nll}}(\F,\Prom\oplus x',y)|\le \epsilon$ for any prompt such that $\len(\Prom)\ge c(\len(x)+\len(y))^d$.
%{\color{red}
%If $\len(x)=\len(x')$, $x[\len(x)]=x'[\len(x')]$ then $x=x'$?}
\end{proposition}

\begin{proof}
Firstly, for the given transformer $\F$, let query $x,x'$ satisfy $\len(x)=\len(x')$, $\last(x)=\last(x')$ and answer $y$, prompt $\Prom$. Let $y_i=y[1]\oplus y[2]\oplus\dots\oplus y[i]$ and $y[j]$ be the j-th symbol of $y$. Then we have that: $||\F(\Prom\oplus x\oplus y_j)-\F(\Prom\oplus x'\oplus y_j)||\le\frac{C_\F(\len(x)+j)}{\len(\Prom)}$ according to Lemma \ref{l5} for some constant $C_\F$; hence, there is $|-\ln(\pi_{\F}(y[j+1]|\Prom\oplus x\oplus y_j))-(-\ln(\pi_{\F}(y[j+1]|\Prom\oplus x'\oplus y_j)))|\le 2\frac{C_\F(\len(x)+j)}{\len(\Prom)}$.

   Let $y[\len(y)+1]=\sigma_0$, so there are 
    \begin{equation}
    \label{ydyy}
        \begin{array}{cl}
             & |\L_{\text{nll}}(\F,\Prom\oplus x,y)-\L_{\text{nll}}(\F,\Prom\oplus x',y)| \\
            = &|\sum_{i=0}^{\len(y)}-\ln(\pi_{\F}(y[j+1]|\Prom\oplus x\oplus y_j))-(-\ln\pi_{\F}(y[j+1]|\Prom\oplus x'\oplus y_j))|\\
            \le& \frac{2C_\F}{\len(P)}\sum_{i=0}^{\len(y)}{(\len(x)+i)}\\
            \le & O(\frac{(\len(y)+1)(\len(x)+\len(y))}{\len(P)})
        \end{array}
    \end{equation}

When $\len(P)\ge c(\len(x)+\len(y))^d$, there is $\frac{(\len(y)+1)(\len(x)+\len(y))}{\len(P)}\le\frac{2(\len(y)+1)^2}{\len(P)}\le \frac{2}{c(1+\len(y))^{d-2}}$; so consider that $1+\len(y)\ge2$, when $d\ge 2+\log_2^{2A_\F/c\epsilon}$, $A_\F$ is the coefficient in $O$ which only depends on $\F$, there is $|\L_{\text{nll}}(\F,\Prom\oplus x,y)-\L_{\text{nll}}(\F,\Prom\oplus x',y)|\le \epsilon$, which is what we want. So Proposition \ref{th2} stands. 
%Hence, when $\len(P)\ge c(\len(x)+\len(y))^d$ and $e>2$, there is $\frac{(\len(y)+1)(\len(x)+\len(y))}{\len(P)}\le \frac{1}{c(\len(x)+\len(y))^{e-2}}\to 0$ when $\len(y)\to \infty$, so similar as before, proposition \ref{lj} is stand.
\end{proof}

Now, we can prove the theorem \ref{cxxx}, which  can be easily proved from     Proposition \ref{th2}.
\begin{proof}
    We just take the $c,d$ in proposition \ref{th2}, so for any $(c,d)$ prompt, there must be $|\L_{\text{nll}}(\F,\Prom\oplus x,y)-\L_{\text{nll}}(\F,\Prom\oplus x',y)|\le \epsilon$, so when $\L_{\text{nll}}(\F,\Prom\oplus x,y)\le\alpha$, there must be $\L_{\text{nll}}(\F,\Prom\oplus x',y)\le\alpha+\epsilon$, which is what we want. 
\end{proof}

Proposition \ref{lj} can be proved as follows.
\begin{proof}
Following the equation \ref{ydyy}. 
When $\len(P)\ge c(\len(x)+\len(y))^d$ and $d>2$, there is $\frac{(\len(y)+1)(\len(x)+\len(y))}{\len(P)}\le \frac{1}{c(\len(x)+\len(y))^{d-2}}\le \epsilon$ when $\len(y)\ge(1/c\epsilon)^{1/(d-2)}$, so similar to before, take $L=(A_\F /c\epsilon)^{1/(d-2)}$, the Proposition \ref{lj} stands.
\end{proof}

Corollary \ref{lj1} can be proved as follows.
\begin{proof}
Following the equation \ref{ydyy}, we have that:
\begin{equation*}
        \begin{array}{cl}
             & |\L_{\text{Nnll}}(\F,\Prom\oplus x,y)-\L_{\text{Nnll}}(\F,\Prom\oplus x',y)| \\
            = &|\frac{1}{\len(y)+1}\sum_{i=0}^{\len(y)}-\ln(\pi_{\F}(y[j+1]|\Prom\oplus x\oplus y_j))-(-\ln\pi_{\F}(y[j+1]|\Prom\oplus x'\oplus y_j))|\\
            \le& \frac{2C_\F}{\len(P)(\len(y)+1)}\sum_{i=0}^{\len(y)}{(\len(x)+i)}\\
            \le & O(\frac{\len(x)+\len(y)}{\len(P)})
        \end{array}
    \end{equation*}
When $\len(P)\ge c(\len(x)+\len(y))^d$ and $d>1$, there is $O(\frac{\len(x)+\len(y)}{\len(P)})\le O(\frac{1}{c(\len(x)+\len(y))^{d-1}})\le \epsilon$ when $\len(y)\ge(A_\F /c\epsilon)^{\frac{1}{d-1}}$, where $A_\F$ in the coefficient in $O$. Thus, similar to before, Corollary \ref{lj1} stands.
\end{proof}

\subsection{Proof of Proposition \ref{cx1}}

Firstly, we prove the following easy case.
\begin{lemma}
    \label{cx2}
For any transformer $\F$, $\epsilon$, there exists a pair of $c,d$ such that  for any task $S=\{(x_i,y_i)\}_{i=1}^N$ satisfying $\len(y_i)>\len(x_i)$, $\len(x_i)=\len(x_j)$, and $\last(x_i)=\last(x_j)$ for all $i,j\in[N]$, there is $|\L(\F,\Prom\oplus S)-\L(\F,\Prom\oplus S_{k})|\le\epsilon$ for any $k\in[N]$ and $(c,d)$ prompt $\Prom$, where $S_k=\{(x_k,y_i)\}_{i=1}^N$ when $k\in[N]$.
\end{lemma}
\begin{proof}
Lemma \ref{cx2} follows from the proof of Theorem \ref{cxxx}. 
Let the $c,d$ be from  Theorem \ref{cxxx}. We show that such $L$ is also what we want in this corollary.

Let $S_k=\{(x_k,y_i)\}_{i=1}^N$.
We have $|\L(\F,\Prom\oplus S)-\L(\F,\Prom\oplus S_{k})|\le\frac{1}{N}\sum_{i=1}^{N}|\L_{\text{nll}}(\F,x_i,y_i)-\L_{\text{nll}}(\F,x_k,y_i)|$. Note that $\len(x_i)=\len(x_k)$ and $\last(x_i)=\last(x_k)$. Hence, by the $\len(y_i)\ge\len(x_i)$, and according to the proof of Theorem \ref{cxxx}, it holds $|\L_{\text{nll}}(\F,\Prom\oplus x_i,y_i)-\L_{\text{nll}}(\F,\Prom\oplus x_k,y_i)|\le\epsilon$ for such a $\Prom$. Therefore, $|\L(\F,\Prom\oplus S)-\L(\F,\Prom\oplus S_{k})|\le\frac{1}{N}\sum_{i=1}^{N}\epsilon=\epsilon$.
\end{proof}
We now prove Proposition \ref{cx1}.
\begin{proof}
Let the $c,d$ be from Lemma \ref{cx2}. We show that such $c,d$ is also what we want in this corollary.

    Let $S=\cup_{i\in\Z_+}S_i$, where $S_i\subset S$ and $(x_j,y_j)\in S_i$ if and only if $\len(x_j)=i$. For $S_i\ne\phi$, write $S'_i=\{(x_{i'},y_{j})\}_{(x_j,y_j)\in S_i,j\in[N]}$, where $i'=\argmin_{j\in[N]}\{(x_j,y_j)\in S_i\}$.
    
 It is easy
 to see that $\L(\F,S)=\frac{1}{N}\sum_{i\in\Z_+}\sum_{(x,y)\in S_i}\L_{\text{nll}}(\F,x,y)=\sum_{i\in\Z_+}\frac{|S_i|\L(\F,S_i)}{N}$.

Now we consider the $i$ such that $S_i\ne\phi$. By Corollary \ref{cx2} and the definition of $S_i'$, we know that when $\len(y_i)>\len(x_i)$ for all $i\in[N]$, it holds $|\L(\F,S_i)-\L(\F,S'_{i})|\le\epsilon$. Hence, because $y_i\ne y_j$, so that $\sum_{j\in[N],(x_{j},y_{j})\in S_i}\pi_\F(y_j|{x_i'})\le 1$, hence $\L(\F,S'_{i})= \frac{1}{|S_i|}\sum_{j\in[N],(x_{j},y_{j})\in S_i}-\ln\pi_\F(y_j|{x_i'})\ge\ln|S_i|$, we have that $\L(\F,S_i)\ge \ln|S_i|-\epsilon$. 

Considering that $|S_i|=N_i$ and $0\ln 0=0$, we have that $\L(\F,S)\ge\sum_{i\in\Z_+}\frac{|S_i|\L(\F,S_i)}{N}\ge\sum_{i\in\Z_+}\frac{N_i\ln(N_i)}{N}-\epsilon$.
\end{proof}

\section{Proofs of Section \ref{sec-53}}

\subsection{Proof of Proposition \ref{pl1}}

\begin{proof}
Without loss of generality, we consider the case that $T=3$, that is $\Sigma=\{\sigma_i\}_{i=0}^3$. We will consider the transformer $\F$ with one hidden layer as follows:

(1) Embedding layer: the embedding layer of $\sigma_i$ is $v_i$ and satisfies that $v_1=(1,0)$ and $v_i=(0,1)$ when $i>1$;

(2) Hidden layer: In the attention layer, $A_j=I$, $QK=0$, and $V=I$; the FNN layer is 0; thus, the last row of the hidden layer $\F_l(x)$ is calculated as $\F_l(x)=v_{x_{\len(x)}}+\frac{\sum_{j=1}^{\len(x)} v_{x_j}}{\len(x)}$, where $x_j$ is the $j$-th symbol in $x$, and $v_{x_j}$ is the embedding vector of $x_j$;

(3) Output layer and FNN layer: The output layer of $\F$ is as follows: the weight of the second weight (corresponding to $\sigma_2$) is $\F_1(x)=C(\F_l(x)[1]-c(\F_l(x)[2]-1))$, the $4$-th weight (corresponding to $\sigma_0$) is $\F_4(x)=-C(\F_l(x)[1]-c(\F_l(x)[2]-1))$, while others are 0. Here, $C$ is a constant such that: $-\ln\frac{e^{-Cc}}{e^{-Cc}+e^{Cc}+2}\ge 2\alpha_2$, $\frac{2}{e^{C\frac{3\alpha_1c}{1+c}}} \le \alpha_1/2$; it is easy to see that such $C$ must exist.

Let $\sigma_i^m=\sigma_i\oplus\sigma_i\oplus\sigma_i\oplus\dots\oplus\sigma_i$. Take $x=\sigma^1_2$. Then, for any given length $l$, we can find a $L\ge l$ such that $\frac{[cL]-(1-4\alpha_1)cL}{[cL]+L} \ge \frac{3\alpha_1c}{c+1}$ and $-\ln\frac{e^{-C\frac{[cL]-cL}{[cL]+L}}}{e^{C\frac{[cL]-cL}{[cL]+L}}+e^{-C\frac{[cL]-cL}{[cL]+L}}+2}\le2$. we consider the task $S=\{(\sigma^1_2,\sigma_2^L)\}$ and show that such a dataset is what we want. We have that:

(1) Without prompt.  It is easy to see that $-\ln\pi_\F(\sigma_2|\sigma_2^K)=-\ln\frac{e^{-Cc}}{e^{-Cc}+e^{Cc}+2}\ge \alpha_2$ for any $K>0$, so $\L_{\text{Nnll}}(\F,\sigma^1_2,\sigma_2^L)\ge \frac{\sum_{k=1}^L\pi_\F(\sigma_2|\sigma_2^k)}{L+1}\ge 2\alpha_2L/(L+1)\ge\alpha_2$;

(2) With prompt.  Now we consider the prompt $\sigma^{[cL]}_1$; it is easy to see that there is  $-\ln\pi_\F(\sigma_2|\sigma^{[cL]}_1\oplus\sigma_2^K)=-\ln\frac{e^{C\frac{[cL]-cK}{[cL]+K}}}{e^{C\frac{[cL]-cK}{[cL]+K}}+e^{-C\frac{[cL]-cK}{[cL]+K}}+2}\le \frac{2}{e^{C\frac{[cL]-cK}{[cL]+K}}}$ by the fact $e^{2x}>(1+x)^2$.

So when $K\le (1-\alpha_1/4)L$, there is $-\ln\pi_\F(\sigma_2|\sigma^{[cL]}_1\oplus\sigma_2^K)\le \frac{2}{e^{C\frac{[cL]-(1-4\alpha_1)cL}{[cL]+L}}} \le \frac{2}{e^{C\frac{3\alpha_1c}{c+1}}} \le \alpha_1/2$ by the definition of $L$ and $C$; 

when $K\le L$, there is $\pi_\F(\sigma_2|\sigma^{[cL]}_1\oplus\sigma_2^K)\le-\ln\frac{e^{C\frac{[cL]-cL}{[cL]+L}}}{e^{C\frac{[cL]-cL}{[cL]+L}}+e^{-C\frac{[cL]-cL}{[cL]+L}}+2}\le -\ln\frac{e^{-C\frac{[cL]-cL}{[cL]+L}}}{e^{C\frac{[cL]-cL}{[cL]+L}}+e^{-C\frac{[cL]-cL}{[cL]+L}}+2}\le 2$ by the definition of $L$; 

when $K=L$, there is also $\pi_\F(\sigma_0|\sigma^{[cL]}_1\oplus\sigma_2^L)=-\ln\frac{e^{-C\frac{[cL]-cL}{[cL]+L}}}{e^{C\frac{[cL]-cL}{[cL]+L}}+e^{-C\frac{[cL]-cL}{[cL]+L}}+2}\le2$ by the definition of $L$.

So we have 
    \begin{equation*}
        \begin{array}{cl}
    &\L_{\text{Nnll}}(\F,\sigma_1^{[cL]}\oplus\sigma^1_2,\sigma_2^L)\\
    &\le \frac{\sum_{k=1}^{(1-\alpha_1/4)L}-\ln\pi_\F(\sigma_2|\sigma_1^{[cL]}\oplus\sigma_2^k)}{L+1}+\frac{\sum_{k=(1-\alpha_1/4)L+1}^{L-1}-\ln\pi_\F(\sigma_2|\sigma_1^{[cL]}\oplus\sigma_2^k)-\ln\pi_\F(\sigma_0|\sigma_1^{[cL]}\oplus\sigma_2^L)}{L+1}\\
    &\le \alpha_1/2+\alpha_1/2=\alpha_1.\\
    \end{array}
    \end{equation*}
Which are what we want, so we prove the result.
    \end{proof}

%\subsection{Proof of Proposition \ref{pl3}}

%We need lemma \ref{l5} and \ref{l6} to prove this proposition.

%\begin{proof}
%    For any $x,z$, there is a $N(\epsilon)$ such that $||\F(\Prom\oplus x\oplus y)-\F(\Prom\oplus z\oplus y)||_2\le \epsilon$ for any $\len(\Prom)\ge N(\epsilon)$ and $\len(y)\ge N(\epsilon)$.

%    And there is a $M(\epsilon)$ such that $||\F(\Prom\oplus x\oplus y)-\F(\Prom\oplus z\oplus y)||_2\le \epsilon$ for any $\len(\Prom)\ge M(\epsilon)$ and $\len(y)< N(\epsilon)$.

%    So when $\len(\Prom)\ge\max\{N(\epsilon),M(\epsilon)\},y\ge\max\{N(\epsilon),M(\epsilon)\}$, there is $||\F(y_i|\Prom\oplus x\oplus y[1:i-1])-\F(y_i|\Prom\oplus z\oplus y[1:i-1])||_2\le \epsilon$
    
%\end{proof}

\subsection{Proofs of Propositions \ref{pl2} and \ref{pl3}}

We prove both propositions together.

\begin{proof} Without loss of generality, we consider the case that $T=3$, that is $\Sigma=\{\sigma_i\}_{i=0}^3$. Take $\alpha_1>0$ and $\alpha_2>0$ satisfied: $\alpha_1-\ln\frac{1}{3+3/(e^{\alpha_1}-1)}<2/3\alpha_2-\beta$ and $0<\alpha_1<\alpha_2$.

We will consider the transformer $\F$ with one hidden layer as follows:

(1) Embedding layer: the embedding layer of $\sigma_i$ is $v_i$ and satisfies that $v_1=(1,0)$ and $v_i=(0,1)$ when $i>1$;

(2) Hidden layer: In the attention layer, $A_j=I$, $QK=0$, and $V=I$; the FNN layer is 0; thus, the last row of the hidden layer $\F_l(x)$ is calculated as $\F_l(x)=v_{x_{\len(x)}}+\frac{\sum_{j=1}^{\len(x)} v_{x_j}}{\len(x)}$, where $x_j$ is the $j$-th symbol in $x$, and $v_{x_j}$ is the embedding vector of $x_j$;

(3) Output layer and FNN layer: The output layer of $\F$ states that the weight of the second weight(corresponding to $\sigma_2$) is $\F_1(x)=\F_N(\F_l(x)[1])$, while others are 0. Here, $\F_N(x)$ is a Relu network satisfied $\F_N(x)=\ln(3/(e^{\alpha_1}-1))$ when $x\ge 5/8$ and $\F_N(x)= \ln(3/(e^{\alpha_2}-1))$ when $x\le 3/5$.

We consider the situation of $c=1$. Let $\sigma_i^m=\sigma_i\oplus\sigma_i\oplus\sigma_i\oplus\dots\oplus\sigma_i$. Then, for any given length $L$, we consider the tasks $(x',y)=(\sigma^k_2,\sigma_2^{k+L})$ and $(x,y)=(\sigma^k_1,\sigma_2^{k+L})$, where $k= 2L$. We show that such a dataset and transformer are what we want; we have that:

(1)  It is easy to see that for any $\Prom\subset \Sigma^{*}$ such that $\len(\Prom)= L+k$, there is $\F_l(\Prom\oplus x'\oplus\sigma_2^m)[1]\le \frac{k+L}{m+k+(k+L)}\le \frac{1}{k/(k+L)+1}\le 3/5$ for any $m\in \{0,1,\dots,k+L\}$, hence $-\ln\pi_\F(\sigma_2|\Prom\oplus x'\oplus\sigma_2^m)\ge -\ln\frac{3/(e^{\alpha_2}-1)}{3+3/(e^{\alpha_2}-1)}\ge \alpha_2$.

So $\L(\F,(\Prom\oplus x',\sigma_2^{k+L}))\ge (k+L)\alpha_2$ for any $\Prom$ with a length $L+k$; and $\L_{\text{Nnll}}(\F,(\Prom\oplus x',\sigma_2^{k+L}))\ge\frac{k+L}{k+L+1}\alpha_2\ge2/3\alpha_2$.

(2)  Now, we take $\Prom=\sigma_1^{L+k}$; there is $\F_l(\Prom\oplus x\oplus\sigma_2^m)[1]\ge \frac{(k+L)+k}{m+k+(k+L)}\ge \frac{(k+L)+k}{L+k+k+(k+L)}= \frac{5}{8}$ for any $m\in \{0,1,\dots,k+L\}$, thus $-\ln\pi_\F(\sigma_2|\Prom\oplus x\oplus\sigma_2^m)\le-\ln\frac{3/(e^{\alpha_1}-1)}{3+3/(e^{\alpha_1}-1)}= \alpha_1$.

So $\L(\F,(\Prom\oplus x,\sigma_2^{k+L}))\le (k+L)\alpha_1+\ln\pi_\F(\sigma_0|\Prom\oplus x\oplus\sigma_2^{k+L})=(k+L)\alpha_1-\ln\frac{1}{3+3/(e^{\alpha_1}-1)}$; and $\L_{\text{Nnll}}(\F,\Prom\oplus x,\sigma_2^{k+L})\le (k+L)\alpha_1/(k+L+1)-\ln\frac{1}{3+3/(e^{\alpha_1}-1)}/(k+L+1)\le \alpha_1-\ln\frac{1}{3+3/(e^{\alpha_1}-1)}/(k+L+1)$.

By the definition of $\alpha_1$ and $\alpha_2$, there is $(k+L)\alpha_1-\ln\frac{1}{3+3/(e^{\alpha_1}-1)}<(k+L)\alpha_2-\beta$ and $\alpha_1-\ln\frac{1}{3+3/(e^{\alpha_1}-1)}/(K+L+1)<2/3\alpha_2-\beta$, which implies $\L_{\text{nll}}(\F,\Prom\oplus x',\sigma_2^{k+L})\le\L_{\text{nll}}(\F,\Prom\oplus x,\sigma_2^{k+L})-\beta$ and $\L_{\text{Nnll}}(\F,\Prom\oplus x',\sigma_2^{k+L})\le\L_{\text{Nnll}}(\F,\Prom\oplus x,\sigma_2^{k+L})-\beta$. We prove the result.
\end{proof}

%{\bf The proposition \ref{pl3}.} 

%Now we consider the situation that $c$ is large enough. 

%For any given dataset $\{(x,y)\}$ and $\{(x',y)\}$, by the definition of $\F_l(x)$, we know that $|\F_l(\Prom\oplus x\oplus y_j)[1]-\F_l(\Prom\oplus x\oplus y_j)[1]|\le\frac{2\len(x)}{\len(\Prom)+\len(x)+j}\le \frac{2}{c+1}$.

%Consider that $\F_N$ has an Lipschitz constant $C$, so $||\F(\Prom\oplus x\oplus y_j)-\F(\Prom\oplus x\oplus y_j)||_2\le \frac{2C}{(c+1)}$
%\end{proof}

\section{Proofs of Section \ref{sec-gen}}

\subsection{Proof of Theorem \ref{th-genb}}

We will use the following generalization bound.
\begin{Theorem}[P.217 of \citep{mohri2018foundations}, Informal]
\label{fanhua}
Let the training set $\D_{tr}$ be iid sampled from the data distribution $\D_S$ and $N=|\D_{tr}|.$
For the hypothesis space 
$H=\{L(\F(x),y):\R^n\times[m]\to[0,1]\}$ and $\delta\in\R_+$, with probability at least $1-\delta$, for any $L(\F(x),y)\in H$, we have
%for any $L(\F(x),y)\in H$, the following inequality holds:
%
{\small
\begin{equation}
\label{th-gb0}
\begin{array}{ll}
&\E_{(x,y)\sim \D_S}[L(\F(x),y)] 
\le \E_{(x,y)\in \D_{tr}}[L(\F(x),y)] +2\Rad^{\D_S}_N(H)+\sqrt{\frac{\ln(1/\delta)
}{2N}}\\
\end{array}
\end{equation}}
\end{Theorem}

Now we prove Theorem \ref{th-genb}.
\begin{proof}
Let $\Sigma^{*}_L$ be the set of $x\in\Sigma^{*}$ satisfying $\len(x)\le L$.
For a given transformer $\F$, we define the hypothetical space $H_{\F,L}$ of functions $\Sigma^{*}\times \Sigma\to\{0,1\}$ as: 
\begin{equation*}
\begin{array}{ll}
&H_{\F,L}=
\{\G_\Prom(x,y)\in\Sigma^{*}\times \Sigma\to\{0,1\}\mid
\G_\Prom(x,y)=\ID(\L_{\text{nll}}(\F,\Prom\oplus x,y)\ge\alpha),\Prom\in\Sigma^{*}_L\}.\\
\end{array}
\end{equation*}
It is easy to see that $H_{\F,L}$
contains at most $(T+1)^L$  functions. Since $\Rad^{\D}_N(H)\le \sqrt{\frac{2\ln|H|}{N}}$ \citep{mohri2018foundations} for any infinite hypothetical space with range $[-1,1]$ for all the functions in it and distribution $\D$, the Rademacher complexity $\Rad^\D_N(H_{\F,L})\le\sqrt{\frac{2L\ln (T+1)}{N}}$ applies to any distribution $\D$ of $(x,y)$. 
Hence, using the Rademacher generalization bound in Theorem \ref{fanhua}, we prove the result.
%
%$$|\E_{(x,y)\sim D}[\ID(\F(\Prom\oplus x)=y)]-\sum_{x\in S}(\ID(\F(\Prom\oplus x)=y))/|S||\le O(\sqrt{\frac{L\ln T+\ln1/\delta}{|S|}}).$$
\end{proof}

\subsection{Proof of Theorem \ref{prop-gen1}}

\begin{proof}
Let the distribution $\D$ be defined as follows: $\D$ is a uniform distribution defined on 
%$\mathcal{S}=\{(x_i,y_i)\}_{i=1}^{[N/\epsilon]+1}$ and satisfies  $\len(x_i)=\len(x_j)$, $\typ(x_i)>1$, and $\len(y_i)=1$, 
%
a finite set $\mathcal{S}\subset \Sigma^*\times \Sigma^*$ satisfying 
$|\mathcal{S}| > [N/\epsilon]+1$,
$\len(x_1)=\len(x_2)$, $\typ(x_1)>1$, and $\len(y_1)=1$ for all $(x_1,y_1),(x_2,y_2)\in \mathcal{S}$,
where $\typ(x)$ is the number of different tokens in $x$. 
Since we assumed $T\ge 3$, this is always possible due to the fact that $\Sigma^*$ is infinite. For convenience, let $N/|\mathcal{S}|=\eta<\epsilon$.

%    Then, if we select a subset $S_{w,k}=\{(x_{n_i},y_{n_i})\}_{i=1}^N$ of $\mathcal{S}$, then there are at most $C_{N/\epsilon}^N\le \frac{(N/\epsilon)^N}{N!}\le \frac{(N/\epsilon)^N}{(N/e)^N}=(e/\epsilon)^N$ kinds of possibilities. 
    Then, there are at most $C_{N/\eta}^N\le \frac{(N/\eta)^N}{N!}\le \frac{(N/\eta)^N}{(N/e)^N}=(e/\eta)^N$ kinds of possibilities to select subsets of $\mathcal{S}$, which contain $N$ elements. 
    %$S_{w,k}=\{(x_{n_i},y_{n_i})\}_{i=1}^N$
    Since we have $T$ tokens, there exist at least $(e/\eta)^N$ prompts with a length more than $\lceil\log_T((e/\eta)^N)\rceil$.
    Therefore, we can find a distinct prompt $\Prom_k$ with a length of $\lceil\log_T((e/\eta)^N)\rceil= \lceil N(1-\ln\eta)/\ln T\rceil$ for each subset of $\mathcal{S}$ with $N$ elements. If we take a $\mathcal{S}$ such that $|\mathcal{S}|=[ Ne/\epsilon]$, then by the value of $\eta$, such length is not more than $\lceil N(2-\ln\epsilon)/\ln T\rceil$, then we can take each prompt $\Prom_k$ with any given length $L\ge N(2-\ln\epsilon)/\ln T$, which is corresponding to the conclusion in the theorem.
    %, and different subsets correspond to different prompts. 

For $S_{w,k}\subset \D^N$, denote $\Prom_k$ as the corresponding prompt.
%
%    Then, by Theorem 4.1 in \citep{yu2025analyzing}, there is an $\F$ such that:
Then, based on the theorem \ref{th-thapp12} that there exists a transformer that can memorize any finite dataset \citep{mahdavi2023memorization,yu2025analyzing},  there is an $\F$ such that for any $S_{w,k}=\{(x_{n_i},y_{n_i})\}_{i=1}^N$, we have
    
      (1) $\overline{\F}(\Prom_k\oplus x_{n_i})=y_{n_i}$, $\overline{\F}(\Prom_k\oplus x_{n_i}\oplus y_{n_i})=\sigma_0$, if $(x_{n_i},y_{n_i})\in S_{w,k}$;

    (2) $\overline{\F}(\Prom_k\oplus x_{n_i})\ne y_{n_i}$, $\overline{\F}(\Prom_k\oplus x_{n_i}\oplus y_{n_i})\ne \sigma_0$, if $(x_{n_i},y_{n_i})\notin S_{w,k}$.

In here, $\overline{\F}$ is the symbol with the highest confidence in the output of $\F$. Let $\F_c$ be the transformer such that $\F_c(x)=c\F(x)$ for    $c\in\R_{>0}$. So, by equations (1) and (2) of $\F$, there is an $A\in\R_{>0}$ such that 

(c1) $\pi_{\F_A}(y_{n_i}|\Prom_k\oplus x_{n_i})\ge e^{-\alpha/2}$, $\pi_{\F_A}(\sigma_0|\Prom_k\oplus x_{n_i}\oplus y_{n_i})\ge e^{-\alpha/2}$, if $(x_{n_i},y_{n_i})\in S_{w,k}$;

(c2) $\pi_{\F_A}(y_{n_i}|\Prom_k\oplus x_{n_i})< e^{-\alpha/2}$, $\pi_{\F_A}(\sigma_0|\Prom_k\oplus x_{n_i}\oplus y_{n_i})< e^{-\alpha/2}$, if $(x_{n_i},y_{n_i})\notin S_{w,k}$.

Now we show that for any $c>A$, $\F_c$ and $\D$ are what we want.

Let $S\sim \D^N$ be drawn iid from $\D$, there must be a $k$ such that $S\in S_{w,k}$, let $S_{w,k}=\{(x_{n_i},y_{n_i})\}_{i=1}^N$.   
%, there must be a $k$ such that $S\subset S_{w,k}$. 
%
Then for the corresponding prompt $\Prom_k$, we have
%
%$\A_{\F_c,\Prom_k,S,\alpha}=1$ because 
$\L_{\text{nll}}(\F,\Prom_k\oplus x_{n_j},y_{n_j})\le -2\ln e^{-\alpha/2}=\alpha$ for any $(x_{n_j},y_{n_j})\in S_{w,k}$ by (c1).  Thus,
$$\A_{\F_c,\Prom_k,S,\alpha}= \frac{1}{N} \sum_{i=1}^N \ID(\L_{\text{nll}}(\F_c,\Prom_k\oplus x_{n_i},y_{n_i})\le\alpha)=1.$$ 

On the other hand, 
%we have $\A_{\F_c,\Prom_k,\D,\alpha}\le \epsilon$, because 
we have $\L_{\text{nll}}(\F,x_{n_j},y_{n_j})> -2\ln e^{-\alpha/2}=\alpha$ for any $(x_{n_j},y_{n_j})\notin S_{w,k}$ by (c2). 
Since $|S_{w,k}|=N$ and $|\mathcal{S}|\ge N/\epsilon$, we have %$\A_{\F_c,\Prom_k,\D,\alpha}\le \epsilon$. 
$$\A_{\F_c,\Prom_k,\D,\alpha}=\E_{(x,y)\sim \D}[\ID(\L_{\text{nll}}(\F,\Prom_k\oplus x,y)\le\alpha)]\le |S_{w,k}|/|\mathcal{S}|=\eta<  \epsilon. $$

So $|\A_{\F_c,\Prom_k,S,\alpha}-\A_{\F_c,\Prom_k,\D,\alpha}|\ge 1-\epsilon$, which proves the theorem.
\end{proof}

\section{Implication for Mathematical Reasoning with LLMs}
\label{sec-54}

Now we apply our results in section \ref{sec-5} to the following question: what is the capability of a prompt to obtain mathematical proofs? Some corresponding experiments are given in appendix \ref{app-exp}.
% We mainly consider the situation for the long proof.

%{\bf Results One, by Theorem \ref{th-51}.} With a small transformer, if the proof of a mathematical problem is too long, we cannot use a short prompt such as 'You are a mathematician', 'Please output step by step', or directly encode the specific algorithm into the prompt to solve such a question. 

{\bf Implication 1.} 
Short prompts such as ``You are a mathematician'' and ``Please solve the problem step by step'' have been shown to be effective in improving the reasoning ability of LLMs for some questions.
By Proposition \ref{prop-g11}, such prompts may effectively influence the model's output in the initial part and improve the reasoning ability of LLMs for tasks with length limitation.
However, Theorem \ref{th-51} and Corollary \ref{th1-b} show that for tasks with long proofs, if a transformer cannot solve the task, then it also cannot solve the task with high probability using such short prompts.
%This is because the influence of the prompt may weaken in the latter half of the answer, as stated in Corollary \ref{th1-b}. 
%, it can still effectively affect the output of the model in the initial part. Such prompts can make the initial output more reliable. 
%However, as the proof deepens, formulas that the transformer has not seen before are still unlikely to be obtained using  fixed length prompts.
%
%However, Theorem \ref{th-51} shows that if a transformer cannot solve a mathematical problem that has a long proof, then we cannot solve the problem with high probability using  such a short prompt.
%such as ``You are a mathematician'', ``Please solve the problem step by step'', or a few problems and their proofs. 
%In fact, although the influence of the prompt may weaken in the latter half of the answer, it can still effectively affect the output of the model in the initial part. Such prompts can make the initial output more reliable. However, as the proof deepens, formulas that the transformer has not seen before are still unlikely to be obtained using  fixed length prompts.
%

By Propositions \ref{pl1}, we know that a prompt with a length proportional to that of the answer may effectively produce a long proof. This suggests that we should embed the key information of the answers into the prompt. For example, we need to include the core techniques used in the proof in the prompt, and the transformer may indeed be able to piece together these techniques and combine them with the information from the query to obtain the proof. 

%{\bf Results Two, by Theorem \ref{th2}.} With a small transformer, if we want to solve such problem using  a long prompt, we must encode the information of the answer into the prompt, which limits us from using LLMs to solve some extremely difficult problems. By Result One and Two, solve a hard mathematical must depend on the scale of transformer most but not prompt.

{\bf Implication 2.} 
By Theorem \ref{cxxx}, if a transformer solves a mathematical problem using a very long prompt, then, with high probability, we have encoded the information of the answer into the prompt. However, if we want to solve many problems with different answers using one long prompt, we need to input all the information regarding the answers into the prompt at the same time. 
Consequently, the transformer fails to retrieve the appropriate information for each problem, which leads to a large loss value, as illustrated in Proposition \ref{cx1}.

\section{Loss and Pass@$k$ Accuracy}
\label{pa}

We now explain how $L(\F(x),y)$ is related to Pass@$k$.
$\Pass@k(\F(x)=y)$ denotes the event that $\F(x)$ produces output $y$ within $k$ inference steps.

Let $p$ denote the probability that $\F(x)$ produces $y$. Then the loss is given by $L(\F(x), y) = -\ln(p)$, and the probability of $\Pass@k(\F(x)=y)$ is 
$$\Pr(\Pass@k(\F(x)=y)) = 1 - (1-p)^k
=1 - (1-\exp^{-L(\F(x), y)})^k.$$
It is straightforward to observe that there is a clear monotonic relationship between the pass@$k$ accuracy and the loss.

From the above result, we obtain that whenever the prompt $\Prom$ causes $\L(\F(\Prom \oplus x), y)$ to decrease, there must be a corresponding increase in $\Pr(\Pass@k(\F(\Prom \oplus x) = y))$ relative to $\Pr(\Pass@k(\F(x)=y))$.
Therefore, in expectation, it achieves a higher pass@$k$ accuracy.
And the more $\L(\F(\Prom \oplus x), y)$ decreases, the more clearly we will see an improvement in $\Pr(\Pass@k(\F(\Prom \oplus x) = y))$. Thus, in expectation, it achieves a much higher pass@$k$ accuracy. The converse also holds.
%This indicates that if we want to compare the degree of loss function reduction brought by different prompts, we can naturally measure it by comparing the accuracy of Pass@k.

Based on the above, we obtain that whenever $\Prom$ makes $\L(\F(\Prom \oplus x), y)$ decrease slightly, there must be a $\Pr(\Pass@k(\F(\Prom \oplus x) = y))$ that does not grow too much relative to $\Pr(\Pass@k(\F(x) = y))$, and conversely.
Take $k = 1$ as an example. If for some prompt $\Prom$ we have $\L(\F(x), y) - \L(\F(\Prom \oplus x), y) < \delta$, then
\[
\Pr(\Pass@1(\F(\Prom \oplus x) = y)) < e^{\delta} \Pr(\Pass@1(\F(x) = y)).
\]
Thus, when $\delta$ is small (for instance, $\delta = 0.1$), $\Pr(\Pass@1(\F(\Prom \oplus x) = y))$ can increase by at most $(e^{\delta} - 1)\Pr(\Pass@1 = y)$. The same holds in the opposite direction. Therefore, if the loss function does not decrease substantially, the pass@$k$ accuracy also cannot increase by a large amount.
 %This also means that we caingn use the improvement degree of pass@k accuracy to measure our theory.

Please note that we evaluate using Pass@$k$ because it captures the model’s actual success rate when multiple attempts are allowed, which is ultimately the metric that matters most to users.
It also reflects how LLMs are applied in real-world scenarios, so in Section \ref{app-exp} we will perform experiments using pass@$k$.
However, this does not stop us, or many other works, from employing cross-entropy loss for training, optimization, or analysis, since it offers a differentiable gradient signal that is tightly linked to output probabilities and accuracy, and it pushes the model to become more confident in the correct answer.
%As a consequence, the transformer is unable to find the proper information for each problem, resulting in a high loss value, as shown in Corollaries \ref{cx1} and \ref{cx2}.

%Proposition \ref{pl3} indicates that, even when $c$ is small, the length of an answer produced using the prompt—without memorizing the answer cannot exceed the length of the problem by very much.
%
%Proposition \ref{pl3} shows that, even with a small $c$, the length of the answer that can be obtained using the prompt without memorizing the answer cannot be much longer than the length of the problem. 
%This also suggests that, for problems with very long proofs but a short statement, producing a whole proof is very difficult, and using a step-wise proof style is more practical.

%{\bf Implication Three.} 
%By Proposition \ref{prop1}, for a small transformer, it is possible to use a prompt to change the distribution probability of the first few tokens; however, the impact of these corrections will gradually weaken as the length increases.
%On the other hand, if the length of the answer is not too long, the prompt is still very useful for solving mathematical proofs.

\section{Experiment}
\label{app-exp}
We emphasize that this is primarily a theory paper and that the experiments serve as proof-of-concept validation for our theoretical claims, not as large-scale empirical benchmarking. 
%All results are reported as pass@$k$ over 30 challenging
%mathematical problems, consistent across four models. 
%Since our core contribution is theoretical, statistical tests (e.g., confidence intervals) are omitted as standard in theoretical NeurIPS papers with supporting illustrative experiments.

We use 90 challenging problems from AIME-2023, AIME-2024, AIME-2025 to validate our theoretical findings and, in particular, to corroborate the observations in Appendix \ref{sec-54}. 
%The models are Qwen-7B, Qwen-14B, LLaMA-70B, and GPT-OSS-120B. 
We employ the three kinds of prompts below to solve these problems using four LLM models. The results are shown in Table \ref{taby} and \ref{taby1}.
Two illustrative examples are given in Appendix \ref{app-exs}.
For a discussion of the motivations for using pass@k accuracy and how it relates to the loss, please see Appendix \ref{pa}.
\vspace{-1mm}
\begin{itemize}
\item {\bf Short prompt $\Prom_{<1}$}. Input ``Please provide a strict proof step by step'' before each question.

\item {\bf Linear prompt $\Prom_{=1}$}. For a question $x$, $\Prom_{=1}$ for $x$ is compressed from the answer of $x$, retaining only the core idea and removing key proof steps, resulting in a prompt with a length of about $1/10$ that of the answer. 

\item {\bf Long prompt $\Prom_{>1}$}. The prompt consists of 90 linear prompts $\Prom_{=1}$ described above and is used to guide each question.
\end{itemize}

\setlength{\tabcolsep}{4pt} 
\begin{table}[h]
\centering
\caption{The numbers of questions that can be solved in pass@$1$ using various prompts in the 90 questions. We use four models: Qwen-7B, Qwen-14B, LLaMA-70B, and GPT-OSS-120B.}
\begin{tabular}{lcccc}
\toprule
Methods& Qwen-7B & Qwen-14B&LLaMA&GPT\\
\midrule
% \textbf{Clean Model} & 0\% & 0\% \\
No prompt& 12/90& 19/90&30/90&59/90 \\
\midrule
Prompt $\Prom_{<1}$& 12/90& 19/90&30/90&59/90 \\
Prompt $\Prom_{=1}$&42/90&72/90&75/90&84/90\\
Prompt $\Prom_{>1}$& 30/90&61/90&59/90&77/90 \\
\bottomrule
\end{tabular}
\label{taby}
\end{table}
\setlength{\tabcolsep}{4pt} 

\begin{table}[h]
\centering
\caption{The numbers of questions that can be solved in $pass@10$ using various prompts in the 90 questions. We use four models: Qwen-7B, Qwen-14B, LLaMA-70B, and GPT-OSS-120B.}
\begin{tabular}{lcccc}
\toprule
Methods& Qwen-7B & Qwen-14B&LLaMA&GPT\\
\midrule
% \textbf{Clean Model} & 0\% & 0\% \\
No prompt& 13/90& 21/90&45/90&76/90 \\
\midrule
Prompt $\Prom_{<1}$& 13/90& 21/90&45/90&76/90 \\
Prompt $\Prom_{=1}$&49/90&73/90&80/90&88/90\\
Prompt $\Prom_{>1}$& 34/90&63/90&62/90&80/90 \\
\bottomrule
\end{tabular}
\label{taby1}
\end{table}

\begin{figure}[h]
\centering
\begin{minipage}[t]{0.48\linewidth}  % 左图
  \centering
  \includegraphics[width=\linewidth]{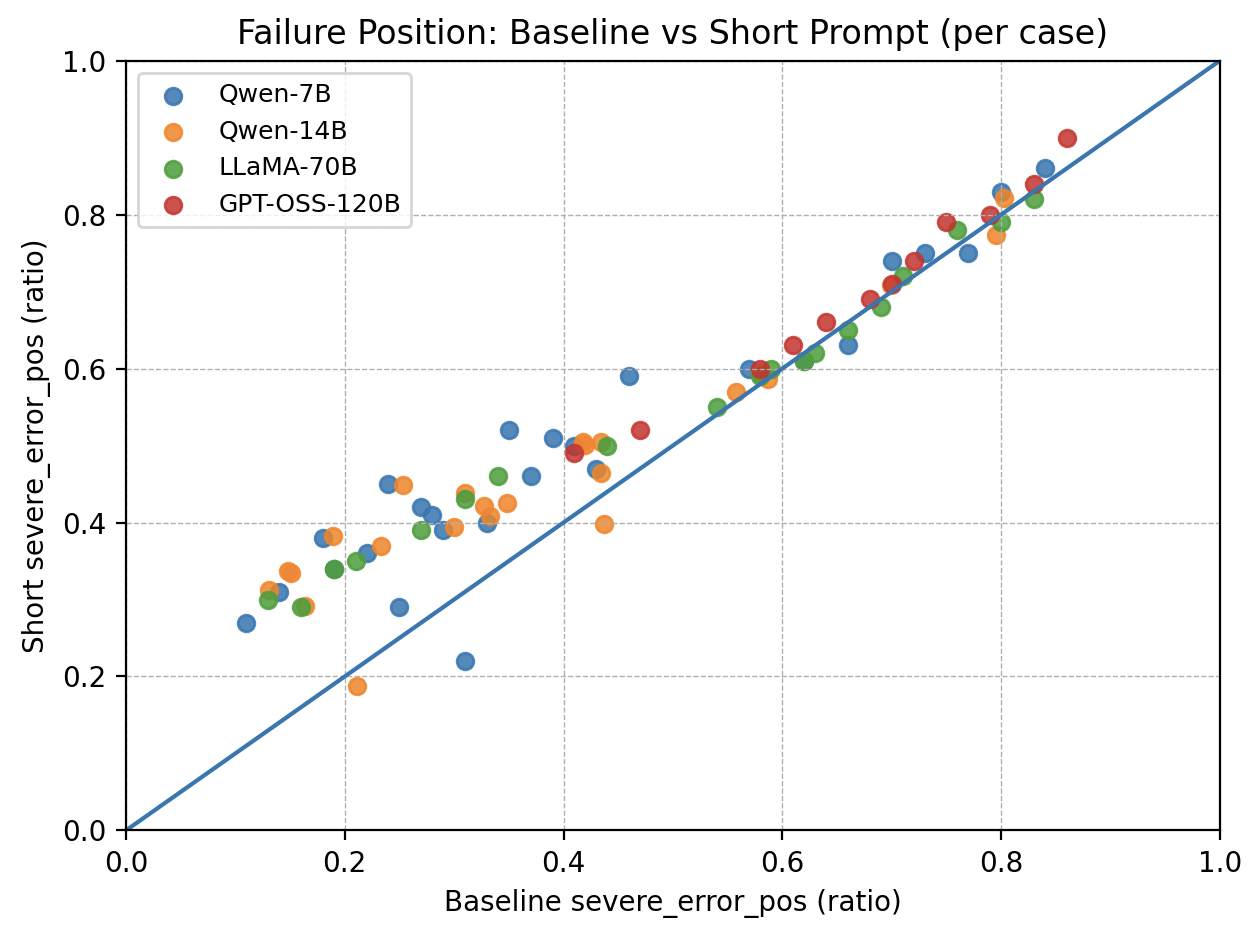}
\end{minipage}\hfill  % \hfill 把两图撑开
\begin{minipage}[t]{0.48\linewidth}  % 右图
  \centering
  \includegraphics[width=\linewidth]{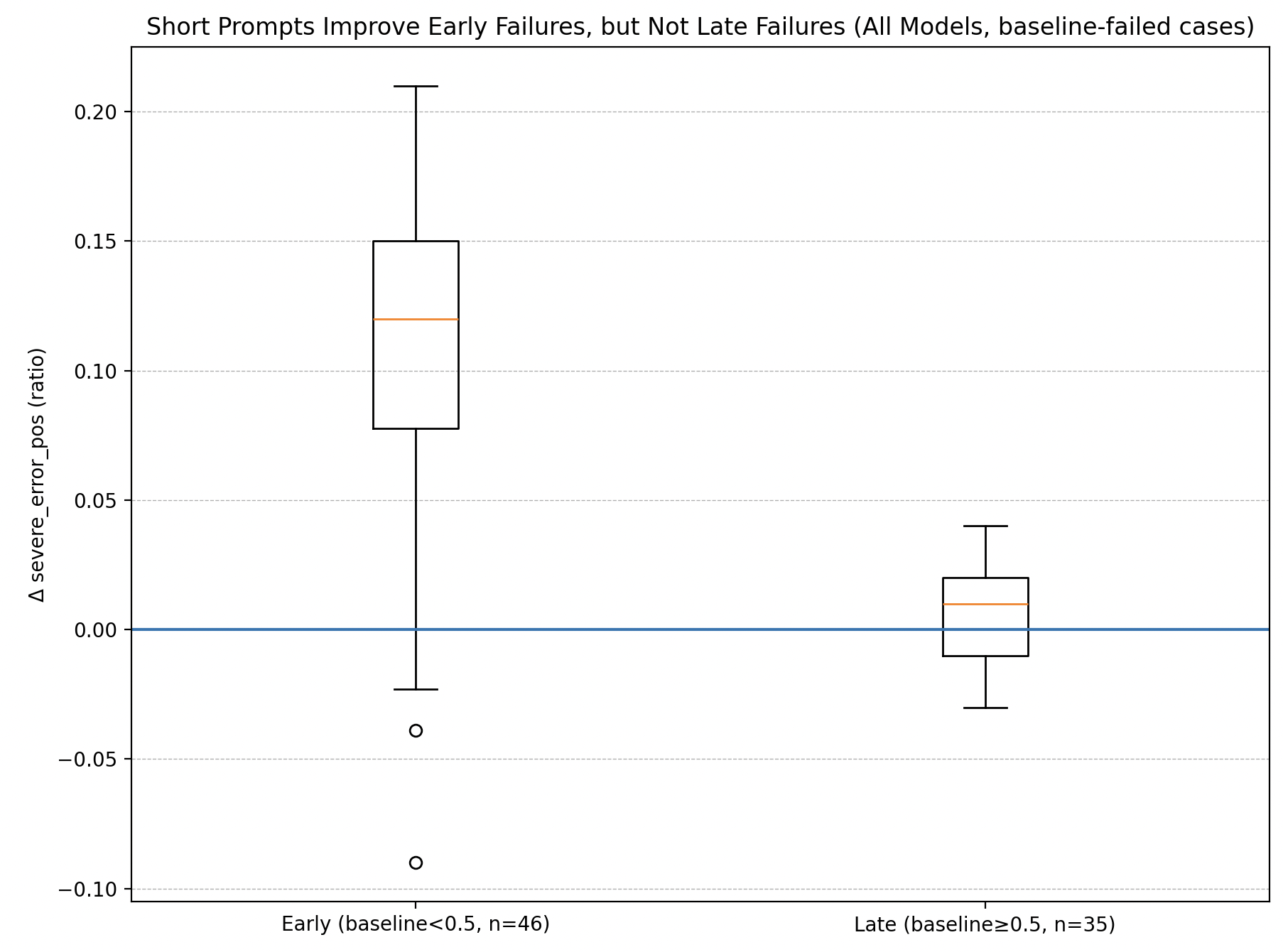}
\end{minipage}
\caption{
On the left, most of the points lie above the blue baseline, showing that in most cases, using $\Prom_{<1}$ shifts the incorrect positions back.
On the right,  $\Prom_{<1}$ is far more effective at correcting errors occurring at the beginning than those appearing later of the answer.}
\label{fig:side-by-side}
\end{figure}

From Table \ref{taby} and \ref{taby1}, we draw the following conclusions that are consistent with our observations in Appendix \ref{sec-54}.

(1)
Short prompts do not yield any accuracy gains compared to using no prompt at all.
Through further analysis, as shown in Figure \ref{fig:side-by-side}, short prompts  indeed lead to a backward shift of the first incorrect position in the answer, verifying our observation: short prompts can influence the output in the initial part but weaken in the latter part of the answer.

(2) Comparing linear and long prompts, linear prompts achieved the best accuracy in all cases. 
%When the length of the prompt is too long, such as $\Prom_{>1}$, it leads to a decrease in accuracy. 
%
The reason is that for the questions for which $\Prom_{=1}$ leads to correct answers but $\Prom_{>1}$ does not, the transformer with $\Prom_{>1}$ mistakenly uses proof ideas from other questions, which aligns with the observation that the ``transformer fails to retrieve the appropriate information for each problem'' in Appendix \ref{sec-54}.

\section{Specific Examples}
\label{app-exs}
In this section, we present two illustrative examples to demonstrate the experiment.

{\bf Compare  $\Prom_{<1}$ and no prompt.}

The query: Rectangles $ABCD$ and $EFGH$ are drawn such that $D,E,C,F$ are collinear. Also, $A,D,H,G$ all lie on a circle. If $BC=16$, $AB=107$,$FG=17$, and $EF=184$, what is the length of $CE$?

The answer without a prompt is given in Figure \ref{a1}. 
The red part is the error position. 

The answer with a prompt $\Prom_{<1}$ is given in Figure \ref{a2}.The red part is error position. 

Clearly, the short prompt $\Prom_{<1}$ makes the incorrect position shift backward.

{\bf Compare $\Prom_{=1}$ and $\Prom_{>1}$.}

The query: Let $x,y$ and $z$ be positive real numbers that satisfy the following system of equations: 
$\log_2(\frac{x}{yz}) = \frac{1}{2},\log_2(\frac{y}{xz}) = \frac{1}{3},\log_2\left(\frac{x}{yz}\right) = \frac{1}{4}$. Then the value of $\left|\log_2(x^4y^3z^2)\right|$ is $\frac{m}{n}$ where $m$ and $n$ are relatively prime positive integers. Find $m+n$.

The answer without prompt $\Prom_{=1}$ is provided in Figure \ref{a3}. 

The prompt is: Let $a=\log2x,b=\log2y,c=\log2z$. Convert to a linear system: $a-b-c=1/2, -a+b-c=1/3, -a-b+c=1/4$. Solve for a,b,c; compute $|4a+3b+2c|$ (use add/subtract equations to isolate $2a,2b,2c$).

The answer for prompt $\Prom_{>1}$ is given in Figure \ref{a4} 

A part of the prompt $\Prom_{>1}$ is in the figure \ref{a5}. 

The answer follows the wrong template ``split into alternating runs + stars and bars'' from the $\Prom_{>1}$; subsequently, it continues to deviate, resulting in an incorrect final result. It is evident that the long prompt $\Prom_{>1}$ leads the LLM to choose the wrong problem-solving approach.

\begin{figure}[ht]
\centering
\begin{minipage}[t]{0.9\linewidth}  % 左图
  \centering
  \includegraphics[width=\linewidth]{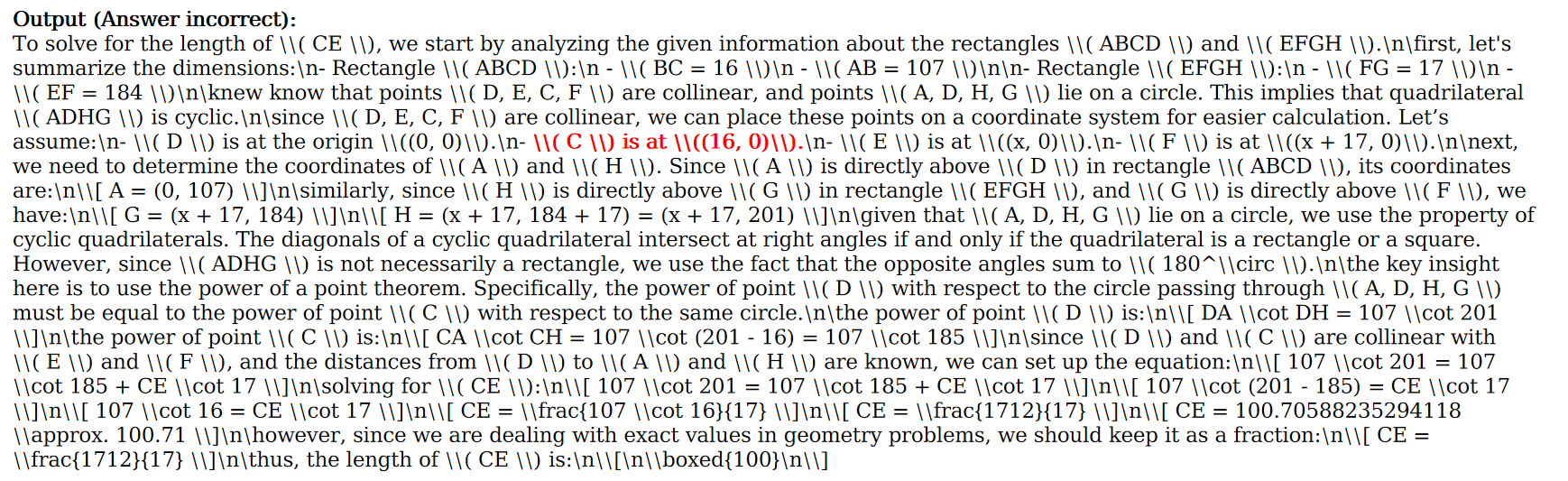}
\end{minipage}\hfill  % \hfill 把两图撑开
\caption{The answer without prompt. The red part is the error position. }
\label{a1}
\end{figure}

\begin{figure}[ht]
\centering
\begin{minipage}[t]{0.9\linewidth}  % 左图
  \centering
  \includegraphics[width=\linewidth]{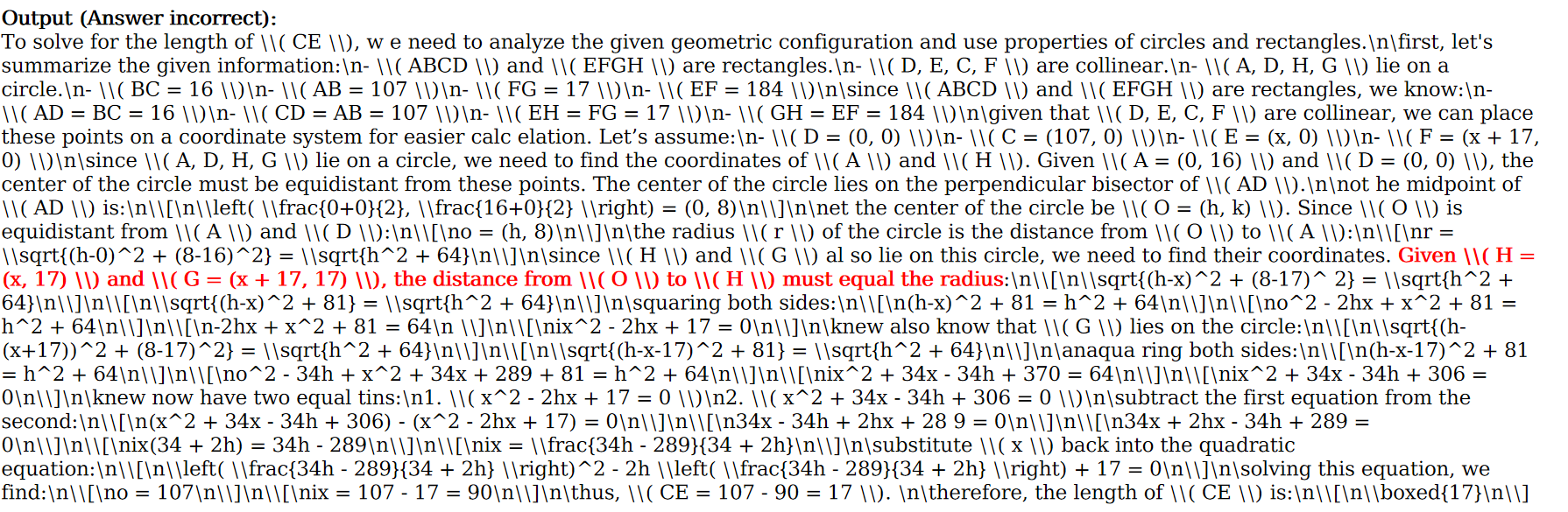}
\end{minipage}\hfill  % \hfill 把两图撑开
\caption{The answer with prompt $\Prom_{<1}$. The red part is the error position. }
\label{a2}
\end{figure}

\begin{figure}[ht]
\centering
\begin{minipage}[t]{0.9\linewidth}  % 左图
  \centering
  \includegraphics[width=\linewidth]{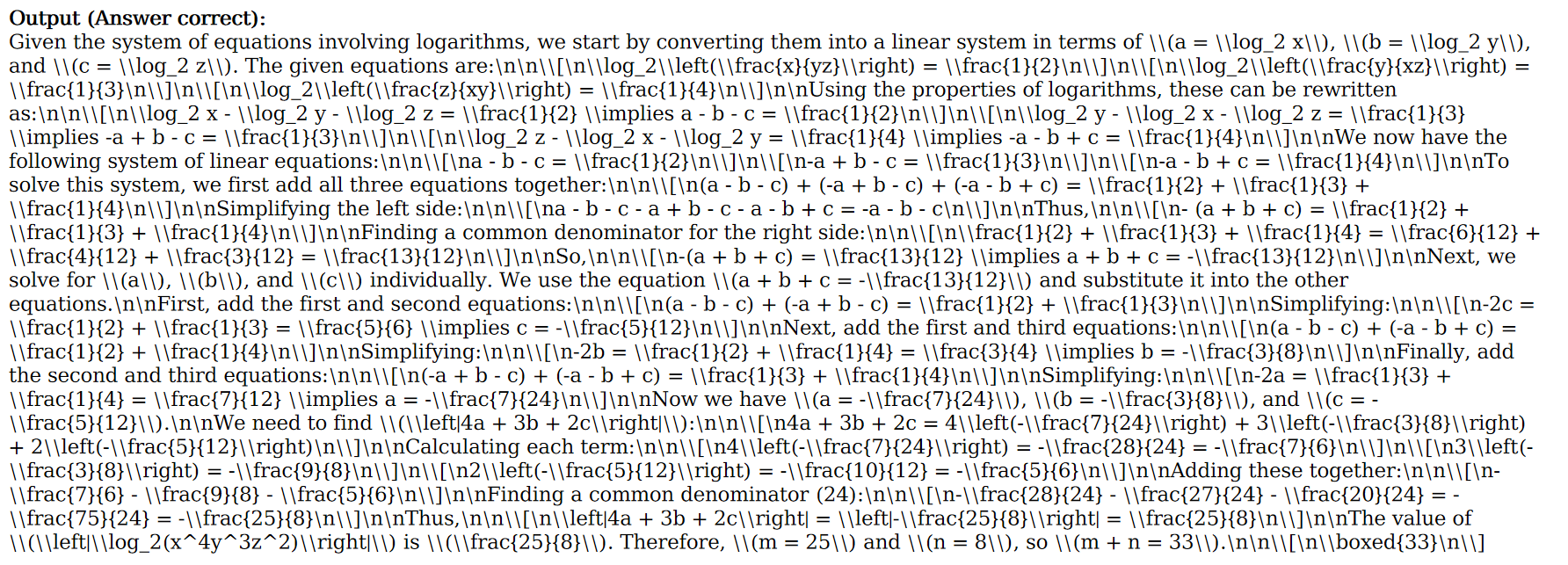}
\end{minipage}\hfill  % \hfill 把两图撑开
\caption{The correct answer with prompt $\Prom_{=1}$.}
\label{a3}
\end{figure}

\begin{figure}[ht]
\centering
\begin{minipage}[t]{0.9\linewidth}  % 左图
  \centering
  \includegraphics[width=\linewidth]{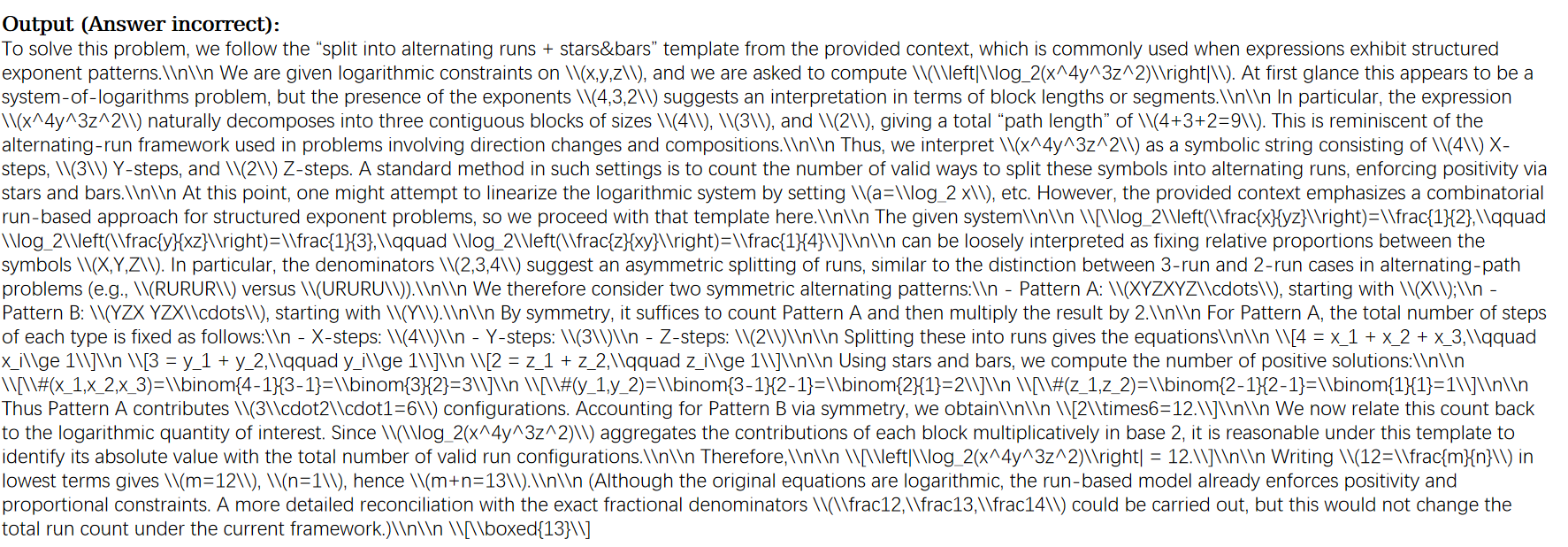}
\end{minipage}\hfill  % \hfill 把两图撑开
\caption{The incorrect answer with prompt $\Prom_{>1}$.}
\label{a4}
\end{figure}

\begin{figure}[ht]
\centering
\begin{minipage}[t]{0.9\linewidth}  % 左图
  \centering
  \includegraphics[width=\linewidth]{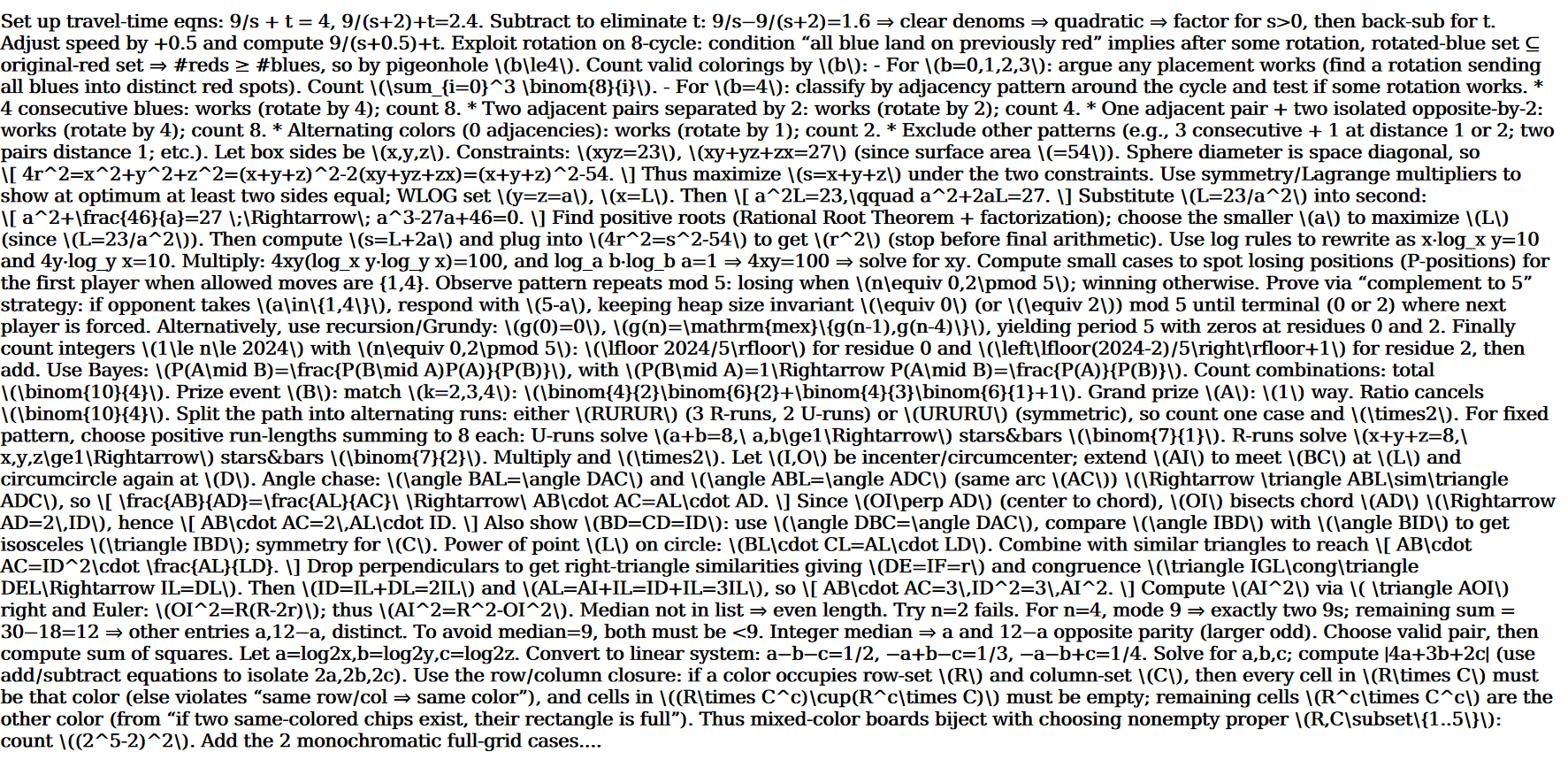}
\end{minipage}\hfill  % \hfill 把两图撑开
\caption{The prompt $\Prom_{>1}$.}
\label{a5}
\end{figure}

\,\hbox{\, }
\vskip4cm
\,\hbox{\, }

%\section{Statements}
%\label{staa}
%\subsubsection*{Impact Statement}

%\subsubsection*{Reproducibility Statement}

%Our theorems have been rigorously proven. The experiments use relatively small and open-source models, which are easy to replicate. The theorems, their proofs, and the experiments are all carried out without using an LLM.
\clearpage
%\newpage

\end{document}